\documentclass[twoside,11pt]{article}
\usepackage{jair, theapa}

\usepackage[T1]{fontenc}

\usepackage{graphicx,amsmath,amsfonts,amscd,amssymb,bm,url,color,wrapfig,latexsym}
\usepackage[small,bf]{caption}
\usepackage{listings}
\usepackage{subcaption}
\usepackage{bbm}
\usepackage{mathtools}
\usepackage{microtype}
\usepackage{siunitx}
\usepackage{algorithm}
\usepackage{algorithmicx}
\usepackage{algpseudocode}
\usepackage[hypertexnames=true,hidelinks,pagebackref,linktocpage,implicit=false]{hyperref}
\usepackage{verbatim}
\usepackage{framed}
\usepackage{comment}
\usepackage{tikz}
\usepackage{pgfplots}
\usepackage{pgf-umlsd}
\usepackage{ifthen}
\usepackage{tabularx}
\usepackage{fge}
\usepackage{atbegshi}
\usepackage{booktabs}
\usepackage{smartdiagram}
\usepackage[shortlabels]{enumitem}
\usepackage{pifont}
\newcommand{\cmark}{\ding{51}}%
\newcommand{\xmark}{\ding{55}}%
\usepackage[toc, page]{appendix}
\hypersetup{
    colorlinks=true,%
    citecolor=blue,%
    filecolor=blue,%
    linkcolor=blue,%
    urlcolor=blue
}
\newtheorem{theorem}{Theorem}[section]
\newtheorem{lemma}[theorem]{Lemma}

\newtheorem{proposition}[theorem]{Proposition}
\newtheorem{definition}[theorem]{Definition}

\newtheorem{example}[theorem]{Example}
\newtheorem{assumption}[theorem]{Assumption}

\renewcommand{\mathbf}{\boldsymbol}

\newcommand{\mb}{\mathbf}
\newcommand{\mc}{\mathcal}

\newcommand{\bb}{\mathbb}

\newcommand{\eps}{\varepsilon}

\newcommand{\indicator}[1]{\mathbbm 1\{#1\}}

\makeatletter
\def\Ddots{\mathinner{\mkern1mu\raise\p@
\vbox{\kern7\p@\hbox{.}}\mkern2mu
\raise4\p@\hbox{.}\mkern2mu\raise7\p@\hbox{.}\mkern1mu}}
\makeatother
\newcommand{\paren}[1]{\left( #1 \right)}
\newcommand{\brac}[1]{\left[ #1 \right]}

\numberwithin{equation}{section}

\def \endprf{\hfill {\vrule height6pt width6pt depth0pt}\medskip}
\newenvironment{proof}{\noindent {\bf Proof} }{\endprf\par}

\firstpageno{847}

\title{Computational Benefits of Intermediate Rewards for Goal-Reaching Policy Learning}
\ShortHeadings{Computational Benefits of Intermediate Rewards}
{Zhai, Baek, Zhou, Jiao, \& Ma}
\author{\name Yuexiang Zhai \email simonzhai@berkeley.edu \\
      \name Christina Baek \email kbaek@cs.cmu.edu \\
      \addr University of California, Berkeley\\
      Department of Electrical Engineering \& Computer Sciences\\ Berkeley, CA, 94720, USA
      \AND
      \name Zhengyuan Zhou \email zzhou@stern.nyu.edu \\
      \addr New York University\\
      Stern School of Business\\
      New York, NY, 10012, USA
      \AND
      \name Jiantao Jiao \email jiantao@eecs.berkeley.edu \\
      \addr University of California, Berkeley\\
      Department of Electrical Engineering \& Computer Sciences\\
      Department of Statistics\\
      Berkeley, CA, 94720, USA\\
      \name Yi Ma \email yima@eecs.berkeley.edu \\
      \addr University of California, Berkeley\\
      Department of Electrical Engineering \& Computer Sciences\\
      Berkeley, CA, 94720, USA}

\begin{document}
\maketitle
\jairheading{73}{2022}{847-896}{09/2021}{03/2022}
\begin{abstract}
    Many goal-reaching reinforcement learning (RL) tasks have empirically verified that rewarding the agent on subgoals improves convergence speed and practical performance. We attempt to provide a theoretical framework to quantify the computational benefits of rewarding the completion of subgoals, in terms of the number of synchronous value iterations. In particular, we consider subgoals as one-way {\em intermediate states}, which can only be visited once per episode and propose two settings that consider these one-way intermediate states: the one-way single-path (OWSP) and the one-way multi-path (OWMP) settings. In both OWSP and OWMP settings, we demonstrate that adding {\em intermediate rewards} to subgoals is more computationally efficient than only rewarding the agent once it completes the goal of reaching a terminal state. We also reveal a trade-off between computational complexity and the pursuit of the shortest path in the OWMP setting: adding intermediate rewards significantly reduces the computational complexity of reaching the goal but the agent may not find the shortest path, whereas with sparse terminal rewards, the agent finds the shortest path at a significantly higher computational cost. We also corroborate our theoretical results with extensive experiments on the MiniGrid environments using Q-learning and some popular deep RL algorithms.
\end{abstract}
\section{Introduction} 
Markov decision processes (MDPs) \shortcite{bertsekas1995dynamic,bertsekas1996neuro,puterman2014markov,russell2016artificial} provide a powerful framework for reinforcement learning (RL) and planning \shortcite{bertsekas2019reinforcement,sutton2018reinforcement,szepesvari2010algorithms,siciliano2010robotics,lavalle2006planning,russell2016artificial} tasks. In particular, many practical tasks \shortcite{Brockman2016OpenAI,vinyals2017starcraft,vinyals2019grandmaster,berner2019dota,ye2020towards} containing specific goal states rely heavily on reward design, especially by rewarding the agent upon its arrival on some {\em one-way} subgoals (each subgoal can only be visited once per episode) \shortcite{racaniere2017imagination,popov2017data,vinyals2017starcraft,vinyals2019grandmaster,berner2019dota,ye2020towards}. \shortciteA{wen2020efficiency} introduced a framework that decomposes the original MDP into ``subMDPs'', and demonstrated that decomposing an MDP into subMDPs indeed leads to statistical and computational efficiency, under certain assumptions. \shortciteA{wen2020efficiency} provides an insightful framework for studying subgoals from a hierarchical RL perspective, but it does not clearly address how rewarding subgoals helps the goal-reaching tasks. To better understand the reward design for goal-reaching tasks, we consider the one-way subgoals as one-way intermediate states (the non-terminal states where the agent receives intermediate rewards upon its arrival). Then we show that under some {\em practically verifiable assumptions} on the intermediate states and simple conditions on the intermediate rewards, rewarding the agent on intermediate states is indeed more computationally efficient for learning a {\em successful policy} that {\em reaches the goal states} in terms of synchronous value iteration (SVI).

\subsection{Motivating Examples}
\label{subsec:MotivatingExamples}
We first start with two examples (the Maze and the Pacman Game in Figure \ref{fig:MotivatingExamples}) to reveal the existence of one-way intermediate states in different goal-reaching tasks. For the Maze problem, suppose one episode ends when the agent reaches the goal, then the Maze problem does not have one-way intermediate states, since all previously visited states can be revisited in one episode. In contrast, the Pacman game intrinsically possesses one-way intermediate states, because the Pacman cannot revisit the states where the previously consumed food pellets are available.

We conduct a toy experiment on the Pacman game to demonstrate how the design of {\em intermediate states} and {\em intermediate rewards} affects the behavior of an agent under the greedy policy \shortcite{sutton2018reinforcement}. In the Pacman game, the agent (Pacman) wins when it consumes all the food before being caught by the ghost. Intuitively, if we want the Pacman to win the game, we need to design the MDP such that the agent receives positive rewards for winning, consuming food, surviving (not being caught by the ghost), and negative rewards for losing (being caught by the ghost). With the aforementioned intuition, we design several different reward functions for the Pacman game. Comparing reward settings (1), (2), and (3) in Table \ref{tab:experiments-1g}, the Pacman performs better when the MDP contains positive intermediate rewards for consuming food (visiting one-way intermediate states). This observation matches the common belief that integrating prior knowledge of the task into the reward design is generally helpful. However, designing rewards based on prior knowledge could sometimes negatively affect the performance. In setting (4), where the Pacman receives a positive reward for survival (visiting non one-way intermediate states), the Pacman focuses more on dodging the ghost rather than consuming food, and eventually gets caught by the ghost. At this point, a natural question is:

\begin{center}
    {\em How does the design of intermediate rewards affect the behavior of a greedy policy for goal-reaching tasks?}
\end{center}

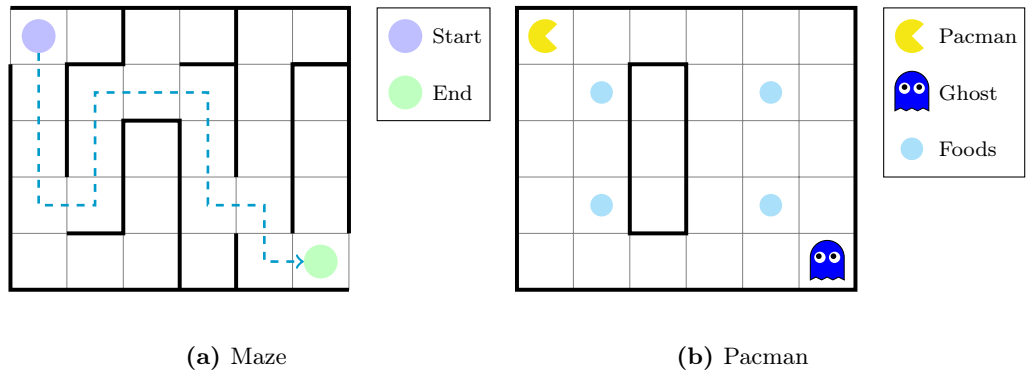
\begin{figure}[ht]
     \begin{subfigure}[b]{0.4\textwidth}
         \centering
            \begin{tikzpicture}[scale=0.75]
\centering
% Grid [0]
\draw[step=1cm, gray, very thin] (0, -5) grid (6, 0);
\fill[blue!25!white] (0.5, -0.5) circle (0.3cm);
\fill[green!25!white] (5.5, -4.5) circle (0.3cm);
% draw boundary
\draw[black, line width=1.5] (0.0, 0.0) -- (6.0, 0.0) -- (6.0, -4.0);
\draw[black, line width=1.5] (6.0, -5.0) -- (0.0, -5.0) -- (0.0, -1.0);
\draw[black, line width=1.5] (2.0, 0.0) -- (2.0, -1.0) -- (1.0, -1.0) -- (1.0, -3.0);
\draw[black, line width=1.5] (1.0, -4.0) -- (2.0, -4.0) -- (2.0, -2.0) -- (3.0, -2.0) -- (3.0, -5.0);
\draw[black, line width=1.5] (4.0, 0.0) -- (4.0, -3.0);
\draw[black, line width=1.5] (3.0, -1.0) -- (4.0, -1.0);
\draw[black, line width=1.5] (5.0, -1.0) -- (6.0, -1.0);
\draw[black, line width=1.5] (6.0, -1.0) -- (5.0, -1.0) -- (5.0, -4.0);
\draw[black, line width=1.5] (4.0, -4.0) -- (4.0, -5.0);

% draw paths
\draw[cyan!80!black, line width=1.0, dashed, ->] (0.5, -0.8) -- (0.5, -3.5) -- (1.5, -3.5) -- (1.5, -1.5) -- (3.5, -1.5) -- (3.5, -3.5) -- (4.5, -3.5) -- (4.5, -4.5) -- (5.2, -4.5);

% legend
Bounding Box
\draw[black] (6.5,0) -- (6.5,-2) -- (8.5,-2) -- (8.5,0) -- cycle;

% Drawing
\fill[blue!25!white] (7.0, -0.5) circle (0.3cm);
\fill[green!25!white] (7.0, -1.5) circle (0.3cm);
% Caption
\node [right, font=\scriptsize] (A) at (7.3,-0.5) {Start};
\node [right, font=\scriptsize] (B) at (7.3,-1.5) {End};
\end{tikzpicture}
         \caption{Maze}
         \label{fig:maze}
     \end{subfigure}
     \hspace{1em}
     \begin{subfigure}[b]{0.4\textwidth}
         \centering
            \begin{tikzpicture}[scale=0.75]
\centering
% Grid [0]
\draw[step=1cm, gray, very thin] (0, -5) grid (6, 0);
% \fill[blue!25!white] (0.5, -0.5) circle (0.3cm);
\fill[cyan!30!white] (1.5, -1.5) circle (0.2cm);
\fill[cyan!30!white] (1.5, -3.5) circle (0.2cm);
\fill[cyan!30!white] (4.5, -1.5) circle (0.2cm);
\fill[cyan!30!white] (4.5, -3.5) circle (0.2cm);
\fill[fill=yellow!95!black,rotate around={45:(0.5,-0.5)}] (0.5,-0.5) -- (0.8,-0.5) arc (0:270:0.3cm);
\begin{scope}[shift={(5.2,-4.925)}]
\draw [fill=blue] (0,0.1) -- (0,0.5) arc (+180:0:0.3) -- (0.6,0.1) --
  (0.5,0.15) -- (0.4,0.1) -- (0.3,.15) -- (0.2,0.1) -- (0.1,.15) -- cycle;
    % \coordinate (eye) at (360*rand:0.03);
    \foreach \x in {0.17,0.43}{
      \fill[white] (\x,0.5) circle[radius=0.1];
      \fill[black] (\x-0.02,0.52) circle[radius=0.05];
    }    
\end{scope}

% draw boundary
\draw[black, line width=1.5] (0.0, 0.0) -- (6.0, 0.0) -- (6.0, -5.0) -- (0.0, -5.0) -- cycle;
\draw[black, line width=1.5] (2.0, -1.0) -- (3.0, -1.0) -- (3.0, -4.0) -- (2.0, -4.0) -- cycle;

%% legend
% Bounding Box
\draw[black] (6.5,0) -- (9.0,0) -- (9.0,-3) -- (6.5,-3) -- cycle;
\node [right, font=\scriptsize] (A) at (7.3,-0.5) {Pacman};
\node [right, font=\scriptsize] (B) at (7.3,-1.5) {Ghost};
\node [right, font=\scriptsize] (C) at (7.3,-2.5) {Foods};
\fill[fill=yellow!95!black,rotate around={45:(7.0,-0.5)}] (7.0,-0.5) -- (7.3,-0.5) arc (0:270:0.3cm);
\begin{scope}[shift={(6.7,-1.925)}]
\draw [fill=blue] (0,0.1) -- (0,0.5) arc (+180:0:0.3) -- (0.6,0.1) --
  (0.5,0.15) -- (0.4,0.1) -- (0.3,.15) -- (0.2,0.1) -- (0.1,.15) -- cycle;
    \foreach \x in {0.17,0.43}{
      \fill[white] (\x,0.5) circle[radius=0.1];
      \fill[black] (\x,0.52) circle[radius=0.05];
    }    
\end{scope}
\fill[cyan!30!white] (7.0, -2.5) circle (0.2cm);
\end{tikzpicture}
         \caption{Pacman}
         \label{fig:Pacman}
     \end{subfigure}
    \centering
    \caption{Examples of two practical tasks. The thick curves represent walls. The agent can take actions from \{left, right, up, down\} and move to the corresponding adjacent locations. The agent stays at the original location if it hits a wall after taking an action. (a): In the maze problem, each state in the MDP is the agent's position. The goal is to find the end from the start. The shortest path is marked with blue dashed curves. (b): In the Pacman game, each state contains the position of the Pacman, the ghost, and all remaining food pellets. The agent (Pacman) needs to consume all the pellets while avoiding the ghost, and each food pellet can only be eaten {\em once}. The ghost moves simultaneously with the Pacman. The Pacman wins once all pellets are consumed and loses if it is caught by the ghost. We use the TikZ code from \url{https://gist.github.com/neic/9546556} to draw the ghost.} 
    \label{fig:MotivatingExamples}
\end{figure}\begin{table}[ht]
    \centering
    \small
    \begin{tabular}{rrr|rrrrrrr}
        \toprule
        &\multicolumn{1}{c}{$r_f$}
        &\multicolumn{1}{c}{$r_s$}&
        \multicolumn{1}{c}{0 ep} & \multicolumn{1}{c}{100 eps} & \multicolumn{1}{c}{500 eps} & \multicolumn{1}{c}{1000 eps} & \multicolumn{1}{c}{1500 eps} & \multicolumn{1}{c}{2000 eps} & \multicolumn{1}{c}{2500 eps} \\
        \toprule
        (1) & 0 & 0 & 0.0\% & 0.1\% & 0.1\% & 0.6\% & 89.5\% & 93.7\% & 96.0\% \\
        (2) & 1 & 0 & 0.0\% & 0.5\% & 33.0\% & 78.0\% & 92.6\% & 97.8\% & 98.4\% \\
        (3) & 10 & 0 & 0.0\% & 0.1\% & 57.8\% & 93.6\% & 95.7\% & 97.0\% & 97.9\%\\
        (4) & 1 & 1 & 0.0\% & 0.0\% & 1.9\% & 2.5\% & 2.2\% & 7.7\% & 12.9\%\\
        \bottomrule
    \end{tabular}
    \vspace{0.5em}
    \caption{The win rate (averaged among 1000 trials) of the Pacman game shown in Figure \ref{fig:Pacman} after $\{0, 100, 500, 1000, 1500, 2000, 2500 \}$ Q-learning training episodes under a greedy policy. We compare the performance under 4 different reward function designs. In all settings, the winning reward is 10 and the losing reward is -10. $r_f$ is the reward of consuming a food pellet. The Pacman receives a survival reward $r_s$ if it is not caught by the ghost after taking an action. Our implementation of the Pacman is based on the code from Berkeley CS188 (\url{https://inst.eecs.berkeley.edu/~cs188/ fa18/project3.html}). Details of the experiments are provided in Appendix \ref{Exp:MotivatingExpDetail}.}
    \label{tab:experiments-1g}
\end{table}

\subsection{Our Formulation and Main Results}

\subsubsection{One-Way Intermediate States}
To study the conditions under which goal-reaching problems with one-way intermediate states enjoy computational benefits, we consider two one-way intermediate state settings as shown in Figure \ref{fig:IRSetting}: 1) one-way single-path (OWSP), and 2) one-way multi-path (OWMP). In both settings, we assume the existence of one-way intermediate states (formally defined in Assumption \ref{assump:OneWayIntermediateStates}) that behave like ``one-way'' checkpoints that {\em cannot be revisited} in one episode and intermediate rewards are assigned to the arrival at such states. Many practical RL tasks implicitly adopt this one-way property, as practitioners often identify subgoals and assign a one-time reward for their completion in solving challenging tasks in addition to the terminal rewards which occur once the agent reaches the goal states \cite{Brockman2016OpenAI,vinyals2017starcraft,vinyals2019grandmaster,berner2019dota,ye2020towards}. For the case where intermediate states are non one-way, we provide an example where the agent gets stuck at an intermediate state permanently instead of pursing the goal, if the intermediate rewards are not ``properly'' designed. As we have observed in setting (4) of the Pacman game in Table \ref{tab:experiments-1g}, improperly designed intermediate rewards can negatively affect the agent's ability to find a successful policy that reaches the goal.

\begin{figure}[ht]
    \centering
    \begin{subfigure}[t]{0.78\textwidth}
        \centering
            \begin{tikzpicture}[scale=0.7]
\centering
% draw start and end
% s0 -> si1
\draw[fill=white] (0.5, 0.5) circle (0.3cm);
\node [below,font=\footnotesize] at (0.5,0) {$s_0$};
\draw[black] (1, 0.6) -- (1.8, 1.2);
\draw[black] (1, 0.4) -- (1.8, -0.2);
\node [font=\footnotesize] at (2.5,1.2) {$\dots$};
\node [font=\footnotesize] at (2.5,0.65) {$\vdots$};
\node [font=\footnotesize] at (2.5,-0.2) {$\dots$};
\draw[->,black] (3.2, 1.2) -- (4, 0.6);
\draw[->,black] (3.2, -0.2) -- (4, 0.4);
\draw[fill=white] (4.5, 0.5) circle (0.3cm);
\node [below,font=\footnotesize] at (4.5,0) {$s_{i_1}$};

% -> si2
\draw[black] (5, 0.6) -- (5.8, 1.2);
\draw[black] (5, 0.4) -- (5.8, -0.2);
\node [font=\footnotesize] at (6.5,1.2) {$\dots$};
\node [font=\footnotesize] at (6.5,0.65) {$\vdots$};
\node [font=\footnotesize] at (6.5,-0.2) {$\dots$};
\draw[->,black] (7.2, 1.2) -- (8, 0.6);
\draw[->,black] (7.2, -0.2) -- (8, 0.4);

\draw[fill=white] (8.5, 0.5) circle (0.3cm);
\node [below,font=\footnotesize] at (8.5,0) {$s_{i_2}$};

% -> ... -> siN
\node [font=\footnotesize] at (10.5,0.5) {$\dots$};
\draw[fill=white] (12.5, 0.5) circle (0.3cm);
\node [below,font=\footnotesize] at (12.5,0) {$s_{i_N}$};
\draw[->,black] (9, 0.5) -- (9.8, 0.5);
\draw[->,black] (11.2, 0.5) -- (12, 0.5);

% -> sT+
\draw[black] (13.0, 0.6) -- (13.8, 1.2);
\draw[black] (13.0, 0.4) -- (13.8, -0.2);
\node [font=\footnotesize] at (14.5,1.2) {$\dots$};
\node [font=\footnotesize] at (14.5,0.65) {$\vdots$};
\node [font=\footnotesize] at (14.5,-0.2) {$\dots$};
\draw[fill=white] (16.5, 0.5) circle (0.3cm);
\draw[fill=white] (16.5, 0.5) circle (0.2cm);
\node [below,font=\footnotesize] at (16.5,0) {$\mc S_{T}$};
\draw[->,black] (15.2, 1.2) -- (16.0, 0.6);
\draw[->,black] (15.2, -0.2) -- (16.0, 0.4);
\end{tikzpicture}  
        \caption{In this case, all intermediate states $\mc S_I = \{s_{i_1},s_{i_2},\dots,s_{i_N}\}$ have to be visited in a certain order.}
        \label{fig:setting1}
    \end{subfigure}
    ~
    \begin{subfigure}[t]{0.78\textwidth}
        \centering
            \begin{tikzpicture}[scale=0.7]
\centering
% draw start and end
% s0 -> sj1
\draw[fill=white] (0.5, 0.5) circle (0.3cm);
\node [below,font=\footnotesize] at (0.5,0) {$s_0$};
\draw[black] (1, 0.6) -- (1.8, 1.5);
\draw[black] (1, 0.4) -- (1.8, -0.5);
\node [font=\footnotesize] at (2.5,1.5) {$\dots$};
\node [font=\footnotesize] at (2.5,0.65) {$\vdots$};
\node [font=\footnotesize] at (2.5,-0.5) {$\dots$};
\draw[->,black] (3.2, 1.5) -- (4, 1.5);
\draw[->,black] (3.2, -0.5) -- (4, -0.5);
\draw[fill=white] (4.5, 1.5) circle (0.3cm);
\node [below,font=\footnotesize] at (4.5,1.0) {$s_{i_{j_1}}$};
\draw[fill=white] (4.5, -0.5) circle (0.3cm);
\node [below,font=\footnotesize] at (4.5,-1.0) {$s_{i_{j^\prime_1}}$};

% -> sj2
\draw[black] (5, 1.6) -- (5.8, 2.2);
\draw[black] (5, 1.4) -- (5.8, 0.8);
\node [font=\footnotesize] at (6.5,2.2) {$\dots$};
\node [font=\footnotesize] at (6.5,1.65) {$\vdots$};
\node [font=\footnotesize] at (6.5,0.8) {$\dots$};
\draw[->,black] (7.2, 2.2) -- (8, 1.6);
\draw[->,black] (7.2, 0.8) -- (8, 1.4);

\draw[fill=white] (8.5, 1.5) circle (0.3cm);
\node [below,font=\footnotesize] at (8.5,1.0) {$s_{i_{j_2}}$};

\draw[black] (5, -0.4) -- (5.8, 0.2);
\draw[black] (5, -0.6) -- (5.8, -1.2);
\node [font=\footnotesize] at (6.5,0.2) {$\dots$};
\node [font=\footnotesize] at (6.5,-0.35) {$\vdots$};
\node [font=\footnotesize] at (6.5,-1.2) {$\dots$};
\draw[->,black] (7.2, 0.2) -- (8, -0.4);
\draw[->,black] (7.2, -1.2) -- (8, -0.6);

\draw[fill=white] (8.5, -0.5) circle (0.3cm);
\node [below,font=\footnotesize] at (8.5,-1.0) {$s_{i_{j^\prime_2}}$};

% -> ... -> sn
\draw[black] (9, 1.5) -- (9.8, 1.5);
\draw[->,black] (11.2, 1.5) -- (12, 1.5);
\node [font=\footnotesize] at (10.5,1.5) {$\dots$};
\draw[fill=white] (12.5, 1.5) circle (0.3cm);
\node [below,font=\footnotesize] at (12.5,1.0) {$s_{i_{j_m}}$};

\node [font=\footnotesize] at (10.5,0.65) {$\vdots$};

\draw[black] (9, -0.5) -- (9.8, -0.5);
\draw[->,black] (11.2, -0.5) -- (12, -0.5);
\node [font=\footnotesize] at (10.5,-0.5) {$\dots$};
\draw[fill=white] (12.5, -0.5) circle (0.3cm);
\node [below,font=\footnotesize] at (12.5,-1.0) {$s_{i_{j^\prime_{m^\prime}}}$};

% -> sT+
\draw[black] (13.0, 1.5) -- (13.8, 1.5);
\node [font=\footnotesize] at (14.5,1.5) {$\dots$};
\node [font=\footnotesize] at (14.5,0.65) {$\vdots$};
\draw[black] (13.0, -0.5) -- (13.8, -0.5);
\node [font=\footnotesize] at (14.5,-0.5) {$\dots$};
\node [below,font=\footnotesize] at (16.5,0) {$\mc S_{T}$};
\draw[->,black] (15.2, 1.5) -- (16.0, 0.6);
\draw[->,black] (15.2, -0.5) -- (16.0, 0.4);
\draw[fill=white] (16.5, 0.5) circle (0.3cm);
\draw[fill=white] (16.5, 0.5) circle (0.2cm);

\end{tikzpicture}  
        \caption{In this case, only a subset of intermediate states $\mc S_J = \{s_{i_{j_1}},s_{i_{j_2}},\dots,s_{i_{j_m}}\}$ (or $\mc S_{J^\prime} = \{s_{i_{j^\prime_1}},s_{i_{j^\prime_2}},\dots,s_{i_{j^\prime_{m^\prime}}}\}$) with at least $n$ states ($m,m^\prime\geq n$) have to be visited in a certain order.}
        \label{fig:setting2}
    \end{subfigure}
    \caption{The transition graph of different settings in Assumption \ref{assump:diff_settings_IS}}
    \label{fig:IRSetting}
\end{figure}
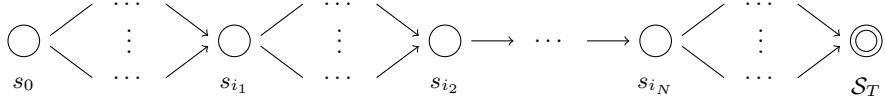
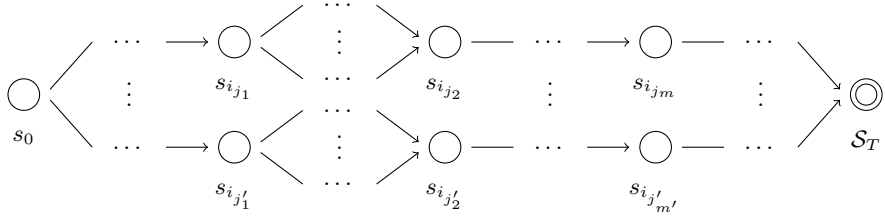

\subsubsection{Main Results}
We focus on comparing the two following reward settings of {\em deterministic MDPs}:
\begin{itemize}
    \item {\bf The sparse reward setting}: the agent only receives terminal rewards $B$ while reaching the goal (terminal states);
    \item {\bf The intermediate reward settings}: the agent receives {\em equal magnitude} intermediate rewards $B_I$ upon the arrival of intermediate states, for both OSWP and OWMP intermediate state settings.
\end{itemize}
Note that the intermediate rewards $B_I$ could be different from the terminal rewards $B$, and we formally define the reward functions in Table \ref{tab:reward_setting} in Section \ref{subsec:rewards}.

\paragraph{The OWSP Intermediate Reward Setting} 
Intuitively, the OWSP setting, which rewards the one-way intermediate states in addition to terminal rewards, is more computationally efficient than only having terminal rewards alone, because the OWSP setting reduces the problem from {\em finding the terminal states} to {\em finding the closest intermediate states}. For example, considering the environment provided in Figure \ref{fig:SparseRewardGrid}, the number of SVI required to obtain a greedy successful policy using sparse terminal rewards is 8 (the {\em distance} or minimum required steps from the initial state $s_0$ to the terminal state $\mc S_T$), whereas the OWSP intermediate reward setting (Figure \ref{fig:OWSPIRGrid}) reduces such computational complexity to 3 (the maximum distance between two intermediate states). Intuitively, Figure \ref{fig:OWSPIRGrid} suggests that the computational complexity of finding a successful policy can be reduced by rewarding the one-way intermediate states. The computational complexity for obtaining a successful policy (in terms of SVI) with sparse terminal rewards and the OWSP intermediate reward settings are formalized in Proposition \ref{prop:comp_sparse_reward} and \ref{prop:com_comp_inter_reward}, respectively. Moreover, Proposition \ref{prop:comp_sparse_reward} and \ref{prop:com_comp_inter_reward} show that a greedy policy follows the shortest path to $\mc S_T$ in both sparse terminal rewards and OWSP setting. 

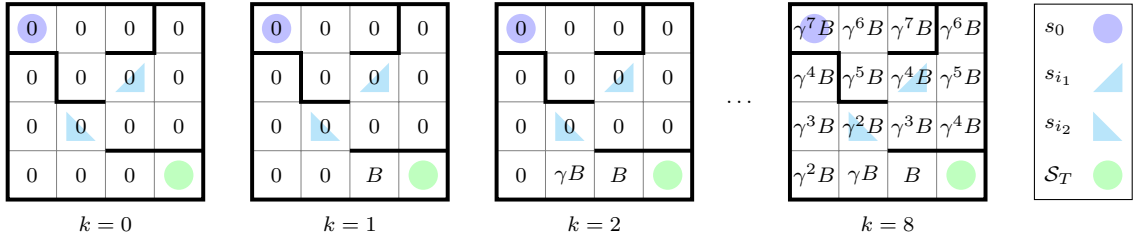
\begin{figure}[ht]
    \centering
    \begin{tikzpicture}[scale=0.65]
\centering
% Grid [0]
\draw[step=1cm, gray, very thin] (0, -4) grid (4, 0);
\fill[blue!25!white] (0.5, -0.5) circle (0.3cm);
\fill[green!25!white] (3.5, -3.5) circle (0.3cm);
% draw IS
\path[fill=cyan!25!white] (2.8,-1.2) -- (2.8,-1.8) -- (2.2,-1.8) -- cycle;
\path[fill=cyan!25!white] (1.2,-2.2) -- (1.8,-2.8) -- (1.2,-2.8) -- cycle;
% Fill value
\node [font=\scriptsize] (a11) at (0.5,-0.5) {$0$};
\node [font=\scriptsize] (a12) at (1.5,-0.5) {$0$};
\node [font=\scriptsize] (a13) at (2.5,-0.5) {$0$};
\node [font=\scriptsize] (a14) at (3.5,-0.5) {$0$};
\node [font=\scriptsize] (a21) at (0.5,-1.5) {$0$};
\node [font=\scriptsize] (a22) at (1.5,-1.5) {$0$};
\node [font=\scriptsize] (a23) at (2.5,-1.5) {$0$};
\node [font=\scriptsize] (a24) at (3.5,-1.5) {$0$};
\node [font=\scriptsize] (a31) at (0.5,-2.5) {$0$};
\node [font=\scriptsize] (a32) at (1.5,-2.5) {$0$};
\node [font=\scriptsize] (a33) at (2.5,-2.5) {$0$};
\node [font=\scriptsize] (a34) at (3.5,-2.5) {$0$};
\node [font=\scriptsize] (a41) at (0.5,-3.5) {$0$};
\node [font=\scriptsize] (a42) at (1.5,-3.5) {$0$};
\node [font=\scriptsize] (a43) at (2.5,-3.5) {$0$};
\node [font=\scriptsize] (a44) at (3.5,-3.5) {};
% draw k = 1
\node [font=\scriptsize] (a) at (2.0,-4.5) {$k=0$};
% draw boundary
\draw[black, line width=1.5] (0.0, 0.0) -- (4.0, 0.0) -- (4.0, -4.0) -- (0.0, -4.0) -- cycle;
\draw[black, line width=1.5] (3.0, 0.0) -- (3.0, -1.0) -- (2.0, -1.0);
\draw[black, line width=1.5] (0.0, -1.0) -- (1.0, -1.0) -- (1.0, -2.0) -- (2.0, -2.0);
\draw[black, line width=1.5] (2.0, -3.0) -- (4.0, -3.0);

% Grid [1] (x+5)
\draw[step=1cm, gray, very thin] (5, -4) grid (9, 0);
\fill[blue!25!white] (5.5, -0.5) circle (0.3cm);
\fill[green!25!white] (8.5, -3.5) circle (0.3cm);
% draw IS
\path[fill=cyan!25!white] (7.8,-1.2) -- (7.8,-1.8) -- (7.2,-1.8) -- cycle;
\path[fill=cyan!25!white] (6.2,-2.2) -- (6.8,-2.8) -- (6.2,-2.8) -- cycle;
% Fill value
\node [font=\scriptsize] (b11) at (5.5,-0.5) {$0$};
\node [font=\scriptsize] (b12) at (6.5,-0.5) {$0$};
\node [font=\scriptsize] (b13) at (7.5,-0.5) {$0$};
\node [font=\scriptsize] (b14) at (8.5,-0.5) {$0$};
\node [font=\scriptsize] (b21) at (5.5,-1.5) {$0$};
\node [font=\scriptsize] (b22) at (6.5,-1.5) {$0$};
\node [font=\scriptsize] (b23) at (7.5,-1.5) {$0$};
\node [font=\scriptsize] (b24) at (8.5,-1.5) {$0$};
\node [font=\scriptsize] (b31) at (5.5,-2.5) {$0$};
\node [font=\scriptsize] (b32) at (6.5,-2.5) {$0$};
\node [font=\scriptsize] (b33) at (7.5,-2.5) {$0$};
\node [font=\scriptsize] (b34) at (8.5,-2.5) {$0$};
\node [font=\scriptsize] (b41) at (5.5,-3.5) {$0$};
\node [font=\scriptsize] (b42) at (6.5,-3.5) {$0$};
\node [font=\scriptsize] (b43) at (7.5,-3.5) {$B$};
\node [font=\scriptsize] (b44) at (8.5,-3.5) {};
% draw k = 2
\node [font=\scriptsize] (b) at (7.0,-4.5) {$k=1$};
% draw boundary
\draw[black, line width=1.5] (5.0, 0.0) -- (9.0, 0.0) -- (9.0, -4.0) -- (5.0, -4.0) -- cycle;
\draw[black, line width=1.5] (8.0, 0.0) -- (8.0, -1.0) -- (7.0, -1.0);
\draw[black, line width=1.5] (5.0, -1.0) -- (6.0, -1.0) -- (6.0, -2.0) -- (7.0, -2.0);
\draw[black, line width=1.5] (7.0, -3.0) -- (9.0, -3.0);

% Grid [2] (x+10)
\draw[step=1cm, gray, very thin] (10, -4) grid (14, 0);
\fill[blue!25!white] (10.5, -0.5) circle (0.3cm);
\fill[green!25!white] (13.5, -3.5) circle (0.3cm);
% draw IS
\path[fill=cyan!25!white] (12.8,-1.2) -- (12.8,-1.8) -- (12.2,-1.8) -- cycle;
\path[fill=cyan!25!white] (11.2,-2.2) -- (11.8,-2.8) -- (11.2,-2.8) -- cycle;
% Fill value
\node [font=\scriptsize] (c11) at (10.5,-0.5) {$0$};
\node [font=\scriptsize] (c12) at (11.5,-0.5) {$0$};
\node [font=\scriptsize] (c13) at (12.5,-0.5) {$0$};
\node [font=\scriptsize] (c14) at (13.5,-0.5) {$0$};
\node [font=\scriptsize] (c21) at (10.5,-1.5) {$0$};
\node [font=\scriptsize] (c22) at (11.5,-1.5) {$0$};
\node [font=\scriptsize] (c23) at (12.5,-1.5) {$0$};
\node [font=\scriptsize] (c24) at (13.5,-1.5) {$0$};
\node [font=\scriptsize] (c31) at (10.5,-2.5) {$0$};
\node [font=\scriptsize] (c32) at (11.5,-2.5) {$0$};
\node [font=\scriptsize] (c33) at (12.5,-2.5) {$0$};
\node [font=\scriptsize] (c34) at (13.5,-2.5) {$0$};
\node [font=\scriptsize] (c41) at (10.5,-3.5) {$0$};
\node [font=\scriptsize] (c42) at (11.5,-3.5) {$\gamma B$};
\node [font=\scriptsize] (c43) at (12.5,-3.5) {$B$};
\node [font=\scriptsize] (c44) at (13.5,-3.5) {};
% draw k = 2
\node [font=\scriptsize] (c) at (12.0,-4.5) {$k=2$};
% draw boundary
\draw[black, line width=1.5] (10.0, 0.0) -- (14.0, 0.0) -- (14.0, -4.0) -- (10.0, -4.0) -- cycle;
\draw[black, line width=1.5] (13.0, 0.0) -- (13.0, -1.0) -- (12.0, -1.0);
\draw[black, line width=1.5] (10.0, -1.0) -- (11.0, -1.0) -- (11.0, -2.0) -- (12.0, -2.0);
\draw[black, line width=1.5] (12.0, -3.0) -- (14.0, -3.0);

% Grid [3] (x+16)
\draw[step=1cm, gray, very thin] (16, -4) grid (20, 0);
\fill[blue!25!white] (16.5, -0.5) circle (0.3cm);
\fill[green!25!white] (19.5, -3.5) circle (0.3cm);
% draw IS
\path[fill=cyan!25!white] (18.8,-1.2) -- (18.8,-1.8) -- (18.2,-1.8) -- cycle;
\path[fill=cyan!25!white] (17.2,-2.2) -- (17.8,-2.8) -- (17.2,-2.8) -- cycle;
% Fill value
\node [font=\scriptsize] (d11) at (16.5,-0.5) {$\gamma^7 B$};
\node [font=\scriptsize] (d12) at (17.5,-0.5) {$\gamma^6 B$};
\node [font=\scriptsize] (d13) at (18.5,-0.5) {$\gamma^7 B$};
\node [font=\scriptsize] (d14) at (19.5,-0.5) {$\gamma^6 B$};
\node [font=\scriptsize] (d21) at (16.5,-1.5) {$\gamma^4 B$};
\node [font=\scriptsize] (d22) at (17.5,-1.5) {$\gamma^5 B$};
\node [font=\scriptsize] (d23) at (18.5,-1.5) {$\gamma^4 B$};
\node [font=\scriptsize] (d24) at (19.5,-1.5) {$\gamma^5 B$};
\node [font=\scriptsize] (d31) at (16.5,-2.5) {$\gamma^3 B$};
\node [font=\scriptsize] (d32) at (17.5,-2.5) {$\gamma^2 B$};
\node [font=\scriptsize] (d33) at (18.5,-2.5) {$\gamma^3 B$};
\node [font=\scriptsize] (d34) at (19.5,-2.5) {$\gamma^4 B$};
\node [font=\scriptsize] (d41) at (16.5,-3.5) {$\gamma^2 B$};
\node [font=\scriptsize] (d42) at (17.5,-3.5) {$\gamma B$};
\node [font=\scriptsize] (d43) at (18.5,-3.5) {$B$};
\node [font=\scriptsize] (d44) at (19.5,-3.5) {};
% draw k = 2
\node [font=\scriptsize] (d) at (18.0,-4.5) {$k=8$};
% draw boundary
\draw[black, line width=1.5] (16.0, 0.0) -- (20.0, 0.0) -- (20.0, -4.0) -- (16.0, -4.0) -- cycle;
\draw[black, line width=1.5] (19.0, 0.0) -- (19.0, -1.0) -- (18.0, -1.0);
\draw[black, line width=1.5] (16.0, -1.0) -- (17.0, -1.0) -- (17.0, -2.0) -- (18.0, -2.0);
\draw[black, line width=1.5] (18.0, -3.0) -- (20.0, -3.0);

% draw ...
\node [font=\scriptsize] (c) at (15,-2) {$\dots$};

%% legend

% Drawing
\fill[blue!25!white] (22.5, -0.5) circle (0.3cm);
\path[fill=cyan!25!white] (22.8,-1.2) -- (22.8,-1.8) -- (22.2,-1.8) -- cycle;
\path[fill=cyan!25!white] (22.2,-2.2) -- (22.8,-2.8) -- (22.2,-2.8) -- cycle;
\fill[green!25!white] (22.5, -3.5) circle (0.3cm);
% Caption
\node [right,font=\scriptsize] (A) at (21,-0.5) {$s_0$};
\node [right,font=\scriptsize] (B) at (21,-1.5) {$s_{i_1}$};
\node [right,font=\scriptsize] (C) at (21,-2.5) {$s_{i_2}$};
\node [right,font=\scriptsize] (D) at (21,-3.5) {$\mc S_{T}$};
% Bounding Box
\draw[black] (21,0) -- (21,-4) -- (23,-4) -- (23,0) -- cycle;

\end{tikzpicture}
    \caption{An example of the evolution of a zero-initialized value function via SVI in the one-way single-path (Assumption \ref{assump:diff_settings_IS} (a)) in sparse reward setting (Table \ref{tab:reward_setting}). Each block in the $4\times4$ grid represents a state, the green state is the terminal state, the blue state is the initial state, and the thick curves represent walls that the agent cannot pass through. The cyan triangles are the one-way intermediate states. The orientation of the apex of a given cyan triangle represents the direction of each intermediate state. Namely, $s_{i_1}$ can only be visited from the left, and $s_{i_2}$ can only be visited from the right, and neither $s_{i_1}$ nor $s_{i_2}$ can be revisited. The agent can choose any action from $\{\text{left, right, up, down}\}$, and the agent will stay in the same state if it takes an action that hits the wall. Note that the value function of the terminal state remains 0 because the MDP stops once the agent reaches $\mc S_T$, so for any terminal state $s_t\in \mc S_T$, $V_k(s_t)$ is never updated.}
    \label{fig:SparseRewardGrid}
\end{figure}
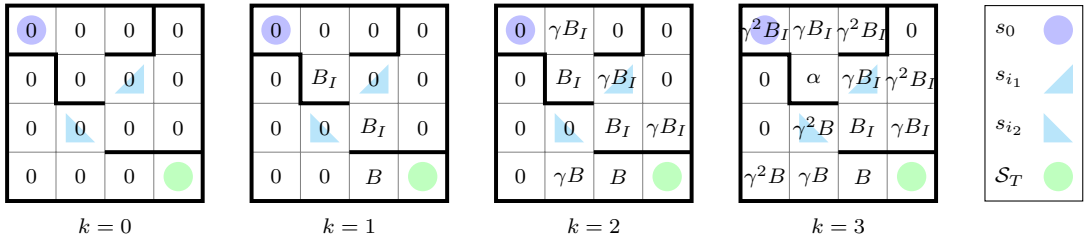
\begin{figure}[ht]
    \centering
    \begin{tikzpicture}[scale=0.65]
\centering
% Grid [0]
\draw[step=1cm, gray, very thin] (0, -4) grid (4, 0);
\fill[blue!25!white] (0.5, -0.5) circle (0.3cm);
\fill[green!25!white] (3.5, -3.5) circle (0.3cm);
% draw IS
\path[fill=cyan!25!white] (2.8,-1.2) -- (2.8,-1.8) -- (2.2,-1.8) -- cycle;
\path[fill=cyan!25!white] (1.2,-2.2) -- (1.8,-2.8) -- (1.2,-2.8) -- cycle;
% Fill value
\node [font=\scriptsize] (a11) at (0.5,-0.5) {$0$};
\node [font=\scriptsize] (a12) at (1.5,-0.5) {$0$};
\node [font=\scriptsize] (a13) at (2.5,-0.5) {$0$};
\node [font=\scriptsize] (a14) at (3.5,-0.5) {$0$};
\node [font=\scriptsize] (a21) at (0.5,-1.5) {$0$};
\node [font=\scriptsize] (a22) at (1.5,-1.5) {$0$};
\node [font=\scriptsize] (a23) at (2.5,-1.5) {$0$};
\node [font=\scriptsize] (a24) at (3.5,-1.5) {$0$};
\node [font=\scriptsize] (a31) at (0.5,-2.5) {$0$};
\node [font=\scriptsize] (a32) at (1.5,-2.5) {$0$};
\node [font=\scriptsize] (a33) at (2.5,-2.5) {$0$};
\node [font=\scriptsize] (a34) at (3.5,-2.5) {$0$};
\node [font=\scriptsize] (a41) at (0.5,-3.5) {$0$};
\node [font=\scriptsize] (a42) at (1.5,-3.5) {$0$};
\node [font=\scriptsize] (a43) at (2.5,-3.5) {$0$};
\node [font=\scriptsize] (a44) at (3.5,-3.5) {};
% draw k = 1
\node [font=\scriptsize] (a) at (2.0,-4.5) {$k=0$};
% draw boundary
\draw[black, line width=1.5] (0.0, 0.0) -- (4.0, 0.0) -- (4.0, -4.0) -- (0.0, -4.0) -- cycle;
\draw[black, line width=1.5] (3.0, 0.0) -- (3.0, -1.0) -- (2.0, -1.0);
\draw[black, line width=1.5] (0.0, -1.0) -- (1.0, -1.0) -- (1.0, -2.0) -- (2.0, -2.0);
\draw[black, line width=1.5] (2.0, -3.0) -- (4.0, -3.0);

% Grid [1] (x+5)
\draw[step=1cm, gray, very thin] (5, -4) grid (9, 0);
\fill[blue!25!white] (5.5, -0.5) circle (0.3cm);
\fill[green!25!white] (8.5, -3.5) circle (0.3cm);
% draw IS
\path[fill=cyan!25!white] (7.8,-1.2) -- (7.8,-1.8) -- (7.2,-1.8) -- cycle;
\path[fill=cyan!25!white] (6.2,-2.2) -- (6.8,-2.8) -- (6.2,-2.8) -- cycle;
% Fill value
\node [font=\scriptsize] (b11) at (5.5,-0.5) {$0$};
\node [font=\scriptsize] (b12) at (6.5,-0.5) {$0$};
\node [font=\scriptsize] (b13) at (7.5,-0.5) {$0$};
\node [font=\scriptsize] (b14) at (8.5,-0.5) {$0$};
\node [font=\scriptsize] (b21) at (5.5,-1.5) {$0$};
\node [font=\scriptsize] (b22) at (6.5,-1.5) {$B_I$};
\node [font=\scriptsize] (b23) at (7.5,-1.5) {$0$};
\node [font=\scriptsize] (b24) at (8.5,-1.5) {$0$};
\node [font=\scriptsize] (b31) at (5.5,-2.5) {$0$};
\node [font=\scriptsize] (b32) at (6.5,-2.5) {$0$};
\node [font=\scriptsize] (b33) at (7.5,-2.5) {$B_I$};
\node [font=\scriptsize] (b34) at (8.5,-2.5) {$0$};
\node [font=\scriptsize] (b41) at (5.5,-3.5) {$0$};
\node [font=\scriptsize] (b42) at (6.5,-3.5) {$0$};
\node [font=\scriptsize] (b43) at (7.5,-3.5) {$B$};
\node [font=\scriptsize] (b44) at (8.5,-3.5) {};
% draw k = 2
\node [font=\scriptsize] (b) at (7.0,-4.5) {$k=1$};
% draw boundary
\draw[black, line width=1.5] (5.0, 0.0) -- (9.0, 0.0) -- (9.0, -4.0) -- (5.0, -4.0) -- cycle;
\draw[black, line width=1.5] (8.0, 0.0) -- (8.0, -1.0) -- (7.0, -1.0);
\draw[black, line width=1.5] (5.0, -1.0) -- (6.0, -1.0) -- (6.0, -2.0) -- (7.0, -2.0);
\draw[black, line width=1.5] (7.0, -3.0) -- (9.0, -3.0);

% Grid [2] (x+10)
\draw[step=1cm, gray, very thin] (10, -4) grid (14, 0);
\fill[blue!25!white] (10.5, -0.5) circle (0.3cm);
\fill[green!25!white] (13.5, -3.5) circle (0.3cm);
% draw IS
\path[fill=cyan!25!white] (12.8,-1.2) -- (12.8,-1.8) -- (12.2,-1.8) -- cycle;
\path[fill=cyan!25!white] (11.2,-2.2) -- (11.8,-2.8) -- (11.2,-2.8) -- cycle;
% Fill value
\node [font=\scriptsize] (c11) at (10.5,-0.5) {$0$};
\node [font=\scriptsize] (c12) at (11.5,-0.5) {$\gamma B_I$};
\node [font=\scriptsize] (c13) at (12.5,-0.5) {$0$};
\node [font=\scriptsize] (c14) at (13.5,-0.5) {$0$};
\node [font=\scriptsize] (c21) at (10.5,-1.5) {$0$};
\node [font=\scriptsize] (c22) at (11.5,-1.5) {$B_I$};
\node [font=\scriptsize] (c23) at (12.5,-1.5) {$\gamma B_I$};
\node [font=\scriptsize] (c24) at (13.5,-1.5) {$0$};
\node [font=\scriptsize] (c31) at (10.5,-2.5) {$0$};
\node [font=\scriptsize] (c32) at (11.5,-2.5) {$0$};
\node [font=\scriptsize] (c33) at (12.5,-2.5) {$B_I$};
\node [font=\scriptsize] (c34) at (13.5,-2.5) {$\gamma B_I$};
\node [font=\scriptsize] (c41) at (10.5,-3.5) {$0$};
\node [font=\scriptsize] (c42) at (11.5,-3.5) {$\gamma B$};
\node [font=\scriptsize] (c43) at (12.5,-3.5) {$B$};
\node [font=\scriptsize] (c44) at (13.5,-3.5) {};
% draw k = 2
\node [font=\scriptsize] (c) at (12.0,-4.5) {$k=2$};
% draw boundary
\draw[black, line width=1.5] (10.0, 0.0) -- (14.0, 0.0) -- (14.0, -4.0) -- (10.0, -4.0) -- cycle;
\draw[black, line width=1.5] (13.0, 0.0) -- (13.0, -1.0) -- (12.0, -1.0);
\draw[black, line width=1.5] (10.0, -1.0) -- (11.0, -1.0) -- (11.0, -2.0) -- (12.0, -2.0);
\draw[black, line width=1.5] (12.0, -3.0) -- (14.0, -3.0);

% Grid [3] (x+15)
\draw[step=1cm, gray, very thin] (15, -4) grid (19, 0);
\fill[blue!25!white] (15.5, -0.5) circle (0.3cm);
\fill[green!25!white] (18.5, -3.5) circle (0.3cm);
% draw IS
\path[fill=cyan!25!white] (17.8,-1.2) -- (17.8,-1.8) -- (17.2,-1.8) -- cycle;
\path[fill=cyan!25!white] (16.2,-2.2) -- (16.8,-2.8) -- (16.2,-2.8) -- cycle;
% Fill value
\node [font=\scriptsize] (c11) at (15.5,-0.5) {$\gamma^2 B_I$};
\node [font=\scriptsize] (c12) at (16.5,-0.5) {$\gamma B_I$};
\node [font=\scriptsize] (c13) at (17.5,-0.5) {$\gamma^2 B_I$};
\node [font=\scriptsize] (c14) at (18.5,-0.5) {$0$};
\node [font=\scriptsize] (c21) at (15.5,-1.5) {$0$};
\node [font=\scriptsize] (c22) at (16.5,-1.5) {$\alpha$};
\node [font=\scriptsize] (c23) at (17.5,-1.5) {$\gamma B_I$};
\node [font=\scriptsize] (c24) at (18.5,-1.5) {$\gamma^2 B_I$};
\node [font=\scriptsize] (c31) at (15.5,-2.5) {$0$};
\node [font=\scriptsize] (c32) at (16.5,-2.5) {$\gamma^2 B$};
\node [font=\scriptsize] (c33) at (17.5,-2.5) {$B_I$};
\node [font=\scriptsize] (c34) at (18.5,-2.5) {$\gamma B_I$};
\node [font=\scriptsize] (c41) at (15.5,-3.5) {$\gamma^2 B$};
\node [font=\scriptsize] (c42) at (16.5,-3.5) {$\gamma B$};
\node [font=\scriptsize] (c43) at (17.5,-3.5) {$B$};
\node [font=\scriptsize] (c44) at (18.5,-3.5) {};
% draw k = 3
\node [font=\scriptsize] (c) at (17.0,-4.5) {$k=3$};
% draw boundary
\draw[black, line width=1.5] (15.0, 0.0) -- (19.0, 0.0) -- (19.0, -4.0) -- (15.0, -4.0) -- cycle;
\draw[black, line width=1.5] (18.0, 0.0) -- (18.0, -1.0) -- (17.0, -1.0);
\draw[black, line width=1.5] (15.0, -1.0) -- (16.0, -1.0) -- (16.0, -2.0) -- (17.0, -2.0);
\draw[black, line width=1.5] (17.0, -3.0) -- (19.0, -3.0);

%% legend
% Bounding Box
\draw[black] (20,0) -- (20,-4) -- (22,-4) -- (22,0) -- cycle;

% Drawing
\fill[blue!25!white] (21.5, -0.5) circle (0.3cm);
\path[fill=cyan!25!white] (21.8,-1.2) -- (21.8,-1.8) -- (21.2,-1.8) -- cycle;
\path[fill=cyan!25!white] (21.2,-2.2) -- (21.8,-2.8) -- (21.2,-2.8) -- cycle;
\fill[green!25!white] (21.5, -3.5) circle (0.3cm);
% Caption
\node [right,font=\scriptsize] (A) at (20,-0.5) {$s_0$};
\node [right,font=\scriptsize] (B) at (20,-1.5) {$s_{i_1}$};
\node [right,font=\scriptsize] (C) at (20,-2.5) {$s_{i_2}$};
\node [right,font=\scriptsize] (D) at (20,-3.5) {$\mc S_{T}$};
\end{tikzpicture}
    \caption{An example of the evolution of a zero-initialized value function via SVI in the one-way single-path (Assumption \ref{assump:diff_settings_IS} (a)) intermediate reward setting (Table \ref{tab:reward_setting}). $\alpha=B_I+\gamma^2 B_I$, and the remaining settings are the same as Figure \ref{fig:SparseRewardGrid}.}
    \label{fig:OWSPIRGrid}
\end{figure}

\paragraph{The OWMP Intermediate Reward Setting} Similar to the OWSP setting, the OWMP is also more computationally efficient than using sparse terminal rewards (in terms of SVI), provided the conditions allow a greedy policy to recursively find the closest intermediate state and eventually reach the terminal state. Theorem \ref{thm:find_closest_IS} provides the conditions that enable the agent to find the closest intermediate state when the terminal states are not {\em directly reachable} (formally defined in Definition \ref{def:DirectlyRechableIS}), and when the terminal states are directly reachable, Theorem \ref{thm:find_closest_TerminalState} characterize such conditions and the associated computational complexity of finding a successful policy. Nevertheless, unlike the OWSP setting or the sparse terminal reward setting, such a policy does not necessarily find the shortest path from $s_0$ to $\mathcal{S}_T$. Our result in the OWMP setting illustrates a trade-off between the computational complexity and the pursuit of shortest path --  adding intermediate rewards on one-way intermediate states generally reduces computational complexity, but it does not necessarily find the shortest path. 

\paragraph{Our Contributions} In this work, we make the following contributions:
\begin{enumerate}
    \item We propose a framework that studies the computational complexity for goal-reaching tasks in terms of synchronous value iteration (SVI);
    \item We propose two one-way intermediate rewards setting: the one-way single-path (OWSP) and the one-way multi-path intermediate reward settings (OWMP) and connect both settings to practical task;
    \item In both OWSP and OWMP settings, we demonstrate that assigning intermediate rewards on one-way intermediate states (in addition to terminal rewards) is more computationally efficient than the sparse reward setting where the agent only receives terminal rewards at the terminal states.
    \item We reveal a trade-off between the computational complexity and the pursuit of the shortest path in the OWMP intermediate settings versus the sparse reward setting. That is, comparing to the sparse reward setting where the agent only receives terminal rewards and follows the shortest path to the goal, rewarding the agent at one-way intermediate states (in addition to terminal states) generally reduces computational complexity but the agent does necessarily follow the shortest path.
    \item We conduct extensive experiments in both OWSP and OWMP settings to demonstrate the computational benefits of rewarding one-way intermediate states and OWMP's trade-off between computational complexity and the shortest path. In particular, even though our theoretical analysis adopts SVI, the computational gains and the trade-off appear in both traditional $\eps$-greedy Q-learning and popular deep RL methods such as DQN \shortcite{mnih2015human}, A2C \shortcite{mnih2016asynchronous}, and PPO \shortcite{schulman2017proximal}. 
\end{enumerate}

\subsection{Organization of this Paper}
The rest of this paper is organized as follows. In Section \ref{sec:Background}, we introduce the basics of MDPs and the definition of successful policies. We then provide the formal definition of intermediate states, the OWSP and OWMP intermediate states in Section \ref{subsec:InterState}. We then discuss how these assumptions are well reflected in RL tasks observed in practice (Section \ref{subsec:InterpretInterState}). In Section \ref{sec:MainResult}, we separate provide the computational complexity of finding a successful trajectory for both the OWSP and OWMP settings. In Section \ref{sec:FindShortestPath}, we discuss the connection between the trajectory generated by a greedy policy and the shortest path for each setting. We then provide experimental results in Section \ref{sec:Experiment} to corroborate our theoretical results in Section \ref{sec:MainResult}. In Section \ref{sec:Discussion}, we discuss the connection between our work and some prior works, the implication of our results, and potential future directions. The proof and experimental details are deferred to the appendix.

\section{Background and Problem Formulation}
\label{sec:Background}
\subsection{Assumptions of the MDP}
\label{subsec:BasicAssump}
\paragraph{Basics of MDPs} We use a quintuple $\mc M = (\mc S,\mc A, P, r, \gamma)$ to denote an MDP, where $\mc S$ is a finite state space, $\mc A$ is a finite action space, $\gamma \in (0,1)$ is the discount factor, $r:(\mc S\times \mc A)\times \mc S \mapsto [0,\infty)$ is the reward function of each state-action pair $(s,a)$ and its subsequent state $s_a$, and $P$ is the probability transition kernel of the MDP, where $P(s_a|s,a)$ denotes the probability of the subsequent state $s_a$ of a state-action pair $(s,a)$. In particular, we focus on {\em deterministic MDPs}: the transition kernel satisfies $\forall (s,a)\in \mc S\times\mc A$, $\exists s_a\in \mc S$, such that $P(s_a|s,a)=1$ and reward function $r(s,a,s_a)$ only depends on the subsequent state $s_a$. We say a state $s^\prime$ is {\em reachable} from $s$ if there exists a {\em path} or a sequence of actions  $\{a_1,a_2,\dots,a_n\}$ that takes the agent from $s$ to $s^\prime$.\footnote{A similar definition of reachability appears in \shortcite{forejt2011automated}.} Note that the definition of reachability {\em does not} necessarily imply $s$ is reachable from $s$, since there may not exist a path from $s$ to itself. Our model also considers a fixed {\em initial state} $s_0\in \mc S$ and {\em terminal states} $\mc S_T\subset \mc S$. Without further explanation, we assume the terminal state $\mc S_T$ is reachable from all $s\in \mc S$. All in all, MDP begins at state $s_0$ and stops once the agent reaches any state $s\in \mc S_T$.

\begin{definition}[Distance between States]
\label{def:distance}
Given an MDP $\mc M = (\mc S,\mc A, P, r, \gamma)$ with initial state $s_0\in \mc S$, $\forall s\in \mc S\backslash \mc S_T,s^\prime \in \mc S$, we define the distance $D(s,s^\prime): \mc S\times \mc S\mapsto \bb N$ as the minimum number of required steps $(\geq 1)$ from $s$ to $s^\prime$. We slightly abuse the notation of $D(\cdot,\cdot)$ by writing $D(s,\mc S_{T}) \doteq \underset{s^\prime \in \mc S_{T}}{\min} D(s,s^\prime)$ as the minimum distance from $s$ to $\mc S_{T}$.
\end{definition}

\subsection{Successful Q-Functions}
In the preceding deterministic MDP formulation, we aim at solving a
goal-reaching RL problem \shortcite{kaelbling1993learning,sutton2011horde,andrychowicz2017hindsight,andreas2017modular,pong2018temporal,ghosh2019learning,eysenbach2020rewriting,eysenbach2020c,kadian2020sim2real,fujita2020distributed,chebotar2021actionable,khazatsky2021can} or a planning problem \shortcite{bertsekas1996neuro,boutilier1999decision,sutton1999between,boutilier2000stochastic,rintanen2001overview,lavalle2006planning,russell2016artificial,nasiriany2019planning}. We say a Q-function is {\em successful} if its associated {\em greedy policy} \shortcite{sutton2018reinforcement} leads the agent to the terminal states $\mc S_T$ from the initial state $s_0$. 
\begin{definition}[Successful Q-functions]
    \label{def:SuccessfulQFunc}
    Given a deterministic MDP $\mc M = (\mc S,\mc A, P, r, \gamma)$, with initial state $s_0$. We say $Q(\cdot,\cdot)$ is a successful Q-function of $\mc M$ if the greedy policy with respect to $Q$ generates a path $\{s_0,a_0,s_1,a_1,\dots, s_H\}$ such that $\forall i = 0,1,\dots,H-1$, $s_i\notin \mc S_{T}$ and $s_H\in \mc S_{T}$.
\end{definition}
All theoretical results of this work are based on {\em synchronous value iteration}:
\begin{equation}    
    \begin{split}
        \label{eq:SyncValueIter}
        V_{k+1}(s)&\doteq\max_{a\in \mc A}\left\{Q_k(s,a)\right\},\;\forall s\in \mc S,\\
        Q_k(s,a) &= r(s,a,s_a)+\gamma V_k(s_a),\;\forall (s,a)\in \mc S\times \mc A.
    \end{split}
\end{equation}
The convergence of {\em synchronous} value iteration has been well studied \shortcite{puterman2014markov,sutton2018reinforcement,bertsekas2019reinforcement}, hence we use it as a starting point for studying the effect of intermediate states and intermediate rewards to be introduced later in Section \ref{sec:InterStateInterReward}.

\section{Intermediate States and Intermediate Rewards}
\label{sec:InterStateInterReward}

\subsection{Intermediate States}
\label{subsec:InterState}

We state the formal definition of intermediate states in Definition \ref{def:IntermediateStates}, and provide two follow-up assumptions (Assumption \ref{assump:OneWayIntermediateStates}, \ref{assump:diff_settings_IS}) regarding the intermediate states.

\begin{definition}[Intermediate States]
\label{def:IntermediateStates}
Given an MDP $\mc M = (\mc S,\mc A, P, r, \gamma)$ with initial state $s_0\in \mc S$ and terminal states $\mc S_{T}$, we define intermediate states $\mc S_I = \{s_{i_1},s_{i_2},\dots,s_{i_N}\}\subset \mc S\backslash \mc S_T$ as the states that satisfy  
\begin{equation}
    r(s,a,s_a)>0,\;\forall s_a\in \mc S_{I},
\end{equation}
where $s_a$ is the subsequent state of the state-action pair $(s,a)$.
\end{definition}

\begin{assumption}[One-Way Intermediate States]
\label{assump:OneWayIntermediateStates}
Given an MDP $\mc M = (\mc S,\mc A, P, r, \gamma)$ with initial state $s_0\in \mc S$, terminal states $\mc S_{T}$, and intermediate states $\mc S_I$, we assume that each intermediate state $s_{i_j}\in \mc S_{I}$ can only be visited {\em at most once} in one episode under any policy, namely, $\forall j\in [N],\;D(s_{i_j},s_{i_j})=\infty$.
\end{assumption}

Intuitively, Assumption \ref{assump:OneWayIntermediateStates} characterizes the states that ``cannot be revisited'' in one episode, upon the agent's arrival. For example, in the Pacman game (Figure \ref{fig:Pacman}), if the Pacman reaches a location $(x,y)$ that contains a food pellet, then the Pacman cannot go back to previous states where the consumed food pellet is still available at $(x,y)$. Similarly for the door \& key environment \shortcite{gym_minigrid} that will be presented in Section \ref{sec:Experiment}, once the agent picks up a key at location $(x,y)$, it possesses the key for the rest of that episode.
Assumption \ref{assump:OneWayIntermediateStates} is widely adopted in practice \cite{Brockman2016OpenAI,vinyals2017starcraft,vinyals2019grandmaster,berner2019dota,ye2020towards}, as many subgoals identified by designers are usually one-way (see more discussion in Section \ref{subsec:InterpretInterState}).

\begin{assumption}[Different Settings of Intermediate States]
\label{assump:diff_settings_IS}
Given an MDP $\mc M = (\mc S,\mc A, P, r, \gamma)$ with initial state $s_0\in \mc S$, intermediate states set $\mc S_I = \{s_{i_1},s_{i_2},\dots,s_{i_N}\}$ and terminal states $\mc S_{T}$ satisfying Assumption \ref{assump:OneWayIntermediateStates}, we study these different settings of $\mc S_I$:
\begin{enumerate}[(a)]
    \item Any path from $s_0$ to $\mc S_{T}$ has to visit all states in $\mc S_I$ in a certain order (i.e., in the order of $s_{i_1},s_{i_2},\dots,s_{i_N}$).
    \item Any path from $s_0$ to $\mc S_{T}$ has to visit at least $n$ intermediate states $\{s_{i_j}|s_{i_j}\in \mc S_I,j\in \mc J \subset[N], |\mc J| \geq n \}$ in a certain order.
\end{enumerate}
\end{assumption}

Assumptions \ref{assump:OneWayIntermediateStates} and \ref{assump:diff_settings_IS} characterize intermediate states as different one-way paths consisting of subgoals that the agent has to complete in a certain order. If we consider all paths from $s_0$ to $\mc S_T$ that pass though the same intermediate states as an equivalence class, the goal-reaching problem using setting (a) is a {\em  single-path problem}, while setting (b) is a {\em multi-path problem}. Additionally, Assumption \ref{assump:diff_settings_IS} indicates that from a given state $s$, not every intermediate state $s_i\in\mc S_I$ can be directly visited (without first visiting any other intermediate states). Thus, it is worth considering the {\em directly reachable} intermediate (and terminal) states from a given state $s$ and the {\em minimum distance} between two intermediate states.

\begin{definition}[Direct Reachability]
\label{def:DirectlyRechableIS}
Given an MDP $\mc M = (\mc S,\mc A, P, r, \gamma)$ with terminal states $\mc S_{T}$ and intermediate states $\mc S_I = \{s_{i_1},s_{i_2},\dots,s_{i_N}\}$ satisfying Assumption \ref{assump:OneWayIntermediateStates}, $\forall s \in \mc S\backslash\mc S_{T}$, let $\mc I_d(s)\subset [N]$ denote the indices of {\em directly reachable} intermediate states of $s$. That is, $\forall j\in \mc I_d(s)$, there exists a path from $s$ to the intermediate state $s_{i_j}$ in the transition graph of $\mc M$ that does not visit any other intermediate state.
\end{definition}
Similar to the intermediate states $\mc S_I$, we say the terminal states $\mc S_T$ are {\em directly reachable} from a state $s$ if there exists a path from $s$ to $\mc S_T$ that does not contain any intermediate states.
\begin{assumption}[Minimum Distance]
\label{assump:min_dist_IS}
Given an MDP $\mc M = (\mc S,\mc A, P, r, \gamma)$ with terminal states $\mc S_{T}$ and intermediate states $\mc S_I = \{s_{i_1},s_{i_2},\dots,s_{i_N}\}$ satisfying Assumption \ref{assump:OneWayIntermediateStates}, $\forall s\in \mc S\backslash\mc S_{T}$, we assume the distance between any two intermediate states is at least $h \in \bb N^+$, and the distance between any intermediate state and $\mc S_{T}$ is also at least $h$. Namely, $\forall s_{i_j}\in \mc S_{I}$, we have
\begin{equation}
    \min_{j^\prime \in \mc I_d(s_{i_j})} D(s_{i_j},s_{i_{j^\prime}}) \geq h,\; \text{ and } \min_{j \in [N]} D(s_{i_j},\mc S_{T}) \geq h,
\end{equation} 
where $\mc I_d(s_{i_j})$ is the set of indices of directly reachable intermediate states from $s_{i_j}$ (see Definition \ref{def:DirectlyRechableIS}).
\end{assumption}
In practice, although the minimum distance $h$ between two intermediate states (and an intermediate state to $\mc S_{T}$) is task dependent, it is generally fair to assume that $h$ satisfies $1<h<D(s_0,S_T)$.

The aforementioned assumptions (Assumption \ref{assump:OneWayIntermediateStates}, \ref{assump:diff_settings_IS}, and \ref{assump:min_dist_IS}) are connected to our main results in the following respects: 
\begin{enumerate}
    \item The one-way Assumption \ref{assump:OneWayIntermediateStates} and \ref{assump:diff_settings_IS} provide the key theoretical framework for our main result. The key insight of the computational gain in Section \ref{sec:MainResult} is that, if the MDP contains some one-way intermediate states, then one can reduce the computational complexity of finding the terminal states $\mc S_T$ to the complexity of {\em finding the closest intermediate states}, which eventually leads to $\mc S_T$. The detailed dependence of Assumption \ref{assump:OneWayIntermediateStates} and \ref{assump:diff_settings_IS} on our main result is provided in Table \ref{tab:assumption_requirement}. 
    \item We use Assumption \ref{assump:min_dist_IS} for quantitatively analyzing the relative scale of the terminal rewards versus the intermediate rewards that {\em leads the agent to the closet intermediate states} and eventually to $\mc S_T$.
\end{enumerate}
Assumption \ref{assump:OneWayIntermediateStates} and \ref{assump:diff_settings_IS} are highly task dependent because they characterize the connections of the one-way intermediate states of different tasks. In the next subsection, we will introduce some practical tasks that adopt Assumption \ref{assump:OneWayIntermediateStates} and \ref{assump:diff_settings_IS}.

\subsection{The Interpretation of Intermediate States}
\label{subsec:InterpretInterState}

We list several practical tasks that adopt Assumption \ref{assump:diff_settings_IS} in Table \ref{tab:app_example} to interpret Assumption \ref{assump:diff_settings_IS}. As illustrated in Table \ref{tab:app_example}, the multi-path intermediate state setting (Assumption \ref{assump:diff_settings_IS} (b)) is more common in practical tasks than the single-path setting (Assumption \ref{assump:diff_settings_IS} (a)), because the single-path setting is essentially a special case of the multi-path setting.

%%%%%%%%%%%%%%%%%%%%%%%%%%%%%%
% Table for the full version %
%%%%%%%%%%%%%%%%%%%%%%%%%%%%%%
\begin{wraptable}{r}{90mm}
    \centering
    \small
    \begin{tabular}{lcc}
        \toprule
        Assumption \ref{assump:diff_settings_IS} & (a) & (b)\\
         \midrule
         Maze (Figure \ref{fig:maze})& \xmark  &\xmark \\
         Pacman (Figure \ref{fig:Pacman})& \xmark & \cmark \\
         Montezuma (\shortcite{Brockman2016OpenAI})& \cmark &  \cmark \\
        %  MiniGrid (\shortcite{gym_minigrid})& \cmark & \cmark \\
         Go (\shortcite{silver2016mastering,silver2017mastering})& \xmark & \xmark \\
         Dota2 (\shortcite{berner2019dota}) & \xmark & \cmark\\
         StarCraft II (\shortcite{vinyals2017starcraft,vinyals2019grandmaster})& \xmark &  \cmark \\
         Honor of Kings (\shortcite{ye2020towards})& \xmark & \cmark \\
        \bottomrule
    \end{tabular}
    \caption{Examples of some RL applications that adopt different settings in Assumption \ref{assump:diff_settings_IS}.}
    \label{tab:app_example}
    \vspace{-1.2em}
\end{wraptable}

Assumption \ref{assump:diff_settings_IS} does not hold for the Maze (See Figure \ref{fig:maze}) because the agent can repeatedly visit any state in the maze. As for Go, Assumption \ref{assump:diff_settings_IS} also does not hold because the existence of subgoals in Go remains ambiguous. Except for the Maze problem and Go, Assumption \ref{assump:diff_settings_IS} (b) fits the others practical tasks in Table \ref{tab:app_example} naturally. Moreover, the stronger Assumption \ref{assump:diff_settings_IS} (a) holds for Montezuma because it requires the agent to visit specific states in a certain order, e.g., the agent will need to pick up a key and unlock a door to proceed to the next chapter. The Pacman game satisfies Assumption \ref{assump:diff_settings_IS} (b) -- considering each state in the Pacman game consists of the location of the Pacman, the ghost, and the remaining foods, then the states where ``the Pacman consumes a food pellet (the location of the Pacman first coincides with an available food pellet)'' can be viewed as a one-way intermediate state. Hence each episode of the Pacman game contains $n$ (the total number of food) intermediate states, and these states appear in the order in which the number of available food is decreasing. As for Dota2, StarCraft II, and Honor of Kings, their winning conditions require the agent to ``destroy'' the enemy's base. However, the enemy's base is not assailable before the agent completes several subtasks. For example, in Dota2 or Honor of Kings, 3 towers block each of the roads to the enemy's base. The agent must first sequentially destroy the 3 towers to reach the enemy's base.\footnote{Information on the rules and gameplay of Dota are available online; i.e., \url{https://purgegamers.true.io/g/dota-2-guide/}. The rule of Honor of Kings is similar to Dota2.} If one views the state where the agent destroys a tower in Dota2 or Honor of Kings as an intermediate state, Assumption \ref{assump:diff_settings_IS} (b) is naturally satisfied, because the agent need to visit at least 3 intermediate states to attack the enemy's base.

Note that even though some of the practical tasks (Pacman, Montezuma, Dota2, StarCraft II, and Honor of Kings) mentioned in Table \ref{tab:app_example} satisfy Assumption \ref{assump:diff_settings_IS}, our theoretical results to be introduced in Section \ref{sec:MainResult} may not be directly applied to these tasks; this is because they contain other factors that cannot be characterized by our current theoretical assumptions (e.g., the deterministic MDP assumption mentioned in Section \ref{subsec:BasicAssump}, the equal magnitude assumption to be mentioned in Section \ref{subsec:rewards}).

\subsection{Rewards}
\label{subsec:rewards}
We assume the reward function $r(s,a,s_a)$ of a state-action pair $(s,a)$ only depends on the subsequent state $s_a$. In particular, our theoretical results (Section \ref{sec:MainResult}) and experiments (Section \ref{sec:Experiment}) focus on comparing these two reward settings: the {\em sparse reward setting} where {\em there are no intermediate states in the environment} hence the agent only receives positive rewards at $\mc S_T$, and the {\em intermediate reward setting}, meaning that {\em the environment contains intermediate states} so the agent receives rewards at both $\mc S_T$ and $\mc S_I$. Formally, we write the reward functions of the sparse reward setting and the intermediate reward setting as follows:
\begin{table}[ht]
    \centering
    \small
    \begin{tabular}{r|cc|ccc}
        \toprule 
        \multicolumn{1}{c}{} & \multicolumn{2}{c}{Sparse Rewards} & \multicolumn{3}{c}{Intermediate Rewards}\\
        \midrule
        $r(s,a,s_a)$& $B$ & $0$ & $B$ & $B_I$ & 0\\
        Condition & $s_a\in \mc S_T$ & $s_a\in \mc S\backslash \mc S_T$ & $s_a\in \mc S_T$ & $s_a\in \mc S_I$ & $s_a\in \mc S\backslash (\mc S_T\cup \mc S_I)$ \\
        \bottomrule
    \end{tabular}
    \caption{The sparse reward setting and the intermediate reward setting.}
    \label{tab:reward_setting}
\end{table}

As shown in Table \ref{tab:reward_setting}, we assume the agent receives terminal rewards $B$ once it reaches $s_t\in \mc S_T$, and receives intermediate rewards $B_I$ once it reaches an intermediate state $s_i\in \mc S_I$. Note that we are {\em not claiming} that the applications provided in Table \ref{tab:app_example} satisfies the intermediate reward setting in Table \ref{tab:reward_setting} where all intermediate rewards have the same magnitude. In fact, the reward setting in Table \ref{tab:reward_setting} only matches the Pacman game in Table \ref{tab:app_example}, where we assign all food with equal intermediate rewards. Still, our theoretical results in Section \ref{sec:MainResult} relying on the equal magnitude of intermediate rewards can be further generalized to the case where intermediate rewards are different. We defer the discussion of such generalization to unequal magnitude of intermediate rewards in Section \ref{sec:Discussion} and leave the formal studies of the generalization for future.

\section{Main Results}
\label{sec:MainResult}

We study the conditions for both the sparse reward and intermediate reward setting, under which the Q-function $Q_k(\cdot,\cdot)$ is a successful Q-function (Definition \eqref{def:SuccessfulQFunc}), where $Q_k(\cdot,\cdot)$ is the Q-function after $k$ update of SVI \eqref{eq:SyncValueIter} from zero initialization: \begin{equation}
    \label{eq:zero_init}
    Q_0(s,a) = 0,\; V_0(s) = 0,\;\forall (s,a)\in \mc S\times\mc A.
\end{equation}
And if $Q_k(\cdot,\cdot)$ is indeed a successful Q-function, we further discuss the computational complexity $k$ (the minimum number of SVI) 
of obtaining a successful $Q_k(\cdot,\cdot)$.

For the sparse reward setting, we first show that for a large enough $k$, $Q_k(\cdot,\cdot)$ is a successful Q-function and provide its computational complexity in Section \ref{subsec:sparse_reward}. Next, we discuss the following MDP settings as shown in Table \ref{tab:reward_setting}.
\begin{table}[H]
    \centering
    \small
    \begin{tabular}{r|c|c}
        \toprule 
        \multicolumn{1}{c}{} & \multicolumn{1}{c}{Assumption \ref{assump:OneWayIntermediateStates}} & \multicolumn{1}{c}{Assumption \ref{assump:diff_settings_IS}}\\
        \midrule
        The One-Way Single-Path (OWSP) setting& \cmark & (a)\\
        \midrule
        The One-Way  Multi-Path (OWMP) setting& \cmark & (b)\\
        \midrule
        The Non One-Way (NOW) setting& \xmark & \xmark\\
        \bottomrule
    \end{tabular}
    \caption{Requirements of different assumptions for the main result.}
    \label{tab:assumption_requirement}
\end{table}

% \begin{enumerate}[(1)]
%     \item The One-Way Single-Path (OWSP) setting: Assumption \ref{assump:OneWayIntermediateStates} and Assumption \ref{assump:diff_settings_IS} (a) are satisfied;
%     \item The One-Way Multi-Path (OWMP) setting: Assumption \ref{assump:OneWayIntermediateStates} and Assumption \ref{assump:diff_settings_IS} (b) are satisfied;
%     \item The Non One-Way (NOW) setting: Assumption \ref{assump:OneWayIntermediateStates} is not satisfied. 
% \end{enumerate}

For the OWSP setting, we show that $Q_k(\cdot,\cdot)$ is a successful Q-function for sufficiently large $k$ and provide its computational complexity in Section \ref{subsec:OWSP_intermediate_states}. For the OWMP setting, we introduce a sufficient condition under which $Q_k(\cdot,\cdot)$ is successful for a large enough $k$, and the corresponding computational complexity in Section \ref{subsec:OWMP_intermediate_states}. As for the Non One-Way (NOW) setting, we provide an example where $Q_k(\cdot,\cdot)$ is not successful for any $k$ in Section \ref{subsec:NonOneWayIntermedaiteReward}. The required assumptions for the OWSP, OWMP, and NOW are provided in Table \ref{tab:assumption_requirement}.

\subsection{Sparse Rewards}
\label{subsec:sparse_reward}
Given zero-initialized value functions \eqref{eq:zero_init}, Figure \ref{fig:SparseRewardGrid} demonstrates the evolution of $V_k(s)$ as the number of SVI ($k$) increases. A direct implication of Figure \ref{fig:SparseRewardGrid} is that, at iteration $k$, $\forall s\in \mc S\backslash \mc S_T$, given $d = D(s,\mc S_T)$, the value function $V_k(\cdot)$ satisfies:
\begin{equation}
    \label{eq:SparseRewardValueMainText}
    V_k(s) = \gamma^{d-1}B\cdot \indicator{d\leq k}.
\end{equation}
The derivation of \eqref{eq:SparseRewardValueMainText} is provided in Lemma \ref{lemma:Vk_bounds_induction_SR}. With \eqref{eq:SparseRewardValueMainText}, we know that when $k\geq D(s_0,\mc S_T)$, the value function at the initial state $V_k(s_0)$ would be positive. Hence, a greedy policy that recursively finds the next state with the largest value from $s_0$, will generate a path which eventually reaches $\mc S_T$, whenever $k\geq D(s_0,\mc S_T)$. More precisely, we have:
\begin{proposition}[Sparse Rewards]
\label{prop:comp_sparse_reward}
Let $\mc M = (\mc S,\mc A, P, r, \gamma)$ be a deterministic MDP with initial state $s_0$ and terminal states $\mc S_{T}$. If the reward function $r(\cdot)$ follows the sparse reward setting (Table \ref{tab:reward_setting}) and the value function and Q-function are zero-initialized \eqref{eq:zero_init}, then after any $k\geq D(s_0,\mc S_{T})$  synchronous value iteration updates \eqref{eq:SyncValueIter}, the Q-function $Q_k$ is a successful Q-function, and a greedy policy follows the shortest path from $s_0$ to $\mc S_T$. 
\end{proposition}

We provide a sketch proof here and leave the details in Appendix \ref{proof:absorb_radius_Q_learning}. With \eqref{eq:SparseRewardValueMainText} from Lemma \ref{lemma:Vk_bounds_induction_SR}, we can write the value function in this setting as $V_k(s) = \gamma^{d-1}B\cdot\indicator{d\leq k}, d = D(s, \mathcal{S}_T)$. Hence a greedy policy taking the agent to the subsequent state with maximum value function will lead the agent one step closer to $\mc S_T$. Recursively applying the same argument, we conclude that after $k\geq D(s_0,\mc S_T)$ SVI, an agent following the greedy policy finds $\mc S_T$ from $s_0$.

Note that our result in the sparse reward setting (Proposition \ref{prop:comp_sparse_reward}) might seem similar to the classic convergence result for SVI via dynamical programming \cite{bertsekas1995dynamic,sutton2018reinforcement}. However, our setting only focuses on the conditions that allow a greedy to find $\mc S_T$ from the initial state $s_0$, rather than all state $s\in \mc S$. Therefore, the computational complexity in the sparse reward setting relies on the distance $D(s_0,\mc S_T)$. Furthermore, in the OWSP and OWMP settings to be introduced later, our proofs contain share the similar method as the sparse reward setting, hence Proposition \ref{prop:comp_sparse_reward} serves as a good preliminary result for the later parts.

\subsection{One-Way Single-Path  Intermediate Rewards}
\label{subsec:OWSP_intermediate_states}

Similar to the sparse reward setting, we first illustrate the evolution of the one-way single-path (OWSP) intermediate rewards in Figure \ref{fig:OWSPIRGrid}. As shown in Figure \ref{fig:OWSPIRGrid}, after $k$ synchronous iterations, the value function at each state $s$ equals to the sum of discounted rewards from all future intermediate states and the terminal states. The value function $V_k(s)$ of state $s$ at iteration $k$ is provided in Lemma \ref{lemma:Vk_bounds_induction_IR_a}. In this case, when $k$ satisfies
\begin{equation*}
    k \geq \max\{D(s_0,s_{i_1}),D(s_{i_1},s_{i_2}),\dots, D(s_{i_{N-1}},s_{i_N}), D(s_{i_N},\mc S_{T})\},
\end{equation*}
$Q_k(\cdot,\cdot)$ is a successful Q-function and a greedy policy will recursively find the next intermediate state, eventually reaching $\mc S_T$. More precisely, we have:

\begin{proposition}[Single-path Intermediate States]
\label{prop:com_comp_inter_reward}
Let $\mc M = (\mc S,\mc A, P, r, \gamma)$ be a deterministic MDP with initial state $s_0$, intermediate states $\mc S_{I} = \{s_{i_1},s_{i_2},\dots,s_{i_N}\}$, and terminal states $\mc S_{T}$. Suppose $\mc M$ satisfies Assumption \ref{assump:OneWayIntermediateStates} and \ref{assump:diff_settings_IS} (a). If the reward function $r(\cdot)$ follows the intermediate reward setting (Table \ref{tab:reward_setting}) and the value function and Q-function are zero-initialized \eqref{eq:zero_init}, then after
\begin{equation}
    k\geq d_{\max}\doteq\max\{D(s_0,s_{i_1}),D(s_{i_1},s_{i_2}),\dots, D(s_{i_{N-1}},s_{i_N}), D(s_{i_N},\mc S_{T})\}
\end{equation} 
synchronous value iteration updates \eqref{eq:SyncValueIter}, the Q-function $Q_k$ is a successful Q-function, and a greedy policy follows the shortest path from $s_0$ to $\mc S_T$. 
\end{proposition}
We defer the details to Appendix \ref{proof:complexity_inter_reward}. To outline, we provide an explicit formulation of $V_k(s)$ for this setting in Lemma \ref{lemma:Vk_bounds_induction_IR_a}, and show that when $k\geq d_{\max}$, a greedy agent will move one step forward to the next closest intermediate state and eventually reach $\mc S_T$.

Comparing to Proposition \ref{prop:comp_sparse_reward}, we know that, the computational complexity (in terms of obtaining a successful Q-function) of an MDP $\mc M$ with the sparse reward setting can be reduced from $D(s_0,\mc S_T)$ to 
\begin{equation*}
    \max\{D(s_0,s_{i_1}),D(s_{i_1},s_{i_2}),\dots, D(s_{i_{N-1}},s_{i_N}), D(s_{i_N},\mc S_{T})\}, 
\end{equation*}
if $\mc M$ adopts the OWSP intermediate reward setting.

\subsection{One-Way Multi-Path Intermediate Rewards}
\label{subsec:OWMP_intermediate_states}

As shown in Table \ref{tab:app_example}, comparing to the OWSP setting, the OWMP setting applies to more practical tasks and can be considered as a special case of the OWSP setting. Therefore, Proposition \ref{prop:com_comp_inter_reward} for the OWSP setting cannot be directly generalized to the OWMP case, since moving towards one intermediate state does not necessarily lead to $\mc S_T$ in the OWMP case.

Still, the intuition for the computational complexity of obtaining a successful Q-function in the OWMP setting is similar to the OWSP setting. For a state $s$, under some conditions regarding rewards $B$, $B_I$, and intermediate states $\mc S_I$, a greedy agent finds the closest directly reachable intermediate state $s_{i_{j_1}}$ of $s$ (Definition \ref{def:DirectlyRechableIS}) after a sufficient number SVI. From any intermediate $s_{i_{j_{m}}}\in \mc S_I$, the greedy agent recursively finds the closest directly reachable intermediate state $s_{i_{j_{m+1}}}$ of $s_{i_{j_{m}}}$ and eventually reaches $\mc S_{T}$, since intermediate states cannot be revisited and the number of total intermediate states is finite. The next theorem characterizes the sufficient conditions for an agent starting from $s\in \mc S\backslash \mc S_T$ following a greedy policy to find the closest directly reachable intermediate of $s$.
\begin{theorem}[Finding the Closest $\mc S_I$]
\label{thm:find_closest_IS}
Let $\mc M = (\mc S,\mc A, P, r, \gamma)$ be a deterministic MDP with initial state $s_0$, intermediate states  $\mc S_{I} = \{s_{i_1},s_{i_2},\dots,s_{i_N}\}$, and terminal states $\mc S_{T}$. Suppose $\mc M$ satisfies Assumption \ref{assump:OneWayIntermediateStates}, \ref{assump:diff_settings_IS} (b), and \ref{assump:min_dist_IS}, if the reward function $r(\cdot)$ follows the intermediate reward setting (Table \ref{tab:reward_setting}) and the value function and Q-function are zero-initialized \eqref{eq:zero_init}, then $\forall s \in \mc S\backslash\mc S_{T}$ that cannot directly reach $\mc S_{T}$, after
\begin{equation}
    k\geq d = D(s,\mc S_{I}) \doteq \min_{j\in \mc I_d(s)}D(s,s_{i_j})
\end{equation}
synchronous value iteration updates \eqref{eq:SyncValueIter}, an agent following the greedy policy will find the closest directly reachable intermediate state $s_{i_{j^\star}}$ of $s$ ($D(s,s_{i_{j^\star}}) = D(s,\mc S_{I})$), given  $\frac{B}{B_I}\in \paren{0,\frac{1}{1-\gamma^h}}$ if $\gamma+\gamma^h\leq 1$ or $\frac{B}{B_I}\in\paren{\frac{1}{1-\gamma^h},\frac{1-\gamma}{\gamma^{1+h}}}$ if $\gamma+\gamma^h>1$, where $h$ is the minimum distance between any two intermediate states (Assumption \ref{assump:min_dist_IS}).
\end{theorem}
Here we highlight some key insights of the proof and leave the details in Appendix \ref{proof:find_closest_IS}. We first show that, when $k$ is not large ($k<D(s,\mc S_I)+h$), $V_k(s)$ will only be affected by the closest $\mc S_I$. Hence the value function $V_k(s)$ has the same property as the sparse reward setting provided in Proposition \ref{prop:comp_sparse_reward}, namely, $V_k(s)$ increases as $D(s,\mc S_I)$ decreases. Then when $k\geq D(s,\mc S_I)+h$, we show that the monotonicity between $V_k(s)$ and $D(s,\mc S_I)$ still holds using the relative magnitude $B/B_I$. More specifically, we use the bound of $B/B_I$ to show the reward from pursuing the closest $\mc S_I$ will be larger than the overall rewards of any other trajectories that {\em do not} pursuit the closest $\mc S_I$. With the monotonicity of $V_k(s)$ and $D(s,\mc S_I)$, we conclude a greedy agent finds the closest intermediate state.

A direct implication of Theorem \ref{thm:find_closest_IS} is that, when the magnitude of the terminal reward $B$ is not significantly larger than the intermediate reward $B_I$, $\forall s\in \mc S\backslash \mc S_{T}$, after $k\geq \underset{s\in \mc S\backslash \mc S_{T}}{\max}\underset{j\in \mc I_d(s)}{\min}D(s,s_{i_j})$ iteration of SVI updates, a greedy agent eventually finds to an intermediate state from which $\mc S_{T}$ is directly reachable. When $\mc S_{T}$ is directly reachable, the next theorem illustrates the condition that enables a greedy agent follow to pursue the shortest path to $\mc S_{T}$.

\begin{theorem}[Finding the Shortest Path to $\mc S_T$]
\label{thm:find_closest_TerminalState}
Let $\mc M = (\mc S,\mc A, P, r, \gamma)$ be a deterministic MDP with initial state $s_0$, intermediate states  $\mc S_{I} = \{s_{i_1},s_{i_2},\dots,s_{i_N}\}$, and terminal states $\mc S_{T}$. Suppose $\mc M$ satisfies Assumption \ref{assump:OneWayIntermediateStates}, \ref{assump:diff_settings_IS} (b), and \ref{assump:min_dist_IS}, the reward function $r(\cdot)$ follows the intermediate reward setting (Table \ref{tab:reward_setting}), and the value function, Q-function are zero-initialized \eqref{eq:zero_init}. $\forall s \in \mc S\backslash\mc S_{T}$, let 
\begin{equation}
     d = D(s,\mc S_{T}) \; \text{ and }\; d_I = D(s,\mc S_{I}) \doteq \min_{j\in \mc I_d(s)} D(s,s_{i_j}).
\end{equation}
If $\mc S_{T}$ is directly reachable from $s$, and $d$ and $d_I$ satisfy
\begin{equation}
    d<\begin{cases}
        d_I+\log_{\frac{1}{\gamma}}\brac{(1-\gamma^h)\frac{B}{B_I}},&\text{ if }\;\frac{B}{B_I}<\frac{1}{1-\gamma^h},\\
        d_I+\log_{\frac{1}{\gamma}}\paren{\frac{B}{B_I+\gamma^hB}},&\text{ if }\;\frac{B}{B_I}\geq \frac{1}{1-\gamma^h},
    \end{cases}
    \;\text{ and }\; d<d_I+h-1,
\end{equation}
where $h$ is the minimum distance between two intermediate states (Assumption \ref{assump:min_dist_IS}), then after $k\geq d$ synchronous value iteration updates \eqref{eq:SyncValueIter}, an agent following the greedy policy will pursue the shortest path to $\mc S_{T}$. 
\end{theorem}
The proof of Theorem \ref{thm:find_closest_TerminalState} shares the same idea as the proof of Theorem \ref{thm:find_closest_IS}. We first provide an explicit expression of the value function $V_k(s)$ for $k\leq d_I+h$, in Lemma \ref{lemma:Vk_Bound_directly_reachable_k<dI+h}. In addition, we demonstrate that $V_k(s)$ is monotonic decreasing as $D(s,\mc S_T)$ increases when $k\leq d_I+h$, since $V_k(s)$ is only affected by the closest $\mc S_I$ and $\mc S_T$ when $k\leq d_I+h$. Then we show the value function is monotonically decreasing as the distance $D(s,\mc S_T)$ increases for all $k\geq d$, when the conditions in Theorem \ref{thm:find_closest_TerminalState} are satisfied. Similar to Theorem \ref{thm:find_closest_IS}, we use the relative magnitude $B/B_I$ to show the the monotonicity of $V_k(s)$ still holds when $k>d_I+h$. More precisely, we use the bound of $B/B_I$ to ensure the rewards from pursuing $\mc S_T$ will be larger than all future rewards from any other trajectories that {\em do not} pursuit $\mc S_T$. Finally, we use the monotonicity of $V_k(s)$ to show that greedy policies find the shortest path to $\mc S_T$. The proof of Theorem \ref{thm:find_closest_TerminalState} is provided in Appendix \ref{proof:find_closest_TerminalState}.

Intuitively, Theorem \ref{thm:find_closest_TerminalState} indicates that, if a state $s$ is ``close enough'' to the terminal states $\mc S_T$, then the greedy policy will lead the agent from $s$ to $\mc S_T$ following the shortest path, given sufficient number of value iteration updates. Note that Theorem \ref{thm:find_closest_TerminalState} implicitly uses the fact that, in the OWMP setting, there could exist some states from which both $\mc S_I$ and $\mc S_T$ are directly reachable (see Figure \ref{fig:setting2}). 

All in all, when $B$ and $B_I$ satisfy the conditions in Theorem \ref{thm:find_closest_IS} and when 
\begin{equation*}
    k\geq \underset{s\in \mc S\backslash\mc S_T}{\max}\underset{j\in \mc I_d(s)}{\min}D(s,s_{i_j}),    
\end{equation*}
a greedy agent will recursively find the closest directly reachable intermediate state, until the terminal state $\mc S_T$ is directly reachable. When $\mc S_T$ is directly reachable and $D(s,\mc S_T)$, $D(s,\mc S_I)$ satisfy the conditions provided in Theorem \ref{thm:find_closest_TerminalState}, the agent following the greedy policy will pursue the shortest path to $\mc S_T$. We shall clarify that our theoretical results (Theorem \ref{thm:find_closest_IS} and \ref{thm:find_closest_TerminalState}) on the OWMP setting only provides {\em sufficient conditions} of obtaining successful Q-functions, but successful Q-functions can actually be obtained under broader conditions, as will be presented in Section \ref{sec:Experiment}.

\subsection{Non One-Way Intermediate Rewards}
\label{subsec:NonOneWayIntermedaiteReward}

The theoretical results in the non one-way intermediate (NOW) reward setting is out of the scope of this work, since we do not have additional assumptions on the intermediate states. Still, we provide an example where SVI {\em never} finds a successful Q-function in Example \ref{example:SyncVIFail} to demonstrate the effect of non-ideal reward design. 

\begin{example}[Synchronous Value Iteration Fails] 
\label{example:SyncVIFail}
Suppose we are given an MDP environment as shown in Figure \ref{fig:NOWIRGrid}, with $B=10$, $B_I=100$, and $\gamma = 0.9$. The value function converges when $k\geq 2$. However, $\forall k \geq 1$, a greedy policy will stay in $s_i$ instead of moving to $\mc S_T$. 

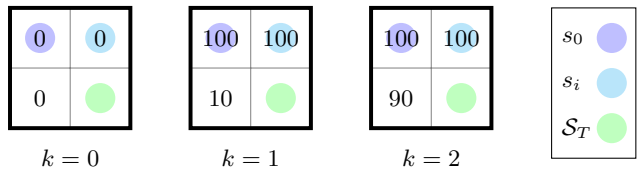
\begin{figure}[ht]
    \centering
    \begin{tikzpicture}[scale=0.8]
\centering
% Grid [0]
\draw[step=1cm, gray, very thin] (0, -2) grid (2, 0);
\fill[blue!25!white] (0.5, -0.5) circle (0.25cm);
\fill[green!25!white] (1.5, -1.5) circle (0.25cm);
% draw IS
\fill[cyan!25!white] (1.5, -0.5) circle (0.25cm);
% Fill value
\node [font=\footnotesize] (a11) at (0.5,-0.5) {$0$};
\node [font=\footnotesize] (a12) at (1.5,-0.5) {$0$};
\node [font=\footnotesize] (a21) at (0.5,-1.5) {$0$};
\node [font=\footnotesize] (a22) at (1.5,-1.5) {};
% draw k = 1
\node [font=\footnotesize] (a) at (1.0,-2.5) {$k=0$};
% draw boundary
\draw[black, line width=1.5] (0.0, 0.0) -- (2.0, 0.0) -- (2.0, -2.0) -- (0.0, -2.0) -- cycle;

% Grid [1] (x+3)
\draw[step=1cm, gray, very thin] (3, -2) grid (5, 0);
\fill[blue!25!white] (3.5, -0.5) circle (0.25cm);
\fill[green!25!white] (4.5, -1.5) circle (0.25cm);
% draw IS
\fill[cyan!25!white] (4.5, -0.5) circle (0.25cm);
% Fill value
\node [font=\footnotesize] (b11) at (3.5,-0.5) {$100$};
\node [font=\footnotesize] (b12) at (4.5,-0.5) {$100$};
\node [font=\footnotesize] (b21) at (3.5,-1.5) {$10$};
\node [font=\footnotesize] (b22) at (4.5,-1.5) {};
% draw k = 1
\node [font=\footnotesize] (b) at (4.0,-2.5) {$k=1$};
% draw boundary
\draw[black, line width=1.5] (3.0, 0.0) -- (5.0, 0.0) -- (5.0, -2.0) -- (3.0, -2.0) -- cycle;

% Grid [2] (x+6)
\draw[step=1cm, gray, very thin] (6, -2) grid (8, 0);
\fill[blue!25!white] (6.5, -0.5) circle (0.25cm);
\fill[green!25!white] (7.5, -1.5) circle (0.25cm);
% draw IS
\fill[cyan!25!white] (7.5, -0.5) circle (0.25cm);
% Fill value
\node [font=\footnotesize] (c11) at (6.5,-0.5) {$100$};
\node [font=\footnotesize] (c12) at (7.5,-0.5) {$100$};
\node [font=\footnotesize] (c21) at (6.5,-1.5) {$90$};
\node [font=\footnotesize] (c22) at (7.5,-1.5) {};
% draw k = 2
\node [font=\footnotesize] (c) at (7.0,-2.5) {$k=2$};
% draw boundary
\draw[black, line width=1.5] (6.0, 0.0) -- (8.0, 0.0) -- (8.0, -2.0) -- (6.0, -2.0) -- cycle;

%% legend
% Bounding Box
\draw[black] (9,0) -- (9,-2.5) -- (10.4,-2.5) -- (10.4,0) -- cycle;

% Drawing
\fill[blue!25!white] (10, -0.5) circle (0.25cm);
\fill[cyan!25!white] (10, -1.25) circle (0.25cm);
\fill[green!25!white] (10, -2.0) circle (0.25cm);
% Caption
\node [right,font=\footnotesize] (A) at (9.0,-0.5) {$s_0$};
\node [right,font=\footnotesize] (B) at (9.0,-1.25) {$s_{i}$};
\node [right,font=\footnotesize] (C) at (9.0,-2.0) {$\mc S_{T}$};
\end{tikzpicture}
    \caption{An example of the evolution of a zero-initialized value function during SVI in the non one-way intermediate reward setting, with $B=10$, $B_I=100$, and $\gamma = 0.9$. The agent may move left, right, up, or down. The cyan circle is the intermediate state that can be {\em repeatedly} visited from $s_0$ and from itself (by taking actions ``right'' and ``up'' and hitting a wall). The remaining settings are the same as Figure \ref{fig:SparseRewardGrid}.}
    \label{fig:NOWIRGrid}
\end{figure}
\end{example}

As shown in Example \ref{example:SyncVIFail}, if we set the intermediate rewards $B_I$ on some {\em non one-way} intermediate states, and when the intermediate rewards are significantly larger than the terminal rewards, a greedy policy will lead the agent to visit the intermediate states repeatedly rather than pursuing the terminal states. A direct implication from Example \ref{example:SyncVIFail} is that, to prevent an agent from staying at non one-way intermediate states, the intermediate rewards on these non one-way intermediate states should be relatively small.

\section{Connections to Finding the Shortest Path}
\label{sec:FindShortestPath}

We discuss the connections between the different reward settings (sparse reward, OWSP, OWMP, NOW) and their connections to finding the shortest path in this section.

\paragraph{The Sparse Reward Setting}
The sparse reward setting {\em requires no prior knowledge nor assumptions of the subgoals of the tasks}, because it contains no intermediate states. Hence, it can further be applied to general goal-reaching tasks. In addition, as illustrated in Proposition \ref{prop:comp_sparse_reward}, the successful policy obtained by SVI pursues the shortest path but at the cost of high computational complexity ($D(s_0,\mc S_T)$ total number of SVI). The computational complexity in the sparse reward setting motivates the use of intermediate rewards: practical tasks with well-designed intermediate rewards should be more computationally efficient than the same task with only sparse terminal rewards. 

\paragraph{The OWSP Intermediate Reward Setting} 
As mentioned in Section \ref{subsec:InterpretInterState}, the OWSP intermediate reward setting only applies to limited number of tasks seen in practice. However, when the subgoals of the goal-reaching tasks indeed {\em has the one-way single-path structure}, Proposition \ref{prop:com_comp_inter_reward} shows that the computational complexity of obtaining a successful Q-function can be significantly reduced, and a greedy policy will find the shortest path from $s_0$ to $\mc S_T$.

\paragraph{The OWMP Intermediate Reward Setting} Comparing to the OWSP setting, the OWMP intermediate setting has much broader practical applications. Though the OWMP intermediate setting is more prevalent, the downside is that generally a greedy policy with a successful Q-function will pursue the closest intermediate states instead of the shortest path to $\mc S_T$. Figure \ref{fig:CompComp_VS_ShortestPath} illustrates the trade-off between computational complexity and achieving the shortest path is provided: A greedy policy reaches $\mc S_T$ when $k$ (the number of SVI) equals 1 or 2, by recursively pursuing the closest intermediate state (path 1). When $k>2$ and the terminal rewards $B$ and intermediate rewards $B_I$ satisfy the conditions stated in Theorem \ref{thm:find_closest_TerminalState}, the greedy policy will find $\mc S_T$ via the shortest path (path 2).
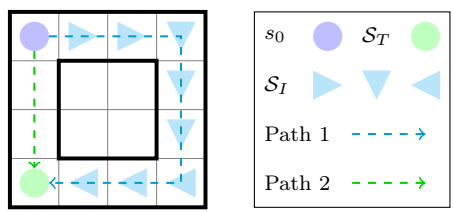
\begin{figure}[ht]
    \centering
    \begin{tikzpicture}[scale=0.65]
\centering
% Grid [0]
\draw[step=1cm, gray, very thin] (0, -4) grid (4, 0);
\fill[blue!25!white] (0.5, -0.5) circle (0.3cm);
\fill[green!25!white] (0.5, -3.5) circle (0.3cm);
% draw IS
\path[fill=cyan!25!white] (1.2,-0.2) -- (1.8,-0.5) -- (1.2,-0.8) -- cycle;
\path[fill=cyan!25!white] (2.2,-0.2) -- (2.8,-0.5) -- (2.2,-0.8) -- cycle;
\path[fill=cyan!25!white] (3.2,-0.2) -- (3.8,-0.2) -- (3.5,-0.8) -- cycle;
\path[fill=cyan!25!white] (3.2,-1.2) -- (3.8,-1.2) -- (3.5,-1.8) -- cycle;
\path[fill=cyan!25!white] (3.2,-2.2) -- (3.8,-2.2) -- (3.5,-2.8) -- cycle;
\path[fill=cyan!25!white] (3.2,-3.5) -- (3.8,-3.2) -- (3.8,-3.8) -- cycle;
\path[fill=cyan!25!white] (2.2,-3.5) -- (2.8,-3.2) -- (2.8,-3.8) -- cycle;
\path[fill=cyan!25!white] (1.2,-3.5) -- (1.8,-3.2) -- (1.8,-3.8) -- cycle;

% draw boundary
\draw[black, line width=1.5] (0.0, 0.0) -- (4.0, 0.0) -- (4.0, -4.0) -- (0.0, -4.0) -- cycle;
\draw[black, line width=1.5] (1.0, -1.0) -- (3.0, -1.0) -- (3.0, -3.0) -- (1.0, -3.0) -- cycle;

% draw paths
\draw[cyan!80!black, line width=0.75, dashed, ->] (0.8, -0.5) -- (3.5, -0.5) -- (3.5, -3.5) -- (0.8, -3.5);
\draw[green!80!black, line width=0.75, dashed, ->] (0.5, -0.8) -- (0.5, -3.2);

%% legend
% Bounding Box
\draw[black] (5,0) -- (5,-4) -- (9,-4) -- (9,0) -- cycle;

% Drawing
\fill[blue!25!white] (6.5, -0.5) circle (0.3cm);
\fill[green!25!white] (8.5, -0.5) circle (0.3cm);
\path[fill=cyan!25!white] (6.2,-1.2) -- (6.2,-1.8) -- (6.8,-1.5) -- cycle;
\path[fill=cyan!25!white] (7.2,-1.2) -- (7.8,-1.2) -- (7.5,-1.8) -- cycle;
\path[fill=cyan!25!white] (8.8,-1.2) -- (8.8,-1.8) -- (8.2,-1.5) -- cycle;
\draw[cyan!80!black, line width=0.75, dashed, ->] (7.0, -2.5) -- (8.5, -2.5);
\draw[green!80!black, line width=0.75, dashed, ->] (7.0, -3.5) -- (8.5, -3.5);
% Caption
\node [right,font=\scriptsize] (A) at (5,-0.5) {$s_0$};
\node [right,font=\scriptsize] (B) at (7,-0.5) {$\mc S_{T}$};
\node [right,font=\scriptsize] (C) at (5,-1.5) {$\mc S_I$};
\node [right,font=\scriptsize] (D) at (5,-2.5) {Path 1};
\node [right,font=\scriptsize] (E) at (5,-3.5) {Path 2};
\end{tikzpicture}
    \caption{The trade-off between the minimum computational complexity and the pursuit of the shortest path in the OWMP intermediate reward setting. The cyan isosceles triangles are the one-way intermediate states and the orientations of each    
    isosceles triangle represents the direction of each intermediate state -- the agent at a given intermediate state can only visit the next state pointed by the orientation of the apex. The remaining settings are the same as Figure \ref{fig:SparseRewardGrid}.}
    \label{fig:CompComp_VS_ShortestPath}
\end{figure}

\section{Experiment}
\label{sec:Experiment}
We experimentally verify in several OpenAI Gym MiniGrid environments that the agent is able to learn a successful trajectory more quickly in OWSP and OWMP intermediate reward settings than the sparse reward setting. We verify our findings on $\eps$-greedy asynchronous Q-learning algorithm, and three popular deep RL algorithms: DQN \shortcite{mnih2015human}, A2C \shortcite{mnih2016asynchronous}, and PPO \shortcite{schulman2017proximal}. For all experiments, the agent observes the whole environment. For asynchronous Q-learning, each state is represented as a string encoding of the grid. For deep RL algorithms, each state is an image of the grid. The detailed parameters and additional related experiments are provide in Appendix \ref{subsec:Hyperparameters}. See \hyperlink{https://github.com/kebaek/minigrid}{https://github.com/kebaek/minigrid} for the code to run all presented experiments.

The main purpose of our experiments is to justify the following indications from our theoretical results: 1) adding intermediate rewards on one-way intermediate states reduces computational complexity; 2) there is a trade-off between computational complexity and the pursuit of the shortest path in the OWMP setting. Although we have discussed many examples of RL applications in Table \ref{tab:app_example}, we choose the OpenAI Gym MiniGrid environments \shortcite{Brockman2016OpenAI,gym_minigrid} rather than the Games (Pacman, Montezuma, Dota2, StarCraft II, and Honor of Kings) discussed in Table \ref{tab:app_example} due to limited computational resources and the intrinsic randomness in these Game. 

\subsection{Environmental Setting}
\label{subsec:envsetting}

\paragraph{Single-Path Maze} The 7x7 grid maze (Figures \ref{fig:sparsemaze} and \ref{fig:intermaze}) consists of a single path that the agent must navigate through to reach the terminal state. This environment will be used to study the computational benefit of having well-designed intermediate rewards for the OWSP setting. See Figure \ref{fig:minigrid_env} for possible actions and reward design. Note that this maze environment is one-way since each intermediate reward may only be obtained once, i.e., the blue circle disappears from the environment once the agent reaches the corresponding square.

\paragraph{3-Door/4-Door} These 9x9 grid mazes (Figures \ref{fig:3door} and \ref{fig:4door}) consist of 3 different paths to to the terminal state each sealed by a series of locked doors. Each door has a corresponding key of the same color. The goal of the agent is to reach the terminal state by picking up the corresponding keys to unlock all the doors along at least one of the three paths. Note that this environment is one-way since a door cannot lock once unlocked and a key cannot be dropped once picked up. These environments will be used to study the computational benefit of having well-designed intermediate rewards and the trade-off between computational complexity and shortest path for the OWMP setting (Section \ref{subsec:OWMP_intermediate_states}).

\begin{figure}[ht]
    \centering
    \begin{subfigure}[t]{0.21\textwidth}
        \centering
        \includegraphics[width=\textwidth,trim={0cm 0cm 0cm 0cm},clip]{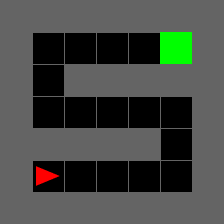}
        \caption{\footnotesize Sparse reward}
        \label{fig:sparsemaze}
    \end{subfigure}
    ~
    \begin{subfigure}[t]{0.21\textwidth}
        \centering
        \includegraphics[width=\textwidth,trim={0cm 0cm 0cm 0cm},clip]{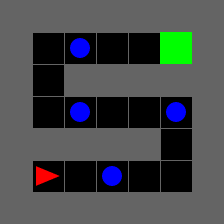}
        \caption{\footnotesize Intermediate reward}
        \label{fig:intermaze}
    \end{subfigure}
    ~
    \begin{subfigure}[t]{0.21\textwidth}
        \centering
        \includegraphics[width=\textwidth,trim={0cm 0cm 0cm 0cm},clip]{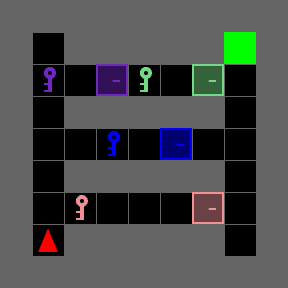}
        \caption{\footnotesize 3-Door env.}
        \label{fig:3door}
    \end{subfigure}
    ~
    \begin{subfigure}[t]{0.21\textwidth}
        \centering
        \includegraphics[width=\textwidth,trim={0cm 0cm 0cm 0cm},clip]{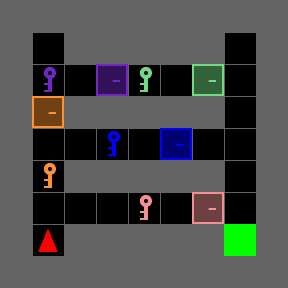}
        \caption{\footnotesize 4-Door env.}
        \label{fig:4door}
    \end{subfigure}
    \caption{\footnotesize The red triangle is the agent and the green square is the terminal state. (a)-(b): The Single-Path Maze environment. Agent takes actions from \{go forward, turn $90^\circ$, turn $-90^\circ$\}. For the sparse reward setting (a), the agent receives a terminal reward of +10. For the intermediate reward setting (b), the agent also receives +1 for arriving at any square with a blue circle. (c)-(d): The Door-Key environments. Agent takes actions from \{go forward, turn $90^\circ$, turn $-90^\circ$, pick up key, open door\}. For the sparse reward setting of (c) and (d), the terminal reward is +10. For the intermediate reward setting of (c) and (d), the agent also receives +2 for picking up a key or opening a door. All rewards in environments (a)-(d) can only be obtained {\em once} per episode.}
    \label{fig:minigrid_env}
\end{figure}

\subsection{Results}
On the Single-Path Maze environment and the 3-Door environment, we compare the number of training episodes required for the agent to find a successful policy in the sparse reward setting and the intermediate reward setting (See Figure \ref{fig:sparsemaze} and \ref{fig:3door} for more details about reward design).
For 0.8 $\eps$-greedy Q-learning, we train 100 independent models and evaluate each model once, for a total of 100 trials. For the deep RL algorithms, we train 10 independent models, and evaluate each model 10 times with different seeds, for a total of 100 trials. Win rate is computed as the number of trials that reach the terminal state out of 100.

As expected, we observe that it takes $\eps$-greedy Q-learning 36 episodes to reach a win rate of $100\%$ whereas in the sparse reward setting, $\eps$-greedy Q-learning is only able to reach a win rate of $10\%$ for the same number of episodes (See Table \ref{tab:minigrid1}). Similarly for the 3-Door setting, if the agent is also rewarded for picking up keys and opening doors, significantly less training episodes are required to obtain a win rate of $100\%$ (See Table \ref{tab:minigrid1}). We observe a similar phenomena on popular deep RL algorithms: DQN, A2C, and PPO (See Figure \ref{fig:deep_rl}). Our experimental findings on the computational benefits of using intermediate rewards corroborate our theoretical claim in Section \ref{sec:MainResult}.

\begin{table}[ht]
    \centering
    \small
    \begin{tabular}{rcc|rcc}
        \toprule 
        \multicolumn{3}{c|}{Maze (Figure \ref{fig:sparsemaze} and \ref{fig:intermaze})} & \multicolumn{3}{c}{3-Door (Figure \ref{fig:3door})}\\
        \toprule
        \# Episodes & Sparse & Intermediate & \# Episodes & Sparse & Intermediate\\
        \midrule 18	& 6/100 &59/100 & 40 &	3/100 & 90/100 \\
        \midrule 24	& 7/100 &82/100 & 80 & 10/100 & 100/100 \\
        \midrule 30	& 6/100 &95/100 & 120 &	43/100 & 100/100\\
        \midrule 36	& 10/100 &  100/100 & 160 & 74/100& 100/100\\
        \bottomrule
    \end{tabular}
    \caption{Asynchronous Q-learning: Computational Complexity. We report the number of wins out of 100 trials (100 training sessions evaluated once each) after different training episodes for both the sparse reward and intermediate reward settings.}
    \label{tab:minigrid1}
\end{table}

\begin{table}[ht]
    \centering
    \small
    \begin{tabular}{rcccccc}
        \toprule
        \multicolumn{1}{c|}{} & \multicolumn{3}{c|}{Setting 1: +10} & \multicolumn{3}{c}{Setting 2: +1000}\\
        \midrule
        \# Episodes & Wins & Rewards & Steps & Wins & Rewards & Steps\\
        \toprule 50 & 64/100& 11.28 $\pm$ 2.08 & 95.52 $\pm$ 82.58  & 71/100 & 11.12 $\pm$ 2.30 & 84.39 $\pm$ 79.48\\
        \midrule 150 & 99/100 & 11.2 $\pm$ 2.03 & 24.71 $\pm$ 18.10 & 100/100 & 11.06 $\pm$ 1.99 & 22.53 $\pm$ 3.56 \\
        \midrule 350 & 100/100 & 11.1 $\pm$ 1.93 &  22.65 $\pm$ 3.40 & 100/100 & 10.14 $\pm$ 2.60 &  20.88 $\pm$ 4.70 \\
        \midrule 750 & 100/100 & 11.18 $\pm$ 1.80 & 22.85 $\pm$ 3.24 & 100/100 & 9.24 $\pm$ 2.85 & 19.39 $\pm$ 5.54 \\
        \midrule 1550 & 100/100 & 11.64 $\pm$ 1.42 & 23.36 $\pm$ 2.62  & 100/100 & 5.52 $\pm$ 1.10&  12.95 $\pm$ 1.64\\
        \midrule 3150 & 100/100 & 12.0 $\pm$ 0.0 & 24.0 $\pm$ 0.0 & 100/100 & 4.46 $\pm$ 0.84 &  12.23 $\pm$ 0.42 \\
        \bottomrule
    \end{tabular}
    \caption{4-Door (Figure \ref{fig:4door}) Trade-Off for Asynchronous Q-Learning. We report the number of wins out of 100 trials (100 training sessions evaluated once each), averaged total intermediate rewards, and averaged number of steps taken after different training episodes. The agent receives a terminal reward of +10 in Setting 1 +1000 in Setting 2. The agent also receives +2 for picking up a key and +2 for opening a door in both settings. The episode maxes out at 324 steps.}
    \label{tab:minigrid3}
\end{table}

\begin{figure}[ht]
    \centering
    \begin{subfigure}[b]{\textwidth}
         \centering
            \includegraphics[scale=0.12]{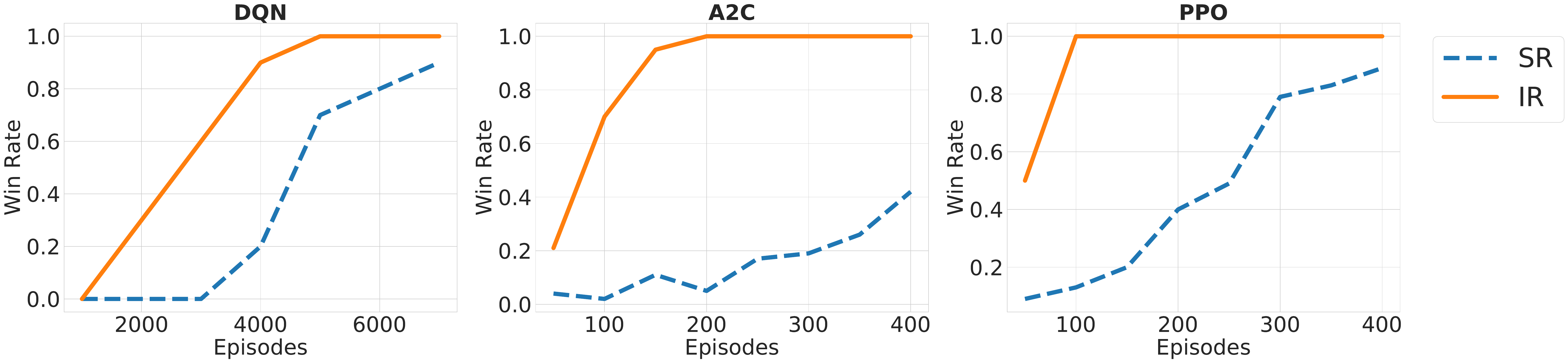}
         \caption{Maze}
         \label{fig:drl_gridmaze}
     \end{subfigure}
     \begin{subfigure}[b]{\textwidth}
         \centering
            \includegraphics[scale=0.12]{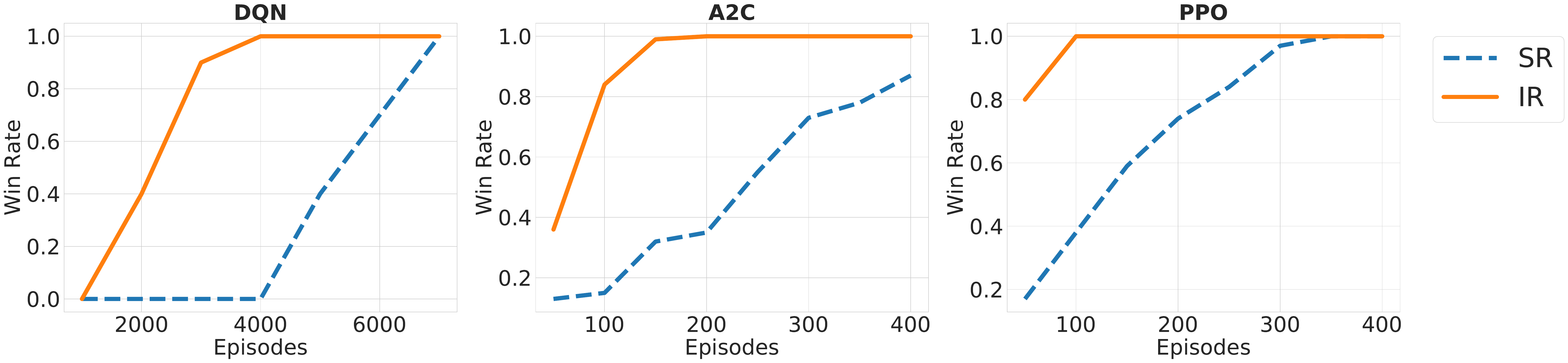}
         \caption{3-Door}
         \label{fig:drl_3door}
     \end{subfigure}
     
    \caption{Deep RL: Computational Complexity. We compare the average number of steps an agent takes to reach the terminal state in the Single-Path Maze and 3-Door environments between sparse versus intermediate reward settings. If the agent does not reach the terminal state, the episode maxes out at 324 steps for 3-Door and 196 steps for Single-Path Maze. The results are averaged over 100 trials (10 training sessions evaluated 10 times each)}
    \label{fig:deep_rl}
\end{figure}

On the 4-Door environment, we test two intermediate reward settings where the agent is either rewarded $10$ or $1000$ points for reaching the terminal state. Additionally, the agent is rewarded $2$ points for either picking up a key or opening a door. The shortest successful path to the terminal state (12 steps) has the least number of doors to unlock, thus less intermediate rewards, whereas the longest of the three paths to the terminal state contains at least 3 doors to unlock, thus more possible intermediate rewards to collect. We compare the average steps required during evaluation for the agent to reach the terminal state between these two settings. From Section \ref{subsec:OWMP_intermediate_states}, we expect that in this OWMP setting, the agent identifies the shortest successful path with Q-learning given a well-designed discount factor and a good ratio between intermediate/terminal rewards. For a discount factor of 0.9, we observe that if the terminal reward is $1000$, the path taken by an agent trained by $\eps$-greedy Q-learning converges to the shortest path of 12 steps. On the other hand, if the terminal reward is $10$, the path taken by the agent converges to that of 24 steps and on average collects more intermediate rewards (See Table \ref{tab:minigrid3}). Similarly, with a discount factor of 0.8 for DQN and 0.9 for A2C and PPO, we observe that the agent chooses the shortest successful trajectory when rewarded $1000$ points for reaching the terminal state during training, and a longer successful trajectory with more intermediate rewards when rewarded $10$ points. Our experimental findings on the trade-off between computational efficiency and the shortest path corroborate our theoretical claim in Section \ref{subsec:OWMP_intermediate_states}.

\begin{figure}[ht]
     \centering
        \includegraphics[scale=0.12]{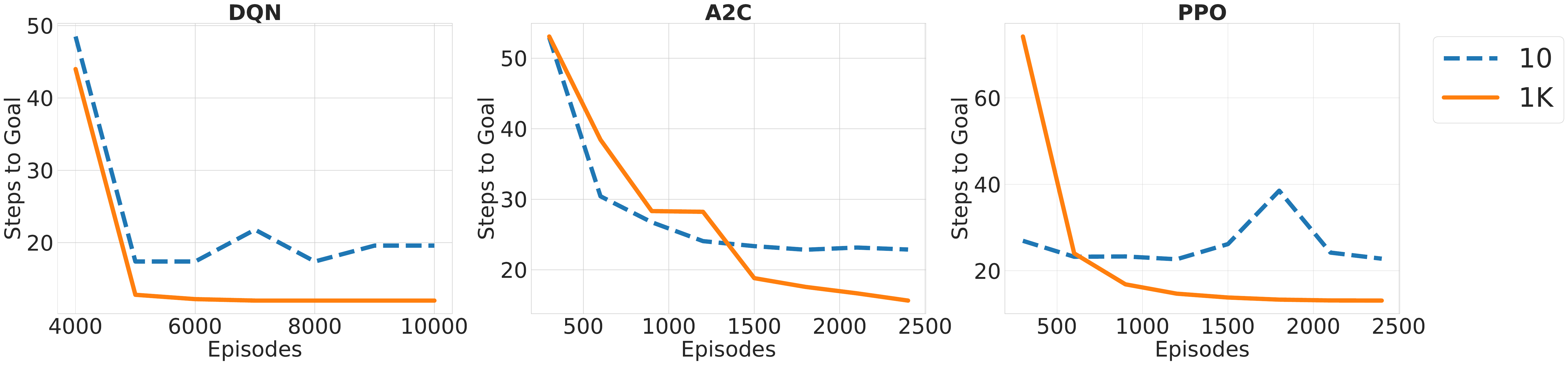}
     \caption{Deep RL: Trade-Off. We compare the average number of steps an agent takes to reach the terminal state in the 4-Door environment if the terminal reward is $10$ versus $1000$. The agent also receives an intermediate reward of $2$ when it picks up a key or opens a door. If the agent does not reach the terminal state, the episode maxes out at 324 steps. The results are averaged over 100 trials (10 training sessions evaluated 10 times each)}
     \label{fig:drl_4door}
\end{figure}

\section{Related Works and Discussion}
\label{sec:Discussion}
\subsection{Related Works}

\paragraph{Hierarchical RL} Hierarchical reinforcement learning and planning are two fundamental problems that have been studied for decades \shortcite{dayan1992feudal,kaelbling1993hierarchical,parr1998hierarchical,parr1998reinforcement,mcgovern1998hierarchical,sutton1998intra,precup1998multi,sutton1999between,dietterich2000hierarchical,mcgovern2001automatic} (See \shortciteR{barto2003recent} for an overview of other earlier works on hierarchical RL and Chapter 11.2 in the book \shortciteR{russell2016artificial} for hierarchical planning). After the success of deep learning \shortcite{lecun2015deep,goodfellow2016deep}, recent works have revisited hierarchical RL under the deep RL framework \shortcite{kulkarni2016hierarchical,vezhnevets2017feudal,andreas2017modular,le2018hierarchical,xu2020hierarchial} (see Chapter 11 in the review paper \shortciteR{li2018deep} for other hierarchical deep RL literature). Prior theoretical attempts on hierarchical RL formulated the problem as MDP decomposition problems \shortcite{dean1995decomposition,singh1998dynamically,meuleau1998solving,wen2020efficiency} or solving subtasks with ``bottleneck'' states \shortcite{sutton1999between,mcgovern2001automatic,stolle2002learning,simsek2008skill,solway2014optimal}. Our result is closely related to \shortciteA{wen2020efficiency} in the following two aspects: 1) Our one-way assumption (Assumption \ref{assump:OneWayIntermediateStates}) on the intermediate states is similar to the exit state of subMDPs (Definition 1 in \shortciteR{wen2020efficiency}), in the sense that our one-way intermediate states can be viewed as the exit states that separate the MDP into subMDPs; 2) Our quantitative results of different MDP settings suggest that partitioning the large MDP via intermediate states and intermediate rewards generally reduce the computational complexity, which corroborates the computational efficiency of subMDPs in \shortciteA{wen2020efficiency}. Though we come to similar theoretical conclusions that intermediate rewards lead to reduced computational complexity, in this work, we simplify down the assumptions and build a theoretical framework that is well connected to practice.

\paragraph{Reward Design} With the recent success of deep learning, RL has experienced a renaissance \shortcite{Krakovsky2016Reinforcement} and has demonstrated super-human performance in various applications \shortcite{mnih2013playing,mnih2015human,silver2016mastering,silver2017mastering,vinyals2017starcraft,vinyals2019grandmaster,berner2019dota,ye2020towards,fuchs2021super}. The reward design varies from task to task. For example, in Go \shortcite{silver2016mastering,silver2017mastering}, the agent only receives {\em terminal rewards} at the end of the game (+1 for winning and -1 for losing); for Starcraft II \shortcite{vinyals2019grandmaster}, the reward function is usually a mixture win-loss terminal rewards (+1 on a win, 0 on a draw, and -1 on a loss) and intermediate rewards based on human data; for multiplayer online battle arena (MOBA) games \shortcite{OpenAI_dota,berner2019dota,ye2020towards}, the reward functions are generally heavily handcrafted (see Table 6 of \shortciteR{berner2019dota} and Table 4 of \shortciteR{ye2020towards}) based on prior knowledge. Besides task-dependent reward design, other works also have studied the reward design for general RL or robotic tasks \shortcite{singh2009rewards,singh2010intrinsically,sorg2010reward,vezhnevets2017feudal,VanSeijen2017Hybrid,Raileanu2020RIDE:,ratner2018simplifying,ratner2018simplifying,he2021assisted} (See \shortciteR{guo2017deep,doroudi2019s} and references therein; also see other literature in Chapter 5 of the review paper \shortciteR{li2018deep}).

\paragraph{Goal-Conditioned RL}
Goal-conditioned RL, the problem of learning a policy that reaches certain goal states, has been empirically studied in many prior works  \shortcite{kaelbling1993learning,sutton2011horde,andrychowicz2017hindsight,fu2018variational,pong2018temporal,ghosh2019learning,eysenbach2020rewriting,eysenbach2020c,kadian2020sim2real,fujita2020distributed,chebotar2021actionable,khazatsky2021can}. The goal-conditioned RL is closely related to the sparse reward setting in our framework, where the agent {\em only} receives terminal rewards at the terminal (goal) states. Moreover, the empirical observations where goal-reaching tasks can be improved via pursuing subgoals \shortcite{andreas2017modular,nasiriany2019planning} also corroborate with the computational benefits of rewarding intermediate states suggested by our theoretical results in Section \ref{sec:MainResult}.

\paragraph{Connection to Reward Shaping \shortcite{ng1999policy}}
For a given MDP $\mc M$ with reward function $r$, \shortcite{ng1999policy} proposed a potential-based shaping function $F$, such that the same MDP $\mc M$ with shaped reward $F+r$ has the same optimal policy as $\mc M$ with the original reward function $r$. Part of our result is related to reward shaping \shortcite{ng1999policy} in the sense that, when certain conditions are satisfied (Assumption \ref{assump:diff_settings_IS} (a)), greedy policies under the sparse reward setting and the intermediate reward setting are the same, as they both follow the shortest path from $s_0$ to $\mc S_T$. Comparing to reward shaping, the advantage of our work is also on the practical side, since the assumptions of the potential function $F$ is generally hard to satisfy. $F$ is usually approximated via neural networks in practice, which requires extra engineering efforts. On the contrary, as we have discussed in Section \ref{sec:InterStateInterReward} and Section \ref{sec:Experiment}, our assumptions on the one-way intermediate states (Assumption \ref{assump:diff_settings_IS}) and relative magnitude between intermediate rewards $B_I$ and terminal rewards $B$
can easily be implemented in practice.

\subsection{Discussion}
\paragraph{Practical Implications} Our work theoretically verifies the common practice of adding intermediate rewards to speed up training in reinforcement learning. Formally, in order to find a goal-reaching successful policy, adding intermediate rewards based on prior knowledge of the practical tasks is generally more computationally efficient than using sparse terminal rewards alone. However, unless the intermediate rewards are carefully designed (e.g., like the OWSP setting described in Section \ref{subsec:OWSP_intermediate_states}), greedy policies usually do not follow the shortest path to the terminal states. To prevent the agent from getting stuck at non one-way intermediate states (e.g., like the case discussed in Example \ref{example:SyncVIFail}), we can assign relatively smaller rewards (compared to the terminal rewards) to non one-way intermediate states. Our findings corroborate the reward design of Dota2 in Table 6 of \shortcite{berner2019dota}: one can understand the ``Win'' (with reward +5) as terminal states; ``XP Gained'' (with reward +0.002), ``Gold Gained'' (with reward +0.006), and ``Gold Spent'' (with reward +0.0006) as {\em non one-way} intermediate states;``T1 Tower'' (with reward +2.25), ``T2 Tower'' (with reward +3), ``T3 Tower'' (with reward +4.5), ``T4 Tower'' (with reward +2.25), ``Shrine'' (with reward +2.25), and ``Barracks'' (with reward +6) as {\em one-way} intermediate states. 

\paragraph{Negative Rewards}
One limitation of this work is that we only consider positive rewards, but in many applications described before, the agent may receive negative rewards upon the arrival of some unfavorable states. We will provide several examples here to provide some intuitions on the effect of negative rewards in Figure \ref{fig:NEGReward} and leave formal studies to future. In Figure \ref{fig:NEGReward} (a), for an intermediate state with negative reward which is not in the path from $s_0$ to $\mc S_T$, then given enough computational complexity, a greedy agent still finds $\mc S_T$. However, if there is an intermediate state with negative reward that is in the path from $s_0$ to $\mc S_T$, the behavior of a greed agent will depend on the magnitude of the negative rewards. As shown in Figure \ref{fig:NEGReward} (b) and (c), when the magnitude of the negative intermediate rewards $r(s,a,s_-)$ is relatively small (as shown in $(b)$, $r(s,a,s_-) = -1$), then a greedy agent will still find $\mc S_T$ because the overall reward from $\mc S_T$ and $s_-$ is larger than $0$; however, when the magnitude of the negative reward is large (as shown in $(b)$, $r(s,a,s_-) = -10$), then a greedy agent will not find $\mc S_T$ because the overall reward of $s_0$ and $\mc S_T$ is smaller than $0$.

\begin{figure}[ht]
    \centering
    \begin{tikzpicture}[scale=0.8]
\centering
% Grid [0]
\draw[step=1cm, gray, very thin] (0, -3) grid (3, 0);
\fill[blue!25!white] (0.5, -0.5) circle (0.25cm);
\fill[green!25!white] (2.5, -2.5) circle (0.25cm);
% draw IS
\fill[red!25!white] (1.5, -1.5) circle (0.25cm);
% Fill value
\node [font=\footnotesize] (a11) at (0.5,-0.5) {$7.29$};
\node [font=\footnotesize] (a12) at (1.5,-0.5) {$0$};
\node [font=\footnotesize] (a13) at (2.5,-0.5) {$0$};
\node [font=\footnotesize] (a21) at (0.5,-1.5) {$8.1$};
\node [font=\footnotesize] (a22) at (1.5,-1.5) {$7.29$};
\node [font=\footnotesize] (a23) at (2.5,-1.5) {$0$};
\node [font=\footnotesize] (a31) at (0.5,-2.5) {$9$};
\node [font=\footnotesize] (a32) at (1.5,-2.5) {$10$};
\node [font=\footnotesize] (a33) at (2.5,-2.5) {};
% draw k = 1
\node [font=\footnotesize] (a) at (1.5,-3.5) {(a) $r(s,a,s_-) = -1$};
% draw boundary
\draw[black, line width=1.5] (3.0, -2.0) -- (1.0, -2.0);
\draw[black, line width=1.5] (0.0, 0.0) -- (3.0, 0.0) -- (3.0, -3.0) -- (0.0, -3.0) -- cycle;

% Grid [1] (x+4)
\begin{scope}[shift={(4.0,0)}]
\draw[step=1cm, gray, very thin] (0, -3) grid (3, 0);
\fill[blue!25!white] (0.5, -0.5) circle (0.25cm);
\fill[green!25!white] (2.5, -2.5) circle (0.25cm);
% draw IS
\fill[red!25!white] (0.5, -1.5) circle (0.25cm);
% Fill value
\node [font=\footnotesize] (b11) at (0.5,-0.5) {$6.39$};
\node [font=\footnotesize] (b12) at (1.5,-0.5) {$0$};
\node [font=\footnotesize] (b13) at (2.5,-0.5) {$0$};
\node [font=\footnotesize] (b21) at (0.5,-1.5) {$8.1$};
\node [font=\footnotesize] (b22) at (1.5,-1.5) {$6.39$};
\node [font=\footnotesize] (b23) at (2.5,-1.5) {$0$};
\node [font=\footnotesize] (b31) at (0.5,-2.5) {$9$};
\node [font=\footnotesize] (b32) at (1.5,-2.5) {$10$};
\node [font=\footnotesize] (b33) at (2.5,-2.5) {};
% draw k = 1
\node [font=\footnotesize] (b) at (1.5,-3.5) {(b) $r(s,a,s_-) = -1$};
% draw boundary
\draw[black, line width=1.5] (3.0, -2.0) -- (1.0, -2.0);
\draw[black, line width=1.5] (0.0, 0.0) -- (3.0, 0.0) -- (3.0, -3.0) -- (0.0, -3.0) -- cycle;
\end{scope}

% Grid [2] (x+8)
\begin{scope}[shift={(8.0,0)}]
\draw[step=1cm, gray, very thin] (0, -3) grid (3, 0);
\fill[blue!25!white] (0.5, -0.5) circle (0.25cm);
\fill[green!25!white] (2.5, -2.5) circle (0.25cm);
% draw IS
\fill[red!25!white] (0.5, -1.5) circle (0.25cm);
% Fill value
\node [font=\footnotesize] (c11) at (0.5,-0.5) {$0$};
\node [font=\footnotesize] (c12) at (1.5,-0.5) {$0$};
\node [font=\footnotesize] (c13) at (2.5,-0.5) {$0$};
\node [font=\footnotesize] (c21) at (0.5,-1.5) {$8.1$};
\node [font=\footnotesize] (c22) at (1.5,-1.5) {$0$};
\node [font=\footnotesize] (c23) at (2.5,-1.5) {$0$};
\node [font=\footnotesize] (c31) at (0.5,-2.5) {$9$};
\node [font=\footnotesize] (c32) at (1.5,-2.5) {$10$};
\node [font=\footnotesize] (c33) at (2.5,-2.5) {};
% draw k = 1
\node [font=\footnotesize] (c) at (1.5,-3.5) {(c) $r(s,a,s_-) = -10$};
% draw boundary
\draw[black, line width=1.5] (3.0, -2.0) -- (1.0, -2.0);
\draw[black, line width=1.5] (0.0, 0.0) -- (3.0, 0.0) -- (3.0, -3.0) -- (0.0, -3.0) -- cycle;
\end{scope}

%% legend
% Bounding Box
\draw[black] (12,0) -- (12,-2.3) -- (13.4,-2.3) -- (13.4,0) -- cycle;

% Drawing
\fill[blue!25!white] (13, -0.4) circle (0.25cm);
\fill[red!25!white] (13, -1.15) circle (0.25cm);
\fill[green!25!white] (13, -1.9) circle (0.25cm);

% Caption
\node [right,font=\footnotesize] (A) at (12.0,-0.4) {$s_0$};
\node [right,font=\footnotesize] (B) at (12.0,-1.15) {$s_{-}$};
\node [right,font=\footnotesize] (C) at (12.0,-1.9) {$\mc S_{T}$};
\end{tikzpicture}
    \caption{An example of the value functions after $k=4$ iterations of SVI and with terminal rewards $B = 10,\gamma = 0.9$, with different structures and magnitude of negative intermediate rewards.}
    \label{fig:NEGReward}
\end{figure}
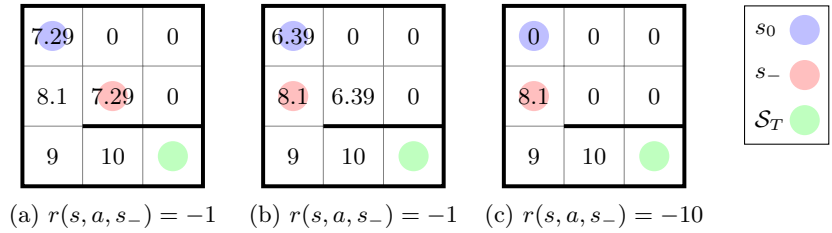

\paragraph{Unequal Magnitude Intermediate Rewards}
Another limitation of this work is the equal magnitude intermediate reward. In Section \ref{subsec:rewards}, we assume all intermediate rewards are equal magnitude $B_I$ (but can be different from the terminal rewards $B$). Such equal magnitude assumption helps us provide a bound w.r.t. relative magnitude between terminal rewards and intermediate rewards (the bound of $B/B_I$ described in Theorem \ref{thm:find_closest_IS} and Theorem \ref{thm:find_closest_TerminalState}), hence ensuring the monotonicity of the value function $V_k(s)$ w.r.t. $D(s,\mc S_I)$, $D(s,\mc S_T)$ in Theorem \ref{thm:find_closest_IS}, Theorem \ref{thm:find_closest_TerminalState}, respectively. One way to generalize the equal magnitude intermediate rewards is to assume all intermediate rewards are bounded between $[B_l,B_r]$, and then apply the same techniques as Theorem \ref{thm:find_closest_IS} and \ref{thm:find_closest_TerminalState} to provide bounds that involve $B/B_l$ and $B/B_u$.

\subsection{Future Works}
Our work can be extended in these following directions: 
1) Generalize from the deterministic MDPs in our framework to stochastic MDPs;
2) Study the tightness of the bounds for the relative magnitude $B/B_I$ in Theorem \ref{thm:find_closest_IS} and \ref{thm:find_closest_TerminalState}; 3) Generalize to other scenarios that contains negative rewards (e.g., from ``bad'' terminal states or intermediate states); 4) Generalize to intermediate rewards of unequal magnitude; 5) Study how rewarding ``good'' non one-way intermediate states (instead of just one-way states) affects the goal-reaching problems in general.

\section*{Acknowledgements}
We would like to thank anonymous reviewers of JAIR for their insightful suggestions on revising the draft, especially on highlighting the main results and contributions in the introduction. We would also like to thank Professor Mykel Kochenderfer of Stanford for coordinating the revision. We would also like to thank Michael Psenka from UC Berkeley for proofreading the manuscript. YZ would like to thank Haozhi Qi from UC Berkeley for insightful discussion on the implementations of practical deep RL tasks. YZ would also like to thank professor Sergey Levine and Qiyang (Colin) Li from the RAIL lab for insightful suggestions on the connections between this work and goal-conditioned RL. CB would also like to thank Siddharth Verma from UC Berkeley for helpful discussions about deep RL algorithms. ZZ acknowledges the JP Morgan AI Research Grant. JJ acknowledges NSF Grants IIS-1901252 and CCF-1909499. YM acknowledges support from ONR grant N00014-20-1-2002 and the joint Simons Foundation-NSF DMS grant \#2031899, as well as support from Berkeley FHL Vive Center for Enhanced Reality and Berkeley Center for Augmented Cognition. YZ and YM acknowledge support from Tsinghua-Berkeley Shenzhen Institute (TBSI) Research Fund.

\appendix
\section{Extra Definition and Proof of Section \ref{sec:MainResult}}
\subsection{Definition of Correct and Incorrect Actions}
We introduce the definition of correct and incorrect actions to facilitate the future proof.
\begin{definition}[Correct and Incorrect Actions]
\label{def:correct_actions}
Let $s_a$ be the subsequent state the state-action pair $(s,a)$, we define the {\em correct action set} $\mc A^+(s)$ of a state $s\in \mc S\backslash \mc S_{T}$ as
    \begin{equation}
        \mc A^+(s) = \{a|a\in \mc A,D(s_a,\mc S_{T})=D(s,\mc S_{T})-1\}.
    \end{equation}
    Conversely, we can define the {\em incorrect action set} of a state $s\in \mc S$ as 
    \begin{equation}
        \mc A^-(s) = \{a|a\in \mc A\backslash\mc A^+(s)\}.
    \end{equation}
\end{definition}
\subsection{Proof of Proposition \ref{prop:comp_sparse_reward}}
\begin{proposition}[Sparse Rewards]
Let $\mc M = (\mc S,\mc A, P, r, \gamma)$ be a deterministic MDP with  initial state $s_0$ and terminal states $\mc S_{T}$. If the reward function $r(\cdot)$ follows the sparse reward setting (Table \ref{tab:reward_setting}) and the value function and Q-function are zero-initialized \eqref{eq:zero_init}, then after any $k\geq D(s_0,\mc S_{T})$  synchronous value iteration updates \eqref{eq:SyncValueIter}, the Q-function $Q_k$ is a successful Q-function, and a greedy policy follows the shortest path from $s_0$ to $\mc S_T$.
\end{proposition}
\begin{proof}
\label{proof:absorb_radius_Q_learning}
$\forall d\leq k$, let $\mc S^{d}$ denote the set of states that is distance $d$ to the desired terminal state $\mc S_{T}$, and let $a^{d+}\in \mc A^+(s^d)$, $a^{d-}\in \mc A^-(s^d)$ denote a correct action and an incorrect action of $s^d$, respectively. From Lemma \ref{lemma:Vk_bounds_induction_SR}, we know that 
\begin{equation*}
    V_k(s^d) = \begin{cases}
        \gamma^{d-1}B, & \forall k,d\in \bb N^+, d\leq k,\\
        0, & \text{ otherwise}.\\    
    \end{cases}
\end{equation*}
Combine Lemma \ref{lemma:Vk_bounds_induction_SR} with the Q-function update in value iteration \eqref{eq:SyncValueIter}:
\begin{equation*}
    Q_k(s,a) = r(s,a,s_a)+\gamma V_k(s_a),\;\forall (s,a)\in \mc S\times \mc A,
\end{equation*}
we know that 
\begin{align}
    Q_{k}(s^d,a^{d+}) &\overset{(i)}{=} 
    \begin{cases}
        \gamma^{d-1}B,&\text{ if }d\leq k,\\
        0,&\text{ otherwise},    
    \end{cases}\label{eq:Qk+_induction_SR}\\
    Q_{k}(s^d,a^{d-}) &\overset{(ii)}{\leq}
    \begin{cases}
        \gamma^{d}B,&\text{ if }d\leq k,\\
        0,&\text{ otherwise},
    \end{cases} \label{eq:Qk-_induction_SR}
\end{align}
where inequality $(ii)$ holds since Definition \ref{def:correct_actions} implies that $D(s_{a^{d-}}^{d},\mc S_{T})\geq D(s,\mc S_{T})$. We shall clarify equality $(i)$ for the case when $d=1$: in this case, $V_k(s) = 0$ always holds $\forall s\in \mc S_{T}$ because given a state-action pair $(s,a)$, the MDP stops once the subsequent state $s_a\in \mc S_{T}$, hence we have 
\begin{equation}
    Q_k(s^1,a^{1+}) = r(s^1,a^{1+},s^1_{a^{1+}})+\gamma V_k(s^1_{a^{1+}}) = B.
\end{equation}
Since $\gamma<1$, we know that 
\begin{equation}
    \label{eq:Qk-<Qk+_SR}
    Q_{k}(s^d,a^{d-})<Q_k(s^d,a^{d+}),\;\forall k\geq d,
\end{equation}
which implies
\begin{equation}
    \underset{a\in \mc A}{\arg\max}\;Q_k(s^d,a)\in \mc A^+(s^d),\;\forall k\geq d.
\end{equation}
Now consider $s_0$, let $d_m = D(s_0,\mc S_{T})$, \eqref{eq:Qk-<Qk+_SR} implies that any greedy action is a correct action:
\begin{equation}
    a_0 = \underset{a}{\arg\max}\; Q_k(s_0,a)\in \mc A^+(s_0), \;\forall k\geq d_m
\end{equation}
By Definition \ref{def:correct_actions}, the subsequent state $s_{1}$ of state-action pair $(s_0,a_0)$ has distance at most $d_m$ $(D(s_1,\mc S_{T})\in [d_m-1,d_m])$ to $\mc S_{T}$, again \eqref{eq:Qk-<Qk+_SR} implies that 
\begin{equation}
    a_1 = \underset{a}{\arg\max}\; Q_k(s_1,a)\in \mc A^+(s_1).
\end{equation}
Likewise, we know that the greedy policy generates a trajectory $\{s_0,a_0,s_1,a_1,\dots\}$, such that
\begin{equation}
    \label{eq:greedy_action_correct_STR}
    a_i = \underset{a}{\arg\max}\; Q_k(s_i,a)\in \mc A^+(s_i),
\end{equation}
which eventually ends up with a state $s_n\in \mc S_{T}$, since Definition \ref{def:correct_actions} indicates that $s_{i+1}$ is one step closer to $\mc S_{T}$ than $s_{i}$. Hence, we conclude that $Q_k$ is a successful Q-function. When $k\geq d$, since \eqref{eq:greedy_action_correct_STR} shows that every action taken by the greedy policy will take the agent one step closer to $\mc S_T$, we can also conclude that the greedy policy also pursues the shortest path from $s_0$ to $\mc S_T$. 
\end{proof}

\subsection{Proof of Proposition \ref{prop:com_comp_inter_reward}}
\begin{proposition}[Single-path Intermediate States]
Let $\mc M = (\mc S,\mc A, P, r, \gamma)$ be a deterministic MDP with initial state $s_0$, intermediate states $\mc S_{I} = \{s_{i_1},s_{i_2},\dots,s_{i_N}\}$, and terminal states $\mc S_{T}$. Suppose $\mc M$ satisfies Assumption \ref{assump:OneWayIntermediateStates} and \ref{assump:diff_settings_IS} (a). If the reward function $r(\cdot)$ follows the intermediate reward setting (Table \ref{tab:reward_setting}) and the value function and Q-function are zero-initialized \eqref{eq:zero_init}, then after
\begin{equation}
    k\geq d_{\max}\doteq\max\{D(s_0,s_{i_1}),D(s_{i_1},s_{i_2}),\dots, D(s_{i_{N-1}},s_{i_N}), D(s_{i_N},\mc S_{T})\}
\end{equation} 
synchronous value iteration updates \eqref{eq:SyncValueIter}, the Q-function $Q_k$ is a successful Q-function, and a greedy policy follows the shortest path from $s_0$ to $\mc S_T$.
\end{proposition}
\begin{proof}
\label{proof:complexity_inter_reward}
The proof of this Proposition is a direct result of Lemma \ref{lemma:Vk_bounds_induction_IR_a}. Similar to Lemma \ref{lemma:Vk_bounds_induction_IR_a}, $\forall j\in [n+1]$, let $\mc S_{[j-1,j)}$ denote the sets
\begin{equation}
    \mc S_{[j-1,j)} = 
    \begin{cases}
        \{s\in \mc S|D(s,s_{i_{1}})<\infty\}& \text{ if } j = 1,\\
        \{s\in \mc S|D(s,s_{i_{j}})=\infty,D(s,s_{i_{j+1}})<\infty\}& \text{ if } j = 2,3,\dots,N,\\
        \{s\in \mc S|D(s,s_{i_{n}})=\infty,D(s,\mc S_{T})<\infty\}   & \text{ if } j = N+1.
    \end{cases}
\end{equation}
For a state $s^{\mb d}\in \mc S_{[j-1,j)}$, let 
\begin{equation}
    \mb d = [d_j,d_{j+1},\dots,d_{n+1}]^\top\in \bb R^{n+1},
\end{equation}
where $d_l = D(s,s_{i_l}), \;\forall l=j,j+1,\dots, N$ and $ d_{N+1} = D(s,\mc S_{T})$, denote a vector whose entries are the distance from $s$ to intermediate states $s_{i_j}$ (and the terminal states $\mc S_{T}$). Let $a^{\mb d+}\in \mc A^+(s^{\mb d})$, $a^{\mb d-}\in \mc A^-(s^{\mb d})$ denote a correct and an incorrect action of $s^{\mb d}$. Combine Lemma \ref{lemma:Vk_bounds_induction_IR_a} with the Q-function update in value iteration \eqref{eq:SyncValueIter}
\begin{equation*}
    Q_k(s,a) = r(s,a,s_a)+\gamma V_k(s_a),\;\forall (s,a)\in \mc S\times \mc A,
\end{equation*}
we know that

\begin{align}
        &Q_{k}(s^{\mb d},a^{\mb d+})= \sum_{l=j}^{N+1} v(k,d_l)B_l=V_k(s),\\ 
        &Q_{k}(s^{\mb d},a^{\mb d-})\overset{(i)}{\leq} \gamma\sum_{l=j}^{N+1}v(k,d_l)B_l=\gamma V_k(s),
\end{align}
where
\begin{equation}
    v(k,d_l)=
    \begin{cases}
        \gamma^{d_l-1}, & \forall k,d\in \bb N^+, d_l\leq k+1,\forall l\in [N+1],\\
        0, & \text{ otherwise},\\    
    \end{cases}
    \;\text{ and }\;
    B_l = \begin{cases}
        B_I,&\text{ if }\; l\in [N],\\
        B,&\text{ if }\; l = N+1,
    \end{cases}
\end{equation}
and inequality $(i)$ holds because Definition \ref{def:correct_actions} implies $D(s_{a^{\mb d-}}^{\mb d},\mc S_{T})\geq D(s^{\mb d},\mc S_{T})$ and Assumption \ref{assump:OneWayIntermediateStates} guarantees that previous intermediate states $s_{i_1},s_{i_2},\dots,s_{i_{j-1}}$ cannot be revisited. Since $\gamma<1$, notice that when $k\geq d_j$, we have 
\begin{equation}
    \label{eq:Qk-<Qk+_IR_a}
    Q_k(s^{\mb d},a^{\mb d+})>Q_k(s^{\mb d},a^{\mb d-}), 
\end{equation}
which implies 
\begin{equation}
    \underset{a\in \mc A}{\arg\max}\;Q_k(s^{\mb d+},a)\in \mc A^+(s^{\mb d}).
\end{equation}
Now consider $s_0$, since we assume that $k\geq d_{\max}$, therefore \eqref{eq:Qk-<Qk+_IR_a} implies that the greedy action is a correct action:
\begin{equation}
    a_0 = \underset{a}{\arg\max}\; Q_k(s_0,a)\in \mc A^+(s_0).
\end{equation} 
By Definition \ref{def:correct_actions} and Assumption \ref{assump:OneWayIntermediateStates}, we consider the following two cases:
\begin{itemize}
    \item The subsequent state $s_{1}$ of state-action pair $(s_0,a_0)$ {\em does not} reach $s_{i_1}$, $(s_1\in \mc S_{[0,1)})$, then $s_1$ has distance at most $d_1$ to states $s_{i_1}$:
\begin{equation}
    \begin{split}
        D(s_1,s_{i_1})\in [d_{1}-1,d_{1}].
    \end{split}
\end{equation} 
Since $k\geq d_{\max}\geq d_1$, \eqref{eq:Qk-<Qk+_IR_a} implies that the greedy action is a correct action:
\begin{equation}
    a_1 = \underset{a}{\arg\max}\; Q_k(s_1,a)\in \mc A^+(s_1).
\end{equation}
\item The subsequent state $s_{1}$ of state-action pair $(s_0,a_0)$ reaches $s_{i_1}$, $(s_1\in \mc S_{[1,2)})$, then $s_1$ has distance at most $d_2$ to states $s_{i_2}$:
\begin{equation}
    \begin{split}
        D(s_1,s_{i_2})&\in [d_{2}-1,d_{2}].
    \end{split}
\end{equation} 
Since $k\geq d_{\max}\geq d_2$, \eqref{eq:Qk-<Qk+_IR_a} implies that the greedy action is a correct action:
\begin{equation}
    a_1 = \underset{a}{\arg\max}\; Q_k(s_1,a)\in \mc A^+(s_1).
\end{equation}
\end{itemize}
Likewise, we know that greedy execution generates a trajectory $\{s_0,a_0,s_1,a_1,\dots\}$, such that
\begin{equation}
    \label{eq:greedy_action_correct_OWSP}
    a_i = \underset{a}{\arg\max}\; Q_k(s_i,a)\in \mc A^+(s_i),
\end{equation}
which eventually ends up with a state $s_n\in \mc S_{T}$, since Definition \ref{def:correct_actions} indicates that $s_{i+1}$ is one step closer to $\mc S_{T}$ than $s_{i}$. Hence, we conclude that $Q_k$ is a successful Q-function. When $k\geq d_{\max}$, since \eqref{eq:greedy_action_correct_OWSP} shows that every action taken by the greedy policy will take the agent one step closer to $\mc S_T$, we can also conclude that the greedy policy also pursues the shortest path from $s_0$ to $\mc S_T$. 
\end{proof}

\subsection{Proof of Theorem \ref{thm:find_closest_IS}}
\begin{theorem}[Finding the Closest $\mc S_I$]
Let $\mc M = (\mc S,\mc A, P, r, \gamma)$ be a deterministic MDP with initial state $s_0$, intermediate states  $\mc S_{I} = \{s_{i_1},s_{i_2},\dots,s_{i_N}\}$, and terminal states $\mc S_{T}$. Suppose $\mc M$ satisfies Assumption \ref{assump:OneWayIntermediateStates}, \ref{assump:diff_settings_IS} (b), and \ref{assump:min_dist_IS}, if the reward function $r(\cdot)$ follows the intermediate reward setting (Table \ref{tab:reward_setting}) and the value function and Q-function are zero-initialized \eqref{eq:zero_init}, then $\forall s \in \mc S\backslash\mc S_{T}$, if $\mc S_{T}$ is not directly reachable from $s$, then after
\begin{equation}
    k\geq d = D(s,\mc S_{I}) \doteq \min_{j\in \mc I_d(s)}D(s,s_{i_j})
\end{equation}
synchronous value iteration update \eqref{eq:SyncValueIter}, an agent following the greedy policy will find the closest directly reachable intermediate state $s_{i_{j^\star}}$ of $s$ ($D(s,s_{i_{j^\star}}) = D(s,\mc S_{I})$), given that $\frac{B}{B_I}\in \paren{0,\frac{1}{1-\gamma^h}}$ when $\gamma+\gamma^h\leq 1$ or $\frac{B}{B_I}\in\paren{\frac{1}{1-\gamma^h},\frac{1-\gamma}{\gamma^{1+h}}}$ when $\gamma+\gamma^h>1$, where $h$ is the minimum distance between two intermediate states (Assumption \ref{assump:min_dist_IS}).
\end{theorem}
\begin{proof}
\label{proof:find_closest_IS}
Recall the definition of directly reachable intermediate states (Definition \ref{def:DirectlyRechableIS}) of state $s$: $\{s_{i_j},\;\forall j\in \mc I_d(s)\}$. $\forall s^\prime\in \mc S$, such that
\begin{equation}
    \label{eq:dprime>d}
    d^\prime = D(s^\prime,\mc S_I) = \min_{j\in \mc I_d(s^\prime)} D(s^\prime,s_{i_j})>d
\end{equation}
is the minimum distance from $s^\prime$ to its closest directly reachable intermediate state, it suffices to show $V_k(s)>V_k(s^\prime)$, since $V_k(s)>V_k(s^\prime)$ can directly lead to $Q_k(s,a_{i+})>Q_k(s,a_{i-})$:
\begin{equation}
    \begin{split}
        Q_{k}(s,a_{i+}) &\overset{(i)}{=} r(s,a_{i+},s_{a_{i+}}) + \gamma V_k(s_{a_{i+}}) \overset{(ii)}{>} \gamma V_k(s_{a_{i-}})\\
        &\overset{(iii)}{=} r(s,a_{i-},s_{a_{i-}}) + \gamma V_k(s_{a_{i-}}) \overset{(iv)}{=} Q_{k}(s,a_{i-}),
    \end{split}
\end{equation}
where $a_{i+}\in \mc A_I^+(s)$ is an action that leads to the closest directly reachable intermediate states of state $s$ and $a_{i-}\in \mc A_I^-(s)$ is the other actions that do not. Equality $(i)$ and $(iv)$ hold by the synchronous value iteration update \eqref{eq:SyncValueIter}, inequality $(ii)$ holds since $D(s_{a_{i-}},\mc S_I)>D(s_{a_{i+}},\mc S_I)$ and we have assumed $V_k(s)>V_k(s^\prime)$ holds $\forall s,s^\prime$ such that $D(s^\prime,\mc S_I) > D(s,\mc S_I)$, and equality $(iii)$ holds because $s_{a_{i-}}$ is not an intermediate state. Moreover, if the agent at $s$ takes an action $a_{i+}$ and moves to $s_{a_{i+}}$, we will have 
\begin{equation}
    k\geq d-1 = D(s_{a_{i+}},s_{i_j})
\end{equation}
if $s_{a_{i+}}$ is not an intermediate state, which implies that the action given by the greedy policy will still lead $s_{a_{i+}}$ one step forward to the closest directly reachable intermediate states. Therefore, we only need to show $V_k(s)>V_k(s^\prime),\forall s,s^\prime$ such that $D(s^\prime,\mc S_I)>D(s,\mc S_I)$, which will be our main focus in the remaining proof.

\paragraph{When $d<k<d+h$.}
In this case, all intermediate rewards are the same (all equal to $B_I$), Lemma \ref{lemma:Vk_bounds_induction_SR} implies that 
\begin{equation}
    V_k(s) = v(k,d)B_I,\; \forall k<d+h,
\end{equation}
where $d = D(s,\mc S_I)$ and
\begin{equation}
    v(k,d)=
    \begin{cases}
        \gamma^{d-1}, & d\leq k,\\ 
        0, & \text{ otherwise}.\\    
    \end{cases}
\end{equation}
Proposition \ref{prop:comp_sparse_reward} implies that the agent will find a path to the closest intermediate state $s_{i_{j^\star_0}}$:
\begin{equation}
    D(s,s_{i_{j^\star_0}}) = D(s,\mc S_I),
\end{equation}
and 
\begin{equation}
    V_k(s) = \gamma^{d-1}B_I > \gamma^dB_I \geq V_k(s^\prime).
\end{equation}

\paragraph{When $k\geq d + h$.}
To complete this theorem, we need to show that, when $k\geq d+h$, $V_k(s)>V_k(s^\prime)$ still holds. Let $\pi$ be any deterministic policy and suppose under policy $\pi$, the agent starting from $s^\prime$ will visit the sequence of intermediate states before reaching $\mc S_{T}$:
\begin{equation}
   \{s_{i_{j^\star_0}},s_{i_{j^\star_1}},\dots,s_{i_{j^\star_{u^\pi}}},s_{i_{j^\star_{u^\pi+1}}}\},
\end{equation}
where we slightly abuse the notation by assuming $s_{i_{j^\star_{u^\pi+1}}}\in \mc S_{T}$. Hence, the discounted reward $V_k^\pi(s^\prime)$ starting from state $s^\prime$ following policy $\pi$ satisfies
\begin{equation}
    \begin{split}
    \label{eq:discounted_value_upperbound}
    V_k^\pi(s^\prime) \leq \underbrace{\gamma^{d^\prime-1}B_I + \brac{ \sum_{l=1}^{u^\pi}\gamma^{d_l-1} }B_I+\gamma^{d_{u^\pi + 1}-1}B}_{\Gamma(\pi)},
    \end{split}
\end{equation}
where 
\begin{equation}
    \label{eq:dist_bound}
        d_l = D(s^\prime,s_{i_{j^\star_0}})+\sum_{m=0}^{l-1}D(s_{i_{j^\star_m}},s_{i_{j^\star_{m+1}}})\overset{(i)}{\geq}d^\prime+lh,\;\forall l = 1,2,\dots, u^\pi+1,
\end{equation}
where inequality $(i)$ hold because of Assumption \ref{assump:min_dist_IS}. Hence, we know that
\begin{equation}
    \Gamma(\pi) = \gamma^{d^\prime-1}B_I + \brac{ \sum_{l=1}^{u^\pi}\gamma^{d_{l}-1} }B_I+\gamma^{d_{u^\pi + 1}-1}B\overset{(ii)}{\leq}\underbrace{\gamma^{d^\prime-1}\brac{\sum_{l=0}^{u^\pi}\gamma^{l h}}B_I+\gamma^{d^\prime+(u^\pi + 1) h-1}B}_{f(u^\pi)},
\end{equation}
where inequality $(ii)$ holds because of \eqref{eq:dist_bound} and $\gamma<1$. By the synchronous value iteration update \eqref{eq:SyncValueIter}, and \eqref{eq:discounted_value_upperbound} the value function of any state $s,s^\prime$ satisfy:
\begin{equation}
    \label{eq:discounted_value_approx}
    \gamma^{d-1}B_I\leq V_k(s),\; V_k(s^\prime) = \max_\pi V_k^\pi(s^\prime)\leq \max_\pi f(u^\pi),
\end{equation}
Next we will show $\gamma^{d-1}B_I>f(u^\pi)$ holds $\forall \pi$, (which directly leads to $V_k(s)>V_k(s^\prime)$, for any policy from \eqref{eq:discounted_value_approx}, after taking the max over all policy $\pi$) under these two following conditions: 1) $\frac{B}{B_I}\in \paren{0,\frac{1}{1-\gamma^h}}$ when $\gamma+\gamma^h\leq 1$; 2) $\frac{B}{B_I}\in\paren{\frac{1}{1-\gamma^h},\frac{1-\gamma}{\gamma^{1+h}}}$ when $\gamma+\gamma^h>1$.

\begin{itemize}
    \item {\bf When $\frac{B}{B_I}\in \paren{0,\frac{1}{1-\gamma^h}}$ and $\gamma+\gamma^h\leq 1$.}
    
    Since $\frac{B}{B_I}\in \paren{0,\frac{1}{1-\gamma^h}}$, we know that $B_I-(1-\gamma^h)B>0$, hence
    \begin{equation}
    \begin{split}
        f(u^\pi) & < f(u^\pi) + \gamma^{d^\prime+(u^\pi + 1) h-1}\brac{B_I-(1-\gamma^h)B} \\
        & = \gamma^{d^\prime-1}\brac{\sum_{l=0}^{u^\pi}\gamma^{l h}}B_I+\gamma^{d^\prime+(u^\pi + 1) h-1}B+\gamma^{d^\prime+(u^\pi + 1) h-1}\brac{B_I-(1-\gamma^h)B}\\
        & = \gamma^{d^\prime-1}\brac{\sum_{l=0}^{u^\pi + 1}\gamma^{l h}}B_I+\gamma^{d^\prime+(u^\pi+2)h-1}B = f(u^\pi+1) <\dots < f(\infty) \\
        & = \gamma^{d^\prime-1}\brac{\sum_{l=0}^\infty\gamma^{l h}}B_I = \frac{\gamma^{d^\prime-1}}{1-\gamma^h}B_I\overset{(i)}{\leq}\frac{\gamma^{d}}{1-\gamma^h}B_I\overset{(ii)}{\leq}\gamma^{d-1}B_I,
    \end{split}
\end{equation}
where inequality $(i)$ holds because \eqref{eq:dprime>d} implies $d^\prime \geq d+1$, and inequality $(ii)$ holds because $\gamma+\gamma^h\leq 1$ implies that $\frac{\gamma}{1-\gamma^h}\leq 1$.

\item {\bf When $\frac{B}{B_I}\in\left[\frac{1}{1-\gamma^h},\frac{1-\gamma}{\gamma^{1+h}}\right)$ and $\gamma+\gamma^h>1$.}

When $\gamma+\gamma^h>1$, we know that the interval $\paren{\frac{1}{1-\gamma^h},\frac{1-\gamma}{\gamma^{1+h}}}$ is well defined as $\gamma+\gamma^h>1$ implies $\frac{1}{1-\gamma^h}<\frac{1-\gamma}{\gamma^{1+h}}$. Since $\frac{B}{B_I}>\frac{1}{1-\gamma^h}$, we know that $(1-\gamma^h)B-B_I>0$, hence
\begin{equation}
    \begin{split}
        f(u^\pi) & < f(u^\pi) + \gamma^{d^\prime+u^\pi h-1}\brac{(1-\gamma^h)B-B_I}\\
        & = \gamma^{d^\prime-1}\brac{\sum_{l=0}^{u^\pi}\gamma^{l h}}B_I+\gamma^{d^\prime+(u^\pi + 1) h-1}B+\gamma^{d^\prime+u^\pi h-1} \brac{(1-\gamma^h)B-B_I}\\
        &=\gamma^{d^\prime-1}\brac{\sum_{l=0}^{u^\pi-1}\gamma^{l h}}B_I+\gamma^{d^\prime+u^\pi h-1}B=f(u^\pi-1)<\dots<f(0)\\
        &=\gamma^{d^\prime-1}B_I+\gamma^{d^\prime+h-1}B\overset{(i)}{\leq}\gamma^{d}B_I+\gamma^{d+h}B\overset{(ii)}{\leq}\gamma^{d-1} B_I. 
    \end{split}
\end{equation}
where inequality $(i)$ holds because \eqref{eq:dprime>d} implies $d^\prime \geq d+1$, and inequality $(ii)$ holds because $\frac{B}{B_I}\leq\frac{1-\gamma}{\gamma^{1+h}}$ implies that $\gamma B_I+\gamma^{1+h}B\leq B_I$. 
\end{itemize}

Hence, we conclude that $f(u^\pi)<\gamma^{d-1}B_I,\forall \pi$, which leads to the result that $V_k(s)>V_k(s^\prime)$, $\forall k\geq d+h$ and $\forall s,s^\prime\in \mc S$ satisfying 
\eqref{eq:dprime>d}. This result indicates that we still have $V_k(s)>V_k(s^\prime)$ in the future update of synchronous value iteration. Therefore, we conclude that $V_k(s)>V_k(s^\prime),\forall s,s^\prime \in \mc S$ satisfying $D(s^\prime,\mc S_I)> D(s,\mc S_I)$, which completes the proof.
\end{proof}

\subsection{Proof of Theorem \ref{thm:find_closest_TerminalState}}
\begin{theorem}[Finding the Shortest Path to $\mc S_T$]
Let $\mc M = (\mc S,\mc A, P, r, \gamma)$ be a deterministic MDP with initial state $s_0$, intermediate states  $\mc S_{I} = \{s_{i_1},s_{i_2},\dots,s_{i_N}\}$, and terminal states $\mc S_{T}$. Suppose $\mc M$ satisfies Assumption \ref{assump:OneWayIntermediateStates}, \ref{assump:diff_settings_IS} (b), and \ref{assump:min_dist_IS}, the reward function $r(\cdot)$ follows the intermediate reward setting (Table \ref{tab:reward_setting}), and the value function, Q-function are zero-initialized \eqref{eq:zero_init}. $\forall s \in \mc S\backslash\mc S_{T}$, suppose $\mc S_{T}$ is directly reachable from $s$, and let 
\begin{equation}
     d = D(s,\mc S_{T})\; \text{ and }\; d_I = D(s,\mc S_{I}) \doteq \min_{j\in \mc I_d(s)} D(s,s_{i_j}).
\end{equation}
If $d$ and $d_I$ satisfies:
\begin{equation}
    \label{eq:d_cond_find_ST}
    d<\begin{cases}
        d_I+\log_{\frac{1}{\gamma}}\brac{(1-\gamma^h)\frac{B}{B_I}},&\text{ if }\;\frac{B}{B_I}<\frac{1}{1-\gamma^h},\\
        d_I+\log_{\frac{1}{\gamma}}\paren{\frac{B}{B_I+\gamma^hB}},&\text{ if }\;\frac{B}{B_I}\geq \frac{1}{1-\gamma^h},
    \end{cases}
    \;\text{ and }\; d<d_I+h-1,
\end{equation}
where $h$ is the minimum distance between two intermediate states (Assumption \ref{assump:min_dist_IS}), then after $k\geq d$ synchronous value iteration updates \eqref{eq:SyncValueIter}, an agent following the greedy policy will pursue the shortest path to $\mc S_{T}$. 
\end{theorem}
\begin{proof}
\label{proof:find_closest_TerminalState}
By the update of Q-function in the synchronous value iteration update \eqref{eq:SyncValueIter}, it suffices to show 
\begin{equation}
    \label{eq:qk(s,a+)>qk(s,a-)ToST}
    Q_k(s,a^+) = r(s,a^+,s_{a^+})+\gamma V_k(s_{a^+})>r(s,a^-,s_{a^-})+\gamma V_k(s_{a^-})=Q_k(s,a^-),
\end{equation}
when condition \eqref{eq:d_cond_find_ST} is satisfied. Lemma \ref{lemma:Vk_Bound_directly_reachable_k<dI+h} implies that when $k<d_I+h$, we have 
\begin{equation}
    V_k(s) = \max\{v(k,d_I)B_I,v(k,d)B\},
\end{equation}
where 
\begin{equation}
    v(k,d)=
    \begin{cases}
        \gamma^{d-1}, & \forall k,d\in \bb N^+, d\leq k,\\
        0, & \text{ otherwise}.\\    
    \end{cases}
\end{equation}
Hence we will focus on proving \eqref{eq:qk(s,a+)>qk(s,a-)ToST} in the remaining proof.
\paragraph{When $d\leq k<d_I+h-1$.}
In this case, the conditions in \eqref{eq:d_cond_find_ST} first guarantee it exists a $k$ such that $d\leq k<d_I+h-1$. First, we will show both conditions
\begin{equation*}
    d<d_I+\log_{\frac{1}{\gamma}}\brac{(1-\gamma^h)\frac{B}{B_I}}\; \text{ and }\; d<d_I+\log_{\frac{1}{\gamma}}\paren{\frac{B}{B_I+\gamma^hB}},
\end{equation*}
indicate that
\begin{equation}
    \label{eq:v(k,d_I)B_I<v(k,d)B}
    \gamma^{d_I}B_I<\gamma^dB.
\end{equation}
\begin{itemize}
    \item For $d<d_I+\log_{\frac{1}{\gamma}}\brac{(1-\gamma^h)\frac{B}{B_I}}$, we have 
    \begin{equation}
        \begin{split}
            &d-d_I<\log_{\frac{1}{\gamma}}\brac{(1-\gamma^h)\frac{B}{B_I}} \iff\paren{\frac{1}{\gamma}}^{d-d_I}<(1-\gamma^h)\frac{B}{B_I}\\
            \iff &\gamma^{d_I} B_I<(1-\gamma^h)\gamma^dB\implies\gamma^{d_I} B_I<\gamma^dB.
        \end{split}
    \end{equation}
    \item For $d<d_I+\log_{\frac{1}{\gamma}}\paren{\frac{B}{B_I+\gamma^hB}}$, we have 
\begin{equation}
    \begin{split}
        &d-d_I<\log_{\frac{1}{\gamma}}\paren{\frac{B}{B_I+\gamma^hB}}\iff \paren{\frac{1}{\gamma}}^{d-d_I}<\frac{B}{B_I+\gamma^h B}\\
        \iff& \gamma^{d_I}(B_I+\gamma^hB)<\gamma^dB\implies \gamma^{d_I} B_I<\gamma^dB.
    \end{split}
\end{equation}
\end{itemize}
Now consider a correct action $a^+\in \mc A^+(s)$ and an incorrect action $a^-\in \mc A^-(s)$, we will next show
\begin{equation}
    \label{eq:Qk(s,a+)_value_Q_k(s,a-)_upper_bound}
    Q_k(s,a^+)=v(k,d)B,\; \text{ and }\; Q_k(s,a^-)\leq \max\{\gamma v(k,d_I-1)B_I,\gamma v(k,d)B\}.    
\end{equation}

\begin{itemize}
    \item For $Q_k(s,a^+)$, when $d=1$, we have
    \begin{equation}
        Q_k(s,a^+) = r(s,a,s_{a^+}) + \gamma V_k(s_{a^+}) \overset{(i)}{=} B = v(k,1)B,
    \end{equation}
    where equality $(i)$ holds because $s_{a^+}\in \mc S_{T}$, hence $r(s,a,s_{a^+}) = B$ and $V_k(s_{a^+}) = 0$. 
    
    When $d>1$, we have
    \begin{equation}
        Q_k(s,a^+) = r(s,a,s_{a^+}) + \gamma V_k(s_{a^+}) \overset{(ii)}{=} \gamma v(k,d-1)B \overset{(iii)}{=} v(k,d)B.
    \end{equation}
    Equality $(ii)$ holds because $r(s,a,s_{a^+}) = 0$ and Lemma \ref{lemma:Vk_Bound_directly_reachable_k<dI+h} implies $V_k(s_{a^+}) = v(k,d-1)B$. Equality $(iii)$ holds because when $k\geq d$, $\gamma v(k,d-1) = \gamma^{d-1}=v(k,d)$. Hence, we conclude that $Q_k(s,a^+) = v(k,d)B$.
    \item For $Q_k(s,a^-)$, consider these two following subsets of $\mc A^-(s)$: $a_{i+}^-\in \mc A^+_I(s)\cap\mc A^-(s)$ and $a_{i-}^-\in \mc A\backslash(\mc A^+_I(s)\cup\mc A^+(s))$, where $\mc A^+_I(s)$ is the set of actions that take state $s$ one step closer to $\mc S_I$:
    \begin{equation}
        % \label{eq:defAI+}
        \mc A^+_I(s) \doteq \{a|a\in \mc A,D(s_a,\mc S_{I})=D(s,\mc S_{I})-1\},
    \end{equation}
    and $D(s,\mc S_{I})$ is defined as
    \begin{equation}
        D(s,\mc S_{I}) \doteq \min_{j\in \mc I_d(s)} D(s,s_{i_j}) = d_I.
    \end{equation}
    We will next show that 
    \begin{equation}
        \label{eq:Qk(s,a-)_upper_bound}
        \begin{split}
            Q_k(s,a_{i+}^-)&\leq  \max\{\gamma v(k,d_I-1)B_I,\gamma v(k,d)B\},\;\forall a_{i+}^-\in \mc A_I^+(s)\cup \mc A^-(s),\\
            Q_k(s,a_{i-}^-)&\leq \max\{\gamma v(k,d_I)B_I,\gamma v(k,d)B\},\;\forall a_{i-}^-\in \mc A\backslash(\mc A^+_I(s)\cup\mc A^+(s)).
        \end{split}
    \end{equation}
    {\bf When $d_I=1$}, we know that $s_{a_{i+}^-}\in \mc S_I$, hence $r(s,a_{i+}^-,s_{a_{i+}^-}) = B_I$. Moreover, when $k<d_I+h$, we have $V_k(s_{a_{i+}^-}) = 0$, this is because $V_0(s_{a_{i+}^-})$ is initialized as $0$ and it takes at least $h$ update of synchronous value iteration for $V_k(s_{a_{i+}^-})$ to be positive but $k<d_I+h-1 = h$. Hence, we know that 
    \begin{equation}
        Q_k(s,a_{i+}^-) = B_I = \gamma v(k,0)B_I\leq \max\{\gamma v(k,0)B_I,\gamma v(k,d)B\}.
    \end{equation}
    As for $Q_k(s,a_{i-}^-)$, we have $r(s,a_{i-}^-,s_{a_{i-}^-}) = 0$ and $V_k(s_{a_{i-}^-})=\max\{v(k,d^\prime_I)B_I,v(k,d^\prime)B\}$, where
    \begin{equation}
        d^\prime_I = D(s_{a_{i-}^-},\mc S_I) \geq d_I,\;d^\prime = D(s_{a_{i-}^-},\mc S_{T}) \geq d.
    \end{equation}
    Hence, we have 
    \begin{equation}
        \begin{split}
            &Q_k(s,a_{i-}^-) = r(s,a_{i-}^-,s_{a_{i-}^-})+\gamma V_{k}(s_{a_{i-}^-})\\
            \overset{(i)}{=}&\gamma\max\{v(k,d_I^\prime)B_I,v(k,d^\prime)B\}\leq\max\{\gamma v(k,1)B_I,\gamma v(k,d)B\},
        \end{split}
    \end{equation}
    where equality $(i)$ holds by Lemma \ref{lemma:Vk_Bound_directly_reachable_k<dI+h}. Therefore, we conclude that \eqref{eq:Qk(s,a-)_upper_bound} holds when $d_I = 1$. 
    
    {\bf When $d_I>1$}, we have $r(s,a_{i+}^-,s_{a_{i+}^-}) = 0$ and $r(s,a_{i-}^-,s_{a_{i-}^-}) = 0$. In this case, let $d^\prime = D(s_{a_{i-}^-},\mc S_{T})$, then
    \begin{equation}
        \begin{split}
            &Q_k(s,a_{i+}^-) = r(s,a_{i+}^-,s_{a_{i+}^-})+\gamma V_k(s_{a_{i+}^-})\\
            \overset{(i)}{=} &\gamma \max\{v(k,d_I-1)B_I,v(k,d^\prime)B\}\leq\max\{\gamma v(k,d_I-1)B_I,v(k,d)B\},
        \end{split}
    \end{equation}
    where equality $(i)$ holds by Lemma \ref{lemma:Vk_Bound_directly_reachable_k<dI+h}. For $Q_k(s,a_{i-}^-)$, we have $r(s,a_{i-}^-,s_{a_{i-}^-}) = 0$ and $V_k(s_{a_{i-}^-})=\max\{v(k,d^{\prime}_I)B_I,v(k,d^{\prime\prime})B\}$, where
    \begin{equation}
        d^\prime_I = D(s_{a_{i-}^-},\mc S_I) \geq d_I,\;d^{\prime\prime} = D(s_{a_{i-}^-},\mc S_{T}) \geq d.
    \end{equation}
    Hence, we have 
    \begin{equation}
        \begin{split}
            &Q_k(s,a_{i-}^-) = r(s,a_{i-}^-,s_{a_{i-}^-})+\gamma V_{k}(s_{a_{i-}^-})\\
            \overset{(i)}{=}&\gamma\max\{v(k,d_I^\prime)B_I,v(k,d^{\prime\prime})B\}\leq\max\{\gamma v(k,d_I)B_I,\gamma v(k,d)B\},
        \end{split}
    \end{equation}
    where equality $(i)$ holds by Lemma \ref{lemma:Vk_Bound_directly_reachable_k<dI+h}. Therefore, we conclude that \eqref{eq:Qk(s,a-)_upper_bound} holds when $d_I > 1$. 
    
    A direct implication of \eqref{eq:Qk(s,a-)_upper_bound} is that
    \begin{equation}
        \begin{split}
            &Q_k(s,a^-)<\max\{Q_k(s,a_{i+}^-),Q_k(s,a_{i-}^-)\}\\
            \leq&\max\{\max\{\gamma v(k,d_I-1)B_I,\gamma v(k,d)B\},\max\{\gamma v(k,d_I)B_I,\gamma v(k,d)B\}\}\\
            = &\max\{\gamma v(k,d_I-1)B_I,\gamma v(k,d)B\}.
        \end{split}
    \end{equation}
\end{itemize}
Therefore, we have shown \eqref{eq:Qk(s,a+)_value_Q_k(s,a-)_upper_bound} for $d\leq k<d_I+h-1$. Hence, combine \eqref{eq:Qk(s,a+)_value_Q_k(s,a-)_upper_bound} and \eqref{eq:v(k,d_I)B_I<v(k,d)B}, we have
\begin{equation}
    \begin{split}
        Q_k(s,a^-)\overset{(i)}{\leq} & \max\{\gamma v(k,d_I-1)B_I,\gamma v(k,d)B\} \leq \max\{\gamma^{d_I-1}B_I,\gamma^dB\}\\
        \leq&\max\{\gamma^{d_I-1}B_I,\gamma^{d-1}B\} \overset{(ii)}{=} \gamma^{d-1}B = v(k,d) = Q_k(s,a^+),
    \end{split}
\end{equation}
where inequality $(i)$ follows \eqref{eq:Qk(s,a+)_value_Q_k(s,a-)_upper_bound} and equality $(ii)$ holds due to \eqref{eq:v(k,d_I)B_I<v(k,d)B}. Hence, know that a greedy action will select $a^+\in \mc A^+(s)$. If we let $d^\prime,d_I^\prime$ denote the distance between $s_{a^+}$ to $\mc S_{T}$ and $\mc S_{I}$, respectively, we have
\begin{equation}
    d^\prime = D(s_{a^+}, \mc S_{T}) = d-1,\;d^\prime_I = D(s_{a^+}, \mc S_{I})\geq d_I-1,
\end{equation}
which implies $d^\prime$ and $d^\prime_I$ still satisfy \eqref{eq:d_cond_find_ST}. Recursively applying the above argument, we conclude that a greedy policy will find the shortest path to $\mc S_{T}$, when $d\leq k<d_I+h-1$.

\paragraph{When $k\geq d_I+h-1$.} To complete the theorem, we need to show when $k\geq d_I+h-1$, $Q_k(s,a^+)>Q_k(s,a^-)$ still holds. We have already shown that 
\begin{equation*}
    Q_k(s,a^+) = v(k,d)B = \gamma^{d-1}B,
\end{equation*}
when $d\leq k<d_I+h-1$, and the Q-function from the synchronous value iteration \eqref{eq:SyncValueIter} update is non-decreasing, hence when $k\geq d_I+h-1$, we know that $Q_k(s,a^+)\geq v(k,d) = \gamma^{d-1}B$ and it suffices to show that 
$\gamma^{d-1}B\geq Q_k(s,a^-)$. Let $\pi$ be a deterministic policy and suppose the agent starting from $s$ will visit the following sequences of intermediate states before reaching $\mc S_{T}$ under policy $\pi$:
\begin{equation}
    \label{eq:IS_sequence}
   \{s_{i_{j^\star_0}},s_{i_{j^\star_1}},\dots,s_{i_{j^\star_{u^\pi}}},s_{i_{j^\star_{u^\pi+1}}}\},
\end{equation}
where we slightly abuse the notation by assuming $s_{i_{j^\star_{u^\pi+1}}}\in \mc S_{T}$. By the synchronous value iteration update \eqref{eq:SyncValueIter}, the Q-function of state action pair $(s,a^-)$ under policy $\pi$ satisfies
\begin{equation}
    \label{eq:discounted_Q_func}
    Q^\pi_k(s,a^-) \leq \underbrace{\gamma^{d_I-1}B_I + \brac{ \sum_{l=1}^{u^\pi-1}\gamma^{d_l-1} }B_I+\gamma^{d_{u^\pi}-1}B}_{\Gamma(\pi)}, 
\end{equation}
where 
\begin{equation}
    \label{eq:dist_bound_for_Q}
    \begin{split}
        d_l &= D(s,s_{i_{j^\star_0}})+\sum_{m=0}^{l-1}D(s_{i_{j^\star_m}},s_{i_{j^\star_{m+1}}})\overset{(i)}{\geq}d_I+lh,\;\forall l = 1,2,\dots, u^\pi+1,
    \end{split}
\end{equation}
where inequality $(i)$ holds because of Assumption \ref{assump:min_dist_IS}. Hence, we know that
\begin{equation}
    \Gamma(\pi) = \gamma^{d_I-1}B_I + \brac{ \sum_{l=1}^{u^\pi-1}\gamma^{d_{l}-1} }B_I+\gamma^{d_{u^\pi}-1}B\overset{(ii)}{\leq}\underbrace{\gamma^{d_I-1}\brac{\sum_{l=0}^{u^\pi-1}\gamma^{l h}}B_I+\gamma^{d_I+u^\pi h-1}B}_{f(u^\pi)},
\end{equation}
where inequality $(ii)$ holds because of \eqref{eq:dist_bound_for_Q} and $\gamma<1$. By the synchronous value iteration update \eqref{eq:SyncValueIter}, the Q-function of state action pair $(s,a^+)$ when $k\geq d_I+h-1$ satisfies $Q_k(s,a^+)\geq \gamma^{d-1}B$, and the Q-function of state action pair $(s,a^-)$ satisfies
\begin{equation}
    Q_k(s,a^-) = \max_\pi Q_k^\pi(s,a^-)\leq \max_\pi f(u^\pi).
\end{equation}
Next we will show $\gamma^{d-1}B>f(u^\pi)$ holds $\forall \pi$ (which directly leads to $Q_k(s,a^+)>Q_k(s,a^-)$) in these following two conditions: 1) $d<d_I+\log_{\frac{1}{\gamma}}\brac{(1-\gamma^h)\frac{B}{B_I}}$ when $\frac{B}{B_I}\in \paren{0,\frac{1}{1-\gamma^h}}$; 2) $ d<d_I+\log_{\frac{1}{\gamma}}\paren{\frac{B}{B_I+\gamma^hB}}$ when $\frac{B}{B_I}\geq \frac{1}{1-\gamma^h}$.

\begin{itemize}
    \item {\bf When $d<d_I+\log_{\frac{1}{\gamma}}\brac{(1-\gamma^h)\frac{B}{B_I}}$ and $\frac{B}{B_I}\in \paren{0,\frac{1}{1-\gamma^h}}$.}
    
    Since $\frac{B}{B_I}\in \paren{0,\frac{1}{1-\gamma^h}}$, we know that $B_I-(1-\gamma^h)B>0$, hence
    \begin{equation}
    \begin{split}
        f(u^\pi) & < f(u^\pi) + \gamma^{d_I+u^\pi h-1}\brac{B_I-(1-\gamma^h)B} \\
        & = \gamma^{d_I-1}\brac{\sum_{l=0}^{u^\pi-1}\gamma^{l h}}B_I+\gamma^{d_I+u^\pi h-1}B+\gamma^{d_I+u^\pi h-1}\brac{B_I-(1-\gamma^h)B}\\
        & = \gamma^{d_I-1}\brac{\sum_{l=0}^{u^\pi}\gamma^{l h}}B_I+\gamma^{d_I+(u^\pi+1)h-1}B = f(u^\pi+1) <\dots < f(\infty) \\
        & = \gamma^{d_I-1}\brac{\sum_{l=0}^\infty\gamma^{l h}}B_I = \frac{\gamma^{d_I-1}}{1-\gamma^h}B_I\overset{(i)}{<}\gamma^{d-1}B,
    \end{split}
\end{equation}
where the last inequality holds because $d<d_I+\log_{\frac{1}{\gamma}}\brac{(1-\gamma^h)\frac{B}{B_I}}$.
    \item {\bf When $ d<d_I+\log_{\frac{1}{\gamma}}\paren{\frac{B}{B_I+\gamma^hB}}$ and $\frac{B}{B_I}\geq \frac{1}{1-\gamma^h}$.}
    
    When $\frac{B}{B_I}>\frac{1}{1-\gamma^h}$, we know that $(1-\gamma^h)B-B_I\geq0$, hence
\begin{equation}
    \begin{split}
        f(u^\pi) & < f(u^\pi) + \gamma^{d_I+(u^\pi-1)h-1}\brac{(1-\gamma^h)B-B_I}\\
        & = \gamma^{d_I-1}\brac{\sum_{l=0}^{u^\pi-1}\gamma^{l h}}B_I+\gamma^{d_I+u^\pi h-1}B+\gamma^{d_I+(u^\pi-1)h-1} \brac{(1-\gamma^h)B-B_I}\\
        &=\gamma^{d_I-1}\brac{\sum_{l=0}^{u^\pi-2}\gamma^{l h}}B_I+\gamma^{d_I+(u^\pi-1)h-1}B=f(u^\pi-1)<\dots<f(1)\\
        &=\gamma^{d_I-1}B_I+\gamma^{d_I+h-1}B\overset{(i)}{\leq}\gamma^{d-1} B. 
    \end{split}
\end{equation}
where the last inequality holds because $ d<d_I+\log_{\frac{1}{\gamma}}\paren{\frac{B}{B_I+\gamma^hB}}$.
\end{itemize}
Hence, we conclude that $\gamma^{d-1}B>f(u^\pi),\forall \pi$. This result implies that a greedy action will select $a^+\in \mc A^+(s)$. If we let $d^\prime,d_I^\prime$ denote the distance between $s_{a^+}$ to $\mc S_{T}$ and $\mc S_{I}$, respectively, we will have
\begin{equation}
    d^\prime = D(s_{a^+}, \mc S_{T}) = d-1,\;d^\prime_I = D(s_{a^+}, \mc S_{I})\geq d_I-1,
\end{equation}
which indicates that $d^\prime$ and $d^\prime_I$ still satisfy \eqref{eq:d_cond_find_ST}. Recursively applying the above argument, we conclude that a greedy policy will find the shortest path to $\mc S_{T}$, when $k\geq d_I+h-1$. Combine this result with the case where $d\leq k<d_I+h-1$, we have shown that the greedy policy finds the shortest path to $\mc S_{T}$, after $k(k>d)$ synchronous value iteration update.
\end{proof}

\section{Auxiliary Lemmas}
\begin{lemma}[$V_k(s)$ with Sparse Rewards]
\label{lemma:Vk_bounds_induction_SR}
Let $\mc M = (\mc S,\mc A, P, r, \gamma)$ be a deterministic MDP with desired terminal states $\mc S_{T}$ and initial state $s_0$. If the reward function $r(\cdot)$ satisfies the sparse reward setting (Table \ref{tab:reward_setting}) and the Q-function is zero-initialized \eqref{eq:zero_init}, $\forall d \in \bb N^+$, let $\mc S^{d}$ denote the set of states whose distance to the terminal state $\mc S_{T}$ is $d$:
\begin{equation*}
    \mc S^{d} = \{s\in\mc S|D(s,\mc S_{T})=d\}.
\end{equation*}
Then $\forall k\in \bb N^+,\forall s^d\in \mc S^d$, $V_k(s^d)$ satisfies:
\begin{equation}
    \label{eq:Value_it_induction_SR}
    V_k(s^d) = \begin{cases}
        \gamma^{d-1}B, & \forall k,d\in \bb N^+, d\leq k,\\
        0, & \text{ otherwise}.\\    
        \end{cases}
\end{equation}
\end{lemma}
\begin{proof}
We will use induction to prove the results. Recall the synchronous value iteration update \eqref{eq:SyncValueIter}:
\begin{equation*}
        V_{k+1}(s)=\max_{a\in \mc A}\left\{r(s,a,s_a)+\gamma V_k(s_a)\right\},\;\forall s\in \mc S.
\end{equation*}
We first check the induction condition \eqref{eq:Value_it_induction_SR} for the initial case $k=1$. 

\paragraph{Initial Condition for $k=1$.}
Since the agent only receives reward $r(s,a,s_a) = B$ when $s_a \in\mc S_{T}$, by the value iteration update, we have:
\begin{equation}
    V_{1}(s) = 
    \begin{cases}
        B, & \text{ if } s \in \mc S^1,\\
        0, &\text{ otherwise},
    \end{cases}
\end{equation}
hence, the initial condition is verified. 

\paragraph{Induction.}
Next, suppose the condition \eqref{eq:Value_it_induction_SR} holds for $1,2,\dots,k$, we will show it also holds for $k+1$. In this case, we only need to verify $V_{k+1}(s^{d})=\gamma^{k}B$ for $d = k+1$, because the induction assumption already implies
\begin{equation*}
    V_{k+1}(s^d) = \begin{cases}
        \gamma^{d-1}B, & \forall k,d\in \bb N^+, d\leq k,\\
        0, & \text{ otherwise}.\\    
        \end{cases}
\end{equation*}
Note that when $d = k+1>1$, we have 
\begin{equation}
    r(s^{k+1},a,s^{k+1}_a) = 0,\;\forall a\in\mc A,\; \text{ and }\;V_{k}(s^{k+1}_a) \overset{(i)}{\leq} V_k(s^k) = \gamma^{k-1} B,
\end{equation}
where $s^{k+1}_{a}$ is the subsequent state of state-action pair $(s^{k+1},a)$, and inequality $(i)$ becomes equality if only if $a$ is a correct action $a\in \mc A^+(s)$ (see Definition \ref{def:correct_actions}). Hence, we know that 
\begin{equation}
    V_{k+1}(s^{k+1})=\max_{a\in \mc A}\left\{r(s^{k+1},a,s_a)+\gamma V_k(s^{k+1}_a)\right\}=\gamma^{k}B.
\end{equation}
As a result, we conclude that 
\begin{equation*}
    V_{k+1}(s^d) = \begin{cases}
        \gamma^{d-1}B, & \forall k,d\in \bb N^+, d\leq k+1,\\
        0, & \text{ otherwise},\\    
        \end{cases}
\end{equation*}
which completes the proof.
\end{proof}

\begin{lemma}[$V_k(s)$ with Intermediate Rewards Setting (a)]
\label{lemma:Vk_bounds_induction_IR_a}
Let $\mc M = (\mc S,\mc A, P, r, \gamma)$ be a deterministic MDP with initial state $s_0$, intermediate states $\mc S_{I} = \{s_{i_1},s_{i_2},\dots,s_{i_N}\}$, and terminal states $\mc S_{T}$. Suppose $\mc M$ satisfies Assumption \ref{assump:OneWayIntermediateStates}, and \ref{assump:diff_settings_IS} (a), if the reward function $r(\cdot)$ follows the intermediate reward setting (Table \ref{tab:reward_setting}) and the Q-function is zero-initialized \eqref{eq:zero_init}, $\forall j \in [N+1]$, let $\mc S_{[j-1,j)}$ denote the set of states such that
\begin{equation}
    \mc S_{[j-1,j)} = 
    \begin{cases}
        \{s\in \mc S|D(s,s_{i_{1}})<\infty\}& \text{ if } j = 1,\\
        \{s\in \mc S|D(s,s_{i_{j}})=\infty,D(s,s_{i_{j+1}})<\infty\}& \text{ if } j = 2,3,\dots,N,\\
        \{s\in \mc S|D(s,s_{i_{N}})=\infty,D(s,\mc S_{T})<\infty\}   & \text{ if } j = N+1.
    \end{cases}
\end{equation}
Given a state $s^{\mb d}\in \mc S^{\mb d}_{[j-1,j)}$, where $\mb d = [d_j,d_{j+1},\dots,d_N,d_{N+1}]^\top\in \bb R^{n-j+2}$ is a vector, such that $d_j<d_{j+1}<\dots<d_n<d_{N+1}$ denote the distance from $s^{\mb d}$ to $s_{i_j},s_{i_{j+1}},\dots,s_{i_N},\mc S_{T}$, respectively:
\begin{equation}
        d_l = D(s^{\mb d},s_{i_l}),\;\forall l=j,j+1,\dots,n, \;\text{ and }\;d_{N+1} = D(s^{\mb d},\mc S_{T}).
\end{equation}
Then $\forall k\in \bb N^+,\forall s^{\mb d}\in \mc S^{\mb d}$, $V_k(s^{\mb d})$ satisfies the following conditions:
\begin{equation}
    \label{eq:Value_it_induction_IR_a}
    V_k(s^{\mb d}) = \sum_{l=j}^{N+1} v(k,d_l)B_l,
\end{equation}
where
\begin{equation}
    \label{eq:Value_it_induction_IR_a_v(k,d)}
    v(k,d_l)=
    \begin{cases}
        \gamma^{d_l-1}, & \forall k,d\in \bb N^+, d_l\leq k,\forall l\in [N+1],\\
        0, & \text{ otherwise},\\    
    \end{cases}
    \;\text{ and }\;
    B_l = \begin{cases}
        B_I,&\text{ if }\; l\in [N],\\
        B,&\text{ if }\; l = N+1.
    \end{cases}
\end{equation}
\end{lemma}
\begin{proof}
We will use induction to prove the results. Recall the synchronous value iteration update \eqref{eq:SyncValueIter}:
\begin{equation*}
    \begin{split}
        V_{k+1}(s)&=\max_{a\in \mc A}\left\{r(s,a,s_a)+\gamma V_k(s_a)\right\},\;\forall s\in \mc S.\\
    \end{split}
\end{equation*}
We first check the induction condition \eqref{eq:Value_it_induction_IR_a} for the initial case $k=1$.

\paragraph{Initial Condition for $k=1$.} 
When $k=1$, the value function of states $s$ {\em before} reaching $s_{i_j}$ will not be affected by the rewards from $s_{i_{j+1}},s_{i_{j+2}},\dots,s_{i_{N}},\mc S_{T}$ and will only be updated using the intermediate reward from $s_{i_j}$. As similarly proven for the sparse reward case (Lemma \ref{lemma:Vk_bounds_induction_SR}), we have
\begin{equation}
    V_{1}(s) = 
    \begin{cases}
        B_{I_l}, & \text{ if }\;D(s, s_{i_l}) = 1,\forall l\in [N+1],\\
        0, &\text{ otherwise},
    \end{cases}
\end{equation}
which implies $V_1(s^{\mb d}) = \sum_{l=j}^{N+1} v(1,d_l)B_l$,
hence, the initial condition is verified.

\paragraph{Induction.}
Next, suppose the conditions \eqref{eq:Value_it_induction_IR_a} hold for $1,2,\dots,k$, we will show they also hold for $k+1$. By the induction assumption, we know that 
\begin{equation*}
    V_k(s^{\mb d}) = \sum_{l=j}^{N+1} v(k,d_l)B_l,
\end{equation*}
holds for $1,2,\dots,k$, therefore we only need to show the induction condition \eqref{eq:Value_it_induction_IR_a} holds when it exists $l^\prime\in \{j,j+1,\dots,N+1\}$, such that $d_{l^\prime}=k+1$, because $v(k,d_{l^\prime})$ remains the same, $\forall d_{l^\prime}\neq k+1$ when we change $k$ to $k+1$. Suppose $s^{\mb d}\in \mc S_{[j-1,j)}$ satisfies $D(s^{\mb d},s_{i_{l^\prime}}) = k+1$, since $d_{j}<d_{j+1}<\dots<d_{l^\prime} = k+1$, we will discuss the following cases.
\begin{itemize}
    \item {\bf When $j < N+1$ and $l^\prime < N+1$.} For a correct action $a^+\in \mc A^+(s^{\mb d})$, we have
    \begin{equation}
        \begin{split}
        &r(s^{\mb d},a^+,s^{\mb d}_{a^+})+\gamma V_k(s^{\mb d}_{a^+})\\
        =& 
        \begin{cases}
            B_I+\gamma V_k(s^{\mb d}_{a^+})=B_I+\gamma\sum_{l=j+1}^{l^\prime-1}\gamma^{d_l-2}B_I+\gamma\cdot\gamma^{k-1}B,&\text{ if }\;d_j=1,\\
            \gamma V_k(s_{a^+}^d) = \gamma \sum_{l=j}^{l^\prime -1}\gamma^{d_l-2}B_I+\gamma\cdot \gamma^{k-1}B,&\text{ if }\;d_j>1.
        \end{cases}\\
        =&\sum_{l=j}^{l^\prime-1}\gamma^{d_l-1}B_I+\gamma^k B_I= \sum_{l=j}^{l^\prime}v(k+1,d_l)B_I = \sum_{l=j}^{N+1}v(k+1,d_l)B_l,  
    \end{split}
    \end{equation}
    where the last equality holds because of \eqref{eq:Value_it_induction_IR_a_v(k,d)}. For an incorrect action $a^{-}\in \mc A^-(s^{\mb d})$ we have
    \begin{equation}
        \begin{split}
            r(s^{\mb d},a^-,s^{\mb d}_{a^-})+\gamma V_k(s^{\mb d}_{a^-}) \overset{(i)}{=} \gamma V_k(s^{\mb d}_{a^-})\overset{(ii)}{\leq} \sum_{l=j}^{l^\prime-1}\gamma^{d_l}B_I,    
        \end{split}
    \end{equation}
    where equality $(i)$ holds because $r(s^{\mb d},a^-,s_{a-}^{\mb d}) = 0$ and inequality $(ii)$ holds because of $\gamma<1$. Since $k+1>1$, we know that $r(s^{\mb d},a,s^{\mb d}_a) = 0$ and hence
    \begin{equation}
       r(s^{\mb d},a^-,s^{\mb d}_{a^-})+\gamma V_k(s^{\mb d}_{a^-}) < r(s^{\mb d},a^+,s^{\mb d}_{a^+})+\gamma V_k(s^{\mb d}_{a^+}).
    \end{equation}
    By the value iteration update \eqref{eq:SyncValueIter}, we know that 
    \begin{equation}
        \begin{split}
            V_{k+1}(s)&=\max_{a\in \mc A}\left\{r(s,a,s_a)+\gamma V_k(s_a)\right\}\\
            &=r(s^{\mb d},a^+,s^{\mb d}_{a^+})+\gamma V_k(s^{\mb d}_{a^+}) = \sum_{l=j}^{N+1}v(k+1,d_l)B_l,
        \end{split}
    \end{equation}
    which completes the induction for this case.
    \item {\bf When $j < N+1$ and $l^\prime = N+1$.} For a correct action $a^+\in \mc A^+(s^{\mb d})$, we have
    \begin{equation}
        \begin{split}
        &r(s^{\mb d},a^+,s^{\mb d}_{a^+})+\gamma V_k(s^{\mb d}_{a^+})\\
        =& 
        \begin{cases}
            B_I+\gamma V_k(s^{\mb d}_{a^+})=B_I+\gamma\sum_{l=j+1}^{N}\gamma^{d_l-2}B_I+\gamma\cdot\gamma^{k-1}B,&\text{ if }\;d_j=1,\\
            \gamma V_k(s_{a^+}^d) = \gamma \sum_{l=j}^{N}\gamma^{d_l-2}B_I+\gamma\cdot \gamma^{k-1}B,&\text{ if }\;d_j>1.
        \end{cases}\\
        =&\sum_{l=j}^{N}\gamma^{d_l-1}B_I+\gamma^k B= \sum_{l=j}^{N+1}v(k+1,d_l)B_l.  
    \end{split}
    \end{equation}
    For an incorrect action $a^{-}\in \mc A^-(s^{\mb d})$ we have
    \begin{equation}
        \begin{split}
            r(s^{\mb d},a^-,s^{\mb d}_{a^-})+\gamma V_k(s^{\mb d}_{a^-}) = \gamma V_k(s^{\mb d}_{a^-})\leq \sum_{l=j}^{N}\gamma^{d_l}B_I.    
        \end{split}
    \end{equation}
    Since $k+1>1$, we know that $r(s^{\mb d},a,s^{\mb d}_a) = 0$ and hence
    \begin{equation}
        r(s^{\mb d},a^-,s^{\mb d}_{a^-})+\gamma V_k(s^{\mb d}_{a^-}) < r(s^{\mb d},a^+,s^{\mb d}_{a^+})+\gamma V_k(s^{\mb d}_{a^+}).
    \end{equation}
    By the value iteration update \eqref{eq:SyncValueIter}, we know that 
    \begin{equation}
        \begin{split}
            V_{k+1}(s)&=\max_{a\in \mc A}\left\{r(s,a,s_a)+\gamma V_k(s_a)\right\}\\
            &=r(s^{\mb d},a^+,s^{\mb d}_{a^+})+\gamma V_k(s^{\mb d}_{a^+}) = \sum_{l=j}^{N+1}v(k+1,d_l)B_l,
        \end{split}
    \end{equation}
    which completes the induction for this case.
    \item {\bf When $j = N+1$ and $l^\prime = N+1$.} For a correct action $a^+\in \mc A^+(s^{\mb d})$, we have
    \begin{equation}
        \begin{split}
        &r(s^{\mb d},a^+,s^{\mb d}_{a^+})+\gamma V_k(s^{\mb d}_{a^+}) = \gamma^k B= v(k+1,d_l)B_l.  
    \end{split}
    \end{equation}
    For an incorrect action $a^{-}\in \mc A^-(s^{\mb d})$ we have
    \begin{equation}
        \begin{split}
            r(s^{\mb d},a^-,s^{\mb d}_{a^-})+\gamma V_k(s^{\mb d}_{a^-}) \overset{(i)}{=} 0,    
        \end{split}
    \end{equation}
    where equality $(i)$ holds because Definition \ref{def:correct_actions} implies that $D(s^{\mb d}_{a^-},\mc S_{T})\geq D(s^{\mb d},\mc S_{T})= k+1$. Since $k+1>1$, we know that $r(s^{\mb d},a,s^{\mb d}_a) = 0$ and hence
    \begin{equation}
        r(s^{\mb d},a^-,s^{\mb d}_{a^-})+\gamma V_k(s^{\mb d}_{a^-}) < r(s^{\mb d},a^+,s^{\mb d}_{a^+})+\gamma V_k(s^{\mb d}_{a^+}).
    \end{equation}
    By the value iteration update \eqref{eq:SyncValueIter}, we know that 
    \begin{equation}
        \begin{split}
            V_{k+1}(s)&=\max_{a\in \mc A}\left\{r(s,a,s_a)+\gamma V_k(s_a)\right\}\\
            &=r(s^{\mb d},a^+,s^{\mb d}_{a^+})+\gamma V_k(s^{\mb d}_{a^+}) = v(k+1,d_{N+1})B,
        \end{split}
    \end{equation} 
    which completes the induction for this case.
\end{itemize}
Hence, we know the induction condition \eqref{eq:Value_it_induction_IR_a} holds for $k+1$, which completes the proof.
\end{proof}

\begin{lemma}[$V_k(s)$ when $\mc S_{T}$ is Directly Reachable]
\label{lemma:Vk_Bound_directly_reachable_k<dI+h}
Let $\mc M = (\mc S,\mc A, P, r, \gamma)$ be a deterministic MDP with initial state $s_0$, intermediate states $\mc S_{I} = \{s_{i_1},s_{i_2},\dots,s_{i_N}\}$, and terminal states $\mc S_{T}$. Suppose $\mc M$ satisfies Assumption \ref{assump:OneWayIntermediateStates}, \ref{assump:diff_settings_IS} (b), and \ref{assump:min_dist_IS}, if the reward function $r(\cdot)$ follows the intermediate reward setting (Table \ref{tab:reward_setting}) and the value function and Q-function are zero-initialized \eqref{eq:zero_init}, $\forall s \in \mc S\backslash\mc S_{T}$, if $\mc S_{T}$ is directly reachable from $s$ and suppose
\begin{equation}
    d = D(s,\mc S_{T})\; \text{ and }\; d_I = D(s,\mc S_{I}) \doteq \min_{j\in \mc I_d(s)} D(s,s_{i_j}),
\end{equation}
then after $k\leq d_I+h$ synchronous value iteration updates \eqref{eq:SyncValueIter}, the value function $V_k(s)$ satisfies 
\begin{equation}
    \label{eq:Vk(s)_ST+_directly_reachable}
    V_k(s) = \max\{v(k,d_I)B_I,v(k,d)B\},
\end{equation}
where 
\begin{equation}
    \label{eq:v(k,d)_def}
    v(k,d)=
    \begin{cases}
        \gamma^{d-1}, & \forall k,d\in \bb N, d\leq k,\\
        0, & \text{ otherwise}.\\    
    \end{cases}
\end{equation}
\end{lemma}
\begin{proof}
Similar to the definition of $\mc A^+(s)$, let $\mc A_I^+(s)$ denotes the set of actions that lead $s$ one step closer to $\mc S_I$: 
\begin{equation}
    \label{eq:defAI+}
    \mc A^+_I(s) = \{a|a\in \mc A,D(s_a,\mc S_{I})=D(s,\mc S_{I})-1\}.
\end{equation}
We will use induction to prove the results. Recall the synchronous value iteration update \eqref{eq:SyncValueIter}:
\begin{equation*}
    \begin{split}
        V_{k+1}(s)&=\max_{a\in \mc A}\left\{r(s,a,s_a)+\gamma V_k(s_a)\right\},\;\forall s\in \mc S.\\
    \end{split}
\end{equation*}
We first check the induction condition \eqref{eq:Vk(s)_ST+_directly_reachable} for the initial case $k=1$.

\paragraph{Initial Condition for $k=1$.}
When $k=1$, we know that 
\begin{equation}
    \begin{split}
        V_1(s) &= \max_a\{r(s,a,s_a)+\gamma V_k(s_a)\}\\
        &=\begin{cases}
            \max\{B_I,B\},&\;\text{if }d_I=d=1,\\
            B_I,&\;\text{if }d_I=1,d>1,\\
            B,&\;\text{if }d=1,d_I>1,\\
            0,&\;\text{otherwise }
        \end{cases}\\
        &=\max\{v(1,d_I)B_I,v(1,d)B\},
    \end{split}
\end{equation}
which verifies the initial condition \eqref{eq:Vk(s)_ST+_directly_reachable}.

\paragraph{Induction.}
We will discuss the following cases: (a) $d_I\leq k+1,d=1$, (b) $d_I\leq k+1,d>1$, and (c) $d_I>k+1$.

\paragraph{(a) When $d_I\leq k+1,d=1$.}
In this case, we consider these following actions:
\begin{equation}
    a^+\in \mc A^+(s),\; a_{i+}^-\in \mc A_I^+(s)\backslash \mc A^+(s),\; a_{i-}^-\in \mc A(s)\backslash(A_I^+(s)\cup\mc A^+(s)).
\end{equation}
\begin{itemize}
    \item {\bf For any $a^+\in \mc A^+(s)$.} In this case, $d=1$ implies that $s_{a^+}\in \mc S_{T}$. Hence we have
    \begin{equation}
        \label{eq:dI=k+1_d=1_a+}
        r(s,a^+,s_{a^+})= B,\;V_k(s_{a^+}) = 0\implies r(s,a^+,s_{a^+})+\gamma V_k(s_{a^+}) = B,
    \end{equation}
    
    \item {\bf For any $a^-_{i+}\in \mc A_I^+(s)\backslash\mc A^+(s)$.} In this case, let 
    \begin{equation}
        d^\prime = D(s_{a_{i+}^-},\mc S_{T}),
    \end{equation}
    by the definition of $\mc A_I^+(s)$ \eqref{eq:defAI+}, we know that 
    \begin{equation}
        d^\prime = D(s_{a^-_{i+}},\mc S_{T})\geq d = 1,\; D(s_{a^-_{i+}},\mc S_{I}) = D(s,\mc S_{I}) - 1 = d_I-1 \leq k.
    \end{equation}
    Hence, the induction condition \eqref{eq:Vk(s)_ST+_directly_reachable} implies that 
    \begin{equation}
        V_k(s_{a_{i+}^-}) = \max\{v(k,d_I-1)B_I,v(k,d^\prime)B\}.
    \end{equation}
    If $d_I >1$, we know that $r(s,a,s_{a_{i+}^-})=0$, and hence
    \begin{equation}
        \label{eq:dI=k+1_d=1_a-i+;dI>1}
        \begin{split}
            &r(s,a_{i+}^-,s_{a_{i+}^-})+\gamma V_k(s_{a^-_{i+}})\\
            =& \gamma\max\{v(k,d_I-1)B_I,v(k,d^\prime)B\} \overset{(i)}{=} \max\{v(k+1,d_I)B_I,\gamma v(k,d^\prime)B\},
        \end{split}
    \end{equation}
    where equality $(i)$ holds because $\gamma v(k,d_I - 1) = \gamma v(k,k) = \gamma^k = v(k+1,k+1) = v(k+1,d_I)$. When $d_I=1$, we have
    \begin{equation}
        \label{eq:dI=k+1_d=1_a-i+;dI=1}
        r(s,a_{i+}^-,s_{a_{i+}^-})+\gamma V_k(s_{a^-_{i+}}) = B_I.
    \end{equation}
    
    \item {\bf For any $a^-_{i-}\in \mc A\backslash(\mc A_I^+(s)\cup\mc A^+(s))$.} In this case, let 
    \begin{equation}
        d^{\prime\prime} = D(s_{a_{i+}^-},\mc S_{T})\geq d = 1,\;d_I^{\prime} = D(s_{a_{i+}^-},\mc S_{I})\geq d_I=k+1,
    \end{equation}
    the induction condition \eqref{eq:Vk(s)_ST+_directly_reachable} implies that 
    \begin{equation}
        V_k(s_{a_{i+}^-}) = \max\{v(k,d_I^\prime)B_I,v(k,d^{\prime\prime})B\}.
    \end{equation}
    Similar to the previous case for $a_{i+}^-$, we also have $r(s,a,s_{a_{i+}^-})=0$, and hence
    \begin{equation}
        \label{eq:dI=k+1_d=1_a-i-}
        r(s,a_{i-}^-,s_{a_{i-}^-})+\gamma V_k(s_{a^-_{i-}}) = \max\{\gamma v(k,d_I^\prime)B_I,\gamma v(k,d^{\prime\prime})B\}.
    \end{equation}
\end{itemize}
Combine \eqref{eq:dI=k+1_d=1_a+}, \eqref{eq:dI=k+1_d=1_a-i+;dI>1}, \eqref{eq:dI=k+1_d=1_a-i-}, and the synchronous value iteration update \eqref{eq:SyncValueIter}, when $d_I>1$ we have
\begin{equation}
    \begin{split}
        V_{k+1}(s) &= \max_{a\in \mc A}\{r(s,a,s_a)+\gamma V_{k}(s_a)\}\\
        &=\max\left\{B,\max\{v(k+1,d_I)B_I,\gamma v(k,d^\prime)B\},\max\{\gamma v(k,d_I^\prime)B_I,\gamma v(k,d^{\prime\prime})B\}\right\}\\
        &\overset{(i)}{=}\max\{v(k+1,d_I)B_I,B\} = \max\{v(k+1,d_I)B_I,v(k+1,1)B\},
    \end{split}
\end{equation}
where equality $(i)$ holds because $d^\prime,d^{\prime\prime}\geq1$ implies that $\gamma v(k,d^\prime)B,\gamma v(k,d^{\prime\prime})B\leq \gamma B<B$.

Similarly, using \eqref{eq:dI=k+1_d=1_a+}, \eqref{eq:dI=k+1_d=1_a-i+;dI=1}, \eqref{eq:dI=k+1_d=1_a-i-} when $d_I=1$, we have
\begin{equation}
    \begin{split}
        V_{k+1}(s) &= \max_{a\in \mc A}\{r(s,a,s_a)+\gamma V_{k}(s_a)\}\\
        &=\max\left\{B,B_I,\max\{\gamma v(k,d_I^\prime)B_I,\gamma v(k,d^{\prime\prime})B\}\right\}\\
        &\overset{(ii)}{=}\max\{B_I,B\} = \max\{v(k+1,1)B_I,v(k+1,1)B\},
    \end{split}
\end{equation}
where equality $(ii)$ holds because $d_I^\prime,d^{\prime\prime}>1$. 
Hence, we have verified the induction condition \eqref{eq:Vk(s)_ST+_directly_reachable} for $d_I\leq k+1,d=1$. 

\paragraph{(b) When $d_I\leq k+1,d>1$.}

In this case, we also consider these following actions:
\begin{equation}
    a^+\in \mc A^+(s),\; a_{i+}^-\in \mc A_I^+(s)\backslash \mc A^+(s),\; a_{i-}^-\in \mc A(s)\backslash(A_I^+(s)\cup\mc A^+(s)).
\end{equation}
\begin{itemize}
    \item {\bf For any $a^+\in \mc A^+(s)$.} Similarly, let 
    \begin{equation}
        d_I^\prime = D(s_{a^+},\mc S_{T}),
    \end{equation}
    hence we know that 
    \begin{equation}
        D(s_{a^+},\mc S_{T})= d - 1\geq1,\; d_I^\prime = D(s_{a^+},\mc S_{I}) \geq D(s,\mc S_{I}) - 1 \leq k.
    \end{equation}
    The induction condition \eqref{eq:Vk(s)_ST+_directly_reachable} implies that 
    \begin{equation}
        V_k(s_{a^+}) = \max\{v(k,d_I^\prime)B_I,v(k,d-1)B\}.
    \end{equation}
    Notice that $d>1$ implies that $r(s,a^+,s_{a^+})=0$, hence
    \begin{equation}
        \label{eq:dI=k+1_d>1_a+}
        r(s,a^+,s_{a^+})+\gamma V_k(s_{a^+}) = \max\{\gamma v(k,d_I^\prime)B_I, \gamma v(k,d-1)B\}.
    \end{equation}
    
    \item {\bf For any $a^-_{i+}\in \mc A_I^+(s)\backslash\mc A^+(s)$.} In this case, let 
    \begin{equation}
        d^\prime = D(s_{a_{i+}^-},\mc S_{T}),
    \end{equation}
    by the definition of $\mc A_I^+(s)$ \eqref{eq:defAI+}, we know that 
    \begin{equation}
        d^\prime = D(s_{a^-_{i+}},\mc S_{T})\geq d > 1,\; D(s_{a^-_{i+}},\mc S_{I}) = D(s,\mc S_{I}) - 1 = k.
    \end{equation}
    Hence, the induction condition \eqref{eq:Vk(s)_ST+_directly_reachable} implies that 
    \begin{equation}
        V_k(s_{a_{i+}^-}) = \max\{v(k,d_I-1)B_I,v(k,d^\prime)B\}.
    \end{equation}
    If $d_I >1$, so $r(s,a,s_{a_{i+}^-})=0$, and hence
    \begin{equation}
        \label{eq:dI=k+1_d>1_a-i+;dI>1}
        \begin{split}                &r(s,a_{i+}^-,s_{a_{i+}^-})+\gamma V_k(s_{a^-_{i+}})\\
        = &\gamma\max\{v(k,d_I-1)B_I,v(k,d^\prime)B\} = \max\{v(k+1,d_I)B_I,\gamma v(k,d^\prime)B\},    
        \end{split}
    \end{equation}
    the last equality holds for the same reason described in the previous case for $d_I>1,d>1$. When $d_I=1$, we have
    \begin{equation}
        \label{eq:dI=k+1_d>1_a-i+;dI=1}
        r(s,a_{i+}^-,s_{a_{i+}^-})+\gamma V_k(s_{a^-_{i+}}) = B_I.
    \end{equation}
    
    \item {\bf For any $a^-_{i-}\in \mc A\backslash(\mc A_I^+(s)\cup\mc A^+(s))$.} Similar to the case where $d_I=k+1,d=1$, let 
    \begin{equation}
        d^{\prime\prime} = D(s_{a_{i+}^-},\mc S_{T})\geq d > 1,\;d_I^{\prime} = D(s_{a_{i+}^-},\mc S_{I})\geq d_I=k+1,
    \end{equation}
    the induction condition \eqref{eq:Vk(s)_ST+_directly_reachable} implies that 
    \begin{equation}
        V_k(s_{a_{i+}^-}) = \max\{v(k,d_I^\prime)B_I,v(k,d^{\prime\prime})B\}.
    \end{equation}
    Similar to the previous case for $a_{i+}^-$, we also have $r(s,a,s_{a_{i+}^-})=0$, and hence
    \begin{equation}
        \label{eq:dI=k+1_d>1_a-i-}
        r(s,a_{i-}^-,s_{a_{i-}^-})+\gamma V_k(s_{a^-_{i-}}) = \max\{\gamma v(k,d_I^\prime)B_I,\gamma v(k,d^{\prime\prime})B\}.
    \end{equation}
\end{itemize}
Combine \eqref{eq:dI=k+1_d>1_a+}, \eqref{eq:dI=k+1_d>1_a-i+;dI>1}, \eqref{eq:dI=k+1_d>1_a-i-}, and the synchronous value iteration update \eqref{eq:SyncValueIter}, when $d_I>1$ we have
\begin{equation}
    \begin{split}
        &V_{k+1}(s)=\; \max_{a\in \mc A}\{r(s,a,s_a)+\gamma V_{k}(s_a)\}\\
        =&\;\max\{\max\{\gamma v(k,d_I^\prime)B_I, \gamma v(k,d-1)B\},\max\{v(k+1,d_I)B_I,\gamma v(k,d^\prime)B\},\\
        &\max\{\gamma v(k,d_I^{\prime\prime})B_I,\gamma v(k,d^{\prime\prime})B\}\}\\
        \overset{(i)}{=}&\;\max\{v(k+1,d_I)B_I,\gamma v(k,d-1)B\} \overset{(ii)}{=} \max\{v(k+1,d_I)B_I,v(k+1,d)B\},
    \end{split}
\end{equation}
where equality $(i)$ holds because $d^\prime_I,d^{\prime\prime}_I\geq d_I$ implies that $\gamma v(k,d_I^\prime),\gamma v(k,d_I^{\prime\prime})\leq \gamma^{d_I}=v(k+1,d_I)$
and $d^\prime,d^{\prime\prime}>d$ implies that $v(k,d^\prime),v(k,d^{\prime\prime}) \leq  v(k,d-1)$. Equality $(ii)$ holds because when $d_I = k+1$, $v(k+1,d_I) = \gamma^{d_I}>0$, then if $\gamma v(k,d-1)B>v(k+1,d_I)B_I$, we know that $\gamma v(k,d-1)B>0$, which implies $k\geq d-1$ (or equivalently $k+1\geq d$) hence $v(k+1,d) = \gamma^{d-1} = \gamma v(k,d-1)$. 

Similarly, when $d_I=1$, \eqref{eq:dI=k+1_d>1_a+}, \eqref{eq:dI=k+1_d>1_a-i+;dI=1}, \eqref{eq:dI=k+1_d>1_a-i-}, we have
\begin{equation}
    \begin{split}
        &V_{k+1}(s)=\; \max_{a\in \mc A}\{r(s,a,s_a)+\gamma V_{k}(s_a)\}\\
        =&\;\max\left\{\max\{\gamma v(k,d_I^\prime)B_I, \gamma v(k,d-1)B\},B_I,\max\{\gamma v(k,d_I^{\prime\prime})B_I,\gamma v(k,d^{\prime\prime})B\}\right\}\\
        \overset{(iii)}{=}&\;\max\{B_I,\gamma v(k,d-1)B\} = \max\{v(k+1,1)B_I,v(k+1,d)B\},
    \end{split}
\end{equation}
where equality $(iii)$ holds because $d^\prime_I,d_I^{\prime\prime}\geq 1,$
Hence, we have verified the induction condition \eqref{eq:Vk(s)_ST+_directly_reachable} for $d_I\leq k+1,d>1$.

\paragraph{(c) When $d_I > k+1$.} When $d_I>k+1$, we have $\forall a\in \mc A$, we have $d_I^\prime = D(s_a,\mc S_I)>k$. By the definition of $v(k,d)$ in \eqref{eq:v(k,d)_def}, we know $v(k,d_I^\prime) = 0$. By the induction assumption \eqref{eq:Vk(s)_ST+_directly_reachable}, we have 
\begin{equation}
    V_k(s_a) = \max\{v(k,d_I^\prime)B_I,v(k,d^\prime)B\} = v(k,d^\prime)B,
\end{equation}
where $d^\prime = D(s_a,\mc S_T)$. Hence, from the synchronous value iteration \eqref{eq:SyncValueIter}, we have
\begin{equation}
    V_{k+1}(s) = \max_{a\in \mc A}\{r(s,a,s_a)+ \gamma V_k(s_a)\} \overset{(i)}{=} v(k,d^\prime)B = \max\{v(k+1,d_I^\prime)B_I,v(k+1,d^\prime)B\},
\end{equation}
where equality $(i)$ holds by Lemma \ref{lemma:Vk_bounds_induction_SR}. Hence we conclude the induction condition \eqref{eq:Vk(s)_ST+_directly_reachable} holds when $d>k+1$.

\paragraph{Conclusion.} In conclusion, we have verified the induction condition \eqref{eq:Vk(s)_ST+_directly_reachable} for all three cases: (a) $d_I=k+1,d=1$, (b) $d_I=k+1,d>1$, (c) $d_I>k+1$. Hence, we conclude that the induction condition \eqref{eq:Vk(s)_ST+_directly_reachable} holds for all $k<d_I+h$.
\end{proof}

\section{Experimental Details}
\subsection{Experimental Details of Section \ref{subsec:MotivatingExamples}}
\label{Exp:MotivatingExpDetail}
To obtain the win rate in \ref{tab:experiments-1g}, for each experiment, we first run asynchronous Q learning for $m$ episodes, where we use the first $m$-100 episodes for training, and the win rate from the last 100 episodes for testing. We repeat the previous experiment 10 times report the average win rate of all 1000 games. In addition to the reward design specified Table \ref{tab:experiments-1g}, we use the following command: \verb|python pacman.py| under under the folder \verb|./cs188/cs188/p3_reinforcement|, with these configurations:
\begin{itemize}
    \item  \verb|-p PacmanQAgent|, set the training algorithm to standard asynchronous Q learning;
    \item \verb|-x| $m$ \verb|-n| $m+100$, set the training and testing episodes, replace $m$ with the number of episode in Table \ref{tab:experiments-1g};
    \item \verb|-l mediumGrid|, set the layout of the game, we use the \verb|mediumGrid| layout for our experiments;
    \item \verb|-g DirectionalGhost|, set the strategy of the ghost, we opt the \verb|DirectionalGhost|, where the ghost will find the shortest path to the Pacman;
    \item \verb|-k 1|, set the number of ghosts, the default number of ghosts is $2$ in the \verb|mediumGrid| layout.
\end{itemize}
Similar phenomenons reported in Table \ref{tab:experiments-1g} also appear in larger layout with more ghosts. Also note that in order to obtain a similar win rate, experiments with larger layouts/more ghosts generally require more training episodes.

\subsection{Experimental Details of Section \ref{sec:Experiment}}
\subsubsection{Asynchronous Q-Learning in Section \ref{sec:Experiment}}
All models for asynchronous q-learning experiments are trained with a learning rate of 0.1, a discount factor of 0.9, and use an 0.8 $\eps$-greedy exploration strategy during training. Each state is the entire environment encoded as a string available for MiniGrid \cite{gym_minigrid} environments.

\subsubsection{Deep RL Hyperparameters in Section \ref{sec:Experiment}}
\label{subsec:Hyperparameters}
Each training session for deep RL algorithms was run using a GeForce RTX 2080 GPU. Shared parameters are listed in Table \ref{tab:shared_parameters}, and parameters specific to each algorithm is provided in Table \ref{tab:specific_parameters}. For DQN, like asynchronous Q-learning, we use a 0.8 $\eps$-greedy exploration strategy.

\begin{table}[H]
    \centering
    \begin{tabular}{c|c}
           1 & Conv2D(inchannels = 3, outchannels = 16, stride = (2,2)) \\
          2 & ReLU \\ 
         3 &  MaxPool2D(2, 2) \\
          4 & Conv2D(inchannels = 16, outchannels = 32, stride = (2,2)) \\
         5 &  ReLU \\
          6 & Conv2D(inchannels = 32, outchannels = 64, stride = (2,2)) \\ 
          7 & ReLU \\ 
          8 & Linear(embedded size, 64) \\ 
          9 & Tanh \\ 
          10 & Linear(64, number of actions) \\
    \end{tabular}
    \caption{Network architecture}
    \label{tab:architecture}
\end{table}

\begin{table}[H]
    \centering
    \begin{tabular}{c|c}
    \toprule
        Parameters & Values\\
    \midrule
         Learning Rate & 0.001 \\ 
         Network Architecture & See Table  \ref{tab:architecture}\\
         Observability of Env & Fully Observable \\
         \bottomrule
    \end{tabular}
    
    \caption{Parameters}
    \label{tab:shared_parameters}
\end{table}

\begin{table}[H]
\centering
\begin{tabular}{c|c|c|c}
\toprule
\multicolumn{1}{c|}{Parameters} & \multicolumn{1}{|c|}{DQN} & \multicolumn{1}{|c|}{PPO} & \multicolumn{1}{|c}{A2C}\\
\toprule optimizer & RMSProp & Adam & RMSProp \\
\midrule Discount Factor ($\gamma$) & 0.90 (Maze, 3-Door) & 0.90 & 0.90 \\
   & 0.80 (4-Door) &   &   \\
\midrule batch size & 128 & 128 & N/A \\
\midrule buffer size & 100000 & N/A & N/A \\
\midrule target net. update interval (steps) & 100 & N/A & N/A \\
\midrule number of actors & N/A & 10 & 10 \\
\midrule steps per actor before update & N/A & 128 & 5 \\
\midrule entropy coeff. & N/A & 0.01 & 0.01 \\
\midrule value loss coeff. & N/A & 0.5 & 0.5 \\ 
\midrule GAE discount ($\lambda$) & N/A & 0.95 & 0.95 \\ 
\midrule max norm of gradient & N/A & 0.5 & 0.5 \\
\midrule clipping $\epsilon$ & N/A & 0.2 & N/A \\ 
\midrule PPO epochs per update & N/A & 4 & N/A \\
\bottomrule
\end{tabular}
\caption{Algorithm Specific Parameters}
\label{tab:specific_parameters}
\end{table}

\addcontentsline{toc}{section}{References}
\bibliographystyle{theapa}
\typeout{}
\bibliography{reference}

\end{document}